\documentclass[a4paper]{article}

\usepackage[margin=1in]{geometry}
\usepackage{microtype}
\usepackage{graphicx}
\usepackage{subfigure}
\usepackage{booktabs} %
\usepackage{natbib}
\usepackage[dvipsnames]{xcolor}         %
\usepackage[colorlinks=true,linkcolor=violet,urlcolor=black,citecolor=MidnightBlue,allbordercolors={1 1 1}]{hyperref}       %

\usepackage{mathtools} %
\usepackage{amssymb,amsfonts,amsthm}
\usepackage{cancel}
\usepackage{algorithm}
\usepackage{algorithmicx}
\usepackage{subfigure}
\usepackage{enumitem}

\usepackage{xcolor}
\definecolor{mblc}{RGB}{25,25,152}
\definecolor{vlt}{RGB}{138,0,136}
\definecolor{fusc}{RGB}{202,44,146}
\hypersetup{colorlinks=true,linkcolor=black,urlcolor=black,citecolor=mblc,allbordercolors={1 1 1}}

\usepackage{cleveref}
\usepackage{tcolorbox}

\usepackage{url}            %
\usepackage{booktabs}       %
\usepackage{rotating}
\usepackage{multirow}       %
\usepackage{amsfonts}       %
\usepackage{nicefrac}       %
\usepackage{microtype}      %

\usepackage{minitoc}

\def\isPreprint{1}  %

\usepackage[color=lightgray,textsize=tiny]{todonotes}

\ifnum\isPreprint=0
\newcommand{\emphU}[1]{\underline{#1}}  %
\newcommand{\vshrink}[1]{\vspace{#1}}
\else
\newcommand{\emphU}[1]{\emph{#1}}  %
\newcommand{\vshrink}[1]{}
\fi

\input{def.set}
\input{def-local.set}
\newcommand{\titlename}{On Uncertainty Quantification for Near-Bayes Optimal Algorithms}

\title{\titlename}
\ifdefined\anonymous
\author{%
Anonymous Authors \\ 
Anonymous Institution 
}
\else
\author{
  Ziyu Wang \\ 
  \normalsize University of Oxford\\
  \normalsize \email{wzy196@gmail.com}\\
\and
  Chris Holmes \\ 
  \normalsize University of Oxford \\
  \normalsize \email{cholmes@stats.ox.ac.uk}\\
}
\date{}
\fi

\begin{document}

\maketitle

\begin{abstract}
  Bayesian modelling %
allows for the quantification of predictive uncertainty which %
is crucial in safety-critical applications. Yet for many machine learning algorithms, it is difficult to construct or implement their Bayesian counterpart. 
In this work we present a promising approach to address this challenge, based on the hypothesis that commonly used ML algorithms are efficient across a wide variety of tasks and may thus be \emph{near Bayes-optimal} w.r.t.~an unknown task distribution. 
We prove that it is possible to recover the Bayesian posterior defined by the task distribution, which is unknown but optimal in this setting, by building a \emph{martingale posterior} using the algorithm. 
We further propose a practical uncertainty quantification method that apply to %
general ML algorithms. Experiments based on a variety of neural network (NN) and non-NN algorithms demonstrate the efficacy of our method.

\end{abstract}

\doparttoc %
\faketableofcontents %

\section{Introduction}\label{sec:intro}

Bayesian modelling represents an important approach that enables %
favourable predictive performance in the small-sample regime and allows for the quantification of predictive uncertainty %
which is vital for high-stakes applications. 
Yet for many machine learning (ML) algorithms 
it can be difficult to %
design or implement their natural 
Bayesian counterpart. %
For example, 
the development of Bayesian neural network (NN) methods encounters challenges with inference \citep{sun2018functional,wang2018function,ma2019variational} and counterintuitive issues of 
model misspecification \citep{aitchison2020statistical,fortuin2021bayesian,kapoor2022uncertainty}; 
AutoML algorithms \citep{karmaker2021automl} involve complex processes for hyperparameter tuning and model aggregation that are hard to replicate in a Bayesian framework; 
and when algorithms are offered as a black-box service \citep[e.g.,][]{openai2023ft}, adapting them to Bayesian principles becomes impossible. 

How can we bring back the %
benefits of the Bayesian paradigm without being limited by its traditional constraints? 
In this work we present a promising approach towards this challenge %
based on the following \textbf{basic 
postulation}: %
the ML algorithm of interest has competitive average-case performance on hypothetical datasets---or \emph{tasks}---sampled from an \emph{unknown task distribution} $\pi$, 
and our present task can be viewed as a random sample from the same $\pi$. 
Formally, suppose the algorithm $\cA$ %
maps a training dataset $z_{1:n}$ to a parameter estimate $\cA(z_{1:n})$; we assume it 
satisfies an inequality similar to the following,
\ifdefined\doublecolumn
\begin{align*}
&\phantom{=} \EE_{\theta_0\sim \pi}\EE_{(z_{1:n}, z_*)\sim \PP_{\theta_0}} \ell(\cA(z_{1:n}), z_*)   \\
&\le \inf_{\cA'} \EE_{\theta_0\sim\pi}\EE_{(z_{1:n}, z_*)\sim p_{\theta_0}} \ell(\cA'(z_{1:n}), z_*) + \epsilon_n.  \numberthis\label{eq:multitask-main}
\end{align*}
\else
\begin{align*}
\EE_{\theta_0\sim \pi}\EE_{(z_{1:n}, z_*)\sim \PP_{\theta_0}} \ell(\cA(z_{1:n}), z_*)  
\le \inf_{\cA'} \EE_{\theta_0\sim\pi}\EE_{(z_{1:n}, z_*)\sim p_{\theta_0}} \ell(\cA'(z_{1:n}), z_*) + \epsilon_n.  \numberthis\label{eq:multitask-main}
\end{align*}
\fi
In the above, %
$\theta_0$ is a parameter that determines the data generating process $p_{\theta_0}$ in the task, %
$(z_{1:n}, z_*)$ denote the training and test samples, $\ell(\theta, z)$ is the loss function, and 
$\cA'$ ranges over all algorithms that maps $z_{1:n}$ to an $\cA'(z_{1:n})\approx \theta_0$; 
$\epsilon_n$ quantifies the suboptimality of $\cA$. %

To understand this postulation, imagine a practitioner working on a new image classification dataset. 
To understand the suitability of a certain algorithm $\cA$ (e.g., 
a combination of an %
NN model and its training recipe), 
it would be natural for them 
to start by reviewing the vast literature on image classification, where many papers %
may have 
evaluated $\cA$ on datasets deemed 
similar to the present one. 
At a high level, 
the past and present datasets %
can be loosely viewed as i.i.d.~samples from the unknown distribution $\pi$, and 
promising reports %
from past literature provide 
evidence that \eqref{eq:multitask-main} holds %
with a smaller $\epsilon_n$. 
The practitioner may then commit to the algorithm with the smallest $\epsilon_n$. 
As another type of example, 
condition \eqref{eq:multitask-main} is also relevant %
in \emph{multi-task learning} scenarios, 
where it often %
appears as the stated goal in algorithm design and analysis \citep[e.g.,][]{pentina2014pac,NEURIPS2020_d902c3ce,rothfuss2021pacoh,riou2023bayes}. 
Foundation models \citep{bommasani2021opportunities} that are pretrained on a diverse mix of datasets can also be viewed as optimised 
for 
\eqref{eq:multitask-main}, with a %
distribution $\pi$ %
designed to align with the %
downstream task of interest. 

Algorithms that satisfy \eqref{eq:multitask-main} are \emph{near-Bayes optimal}: knowledge of the Bayesian posterior defined by $\pi$ would enable %
the minimisation of \eqref{eq:multitask-main} \citep{ferguson2014mathematical}. %
As exemplified above, %
in  many practical scenarios 
there may conceptually exist a $\pi$ that provides a \emph{correctly specified} prior, but 
it is not explicitly known and cannot be used directly;
it is more reasonable to assume knowledge of a near-optimal %
$\cA$ than that of a %
correctly specified $\pi$.
Yet with such a choice of $\cA$, 
the challenge of uncertainty quantification remains; %
for example, for regular parametric models %
maximum likelihood estimation (MLE) can be asymptotically near-Bayes optimal %
\citep{van2000asymptotic}, but it does not provide any (epistemic) uncertainty estimate. The predictive performance of MLE in the small-sample regime may also be well improvable. %

To address these issues, we build on the ideas of \cite{fong_martingale_2021} and 
study \emph{martingale posteriors} (MPs), defined as the distribution of %
parameter estimates obtained by 
first using $\cA$ to generate a synthetic dataset, %
and then applying $\cA$ to the combined sample %
of real and synthetic data (see \S\ref{sec:bg} for a review). 
We prove that %
when $\cA$ defines an approximate martingale, satisfies a condition similar to \eqref{eq:multitask-main} and additional technical %
conditions, the resulted MP will provide a good approximation for the Bayesian 
posterior defined by $\pi$ in a Wasserstein distance. 
Such results allow us to draw from the benefits of the latter without requiring explicit knowledge of $\pi$ (or the ability to conduct approximate inference). 
Our results also improves the theoretical understanding of MPs, %
by better justifying %
its uncertainty estimates, allowing for %
a wider range of algorithms, and by covering the pre-asymptotic %
regime. %

As a further contribution, we present MP-inspired algorithms based on sequential applications of a general estimation algorithm. %
Our analysis, %
if interpreted broadly, 
justifies the use of any algorithm that can be assumed to satisfy \eqref{eq:multitask-main}. 
The method is related to %
bootstrap aggregation \citep{breiman1996bagging} but demonstrates distinct %
advantages. %
We evaluate the proposed method empirically on a variety of tasks involving NN and non-NN algorithms, including Gaussian process learning, %
classification with %
tree and AutoML algorithms, 
and conditional density estimation with diffusion models, where %
it consistently outperforms standard ensemble methods such as deep ensemble \citep{lakshminarayanan2017simple} and bootstrap. %

The rest of the paper is structured as follows: \S\ref{sec:bg} reviews the background; %
\S\ref{sec:theory} presents our theoretical results; \S\ref{sec:method} describes the proposed method, which is evaluated in \S\ref{sec:exp}. We provide concluding remarks in \S\ref{sec:conclusion}. For space reasons, discussion of related work is deferred to \Cref{app:all-disc}.

\vshrink{-0.5em}
\section{Background}\label{sec:bg}

\vshrink{-0.6em}
\paragraph{Notations.} 
We adopt the following notations: %
$\cZ$ denotes the data space. 
$(\cdot)_{m:n}$ denotes a range of subscripts (e.g., $z_{m:n} = (z_m,z_{m+1},\ldots,z_n)$). 
$\lesssim, \gtrsim, \asymp$ denote (in)equality up to a multiplicative constant. 
$\sim$ is used to denote asymptotic equivalence %
and also as a ``distributed as'' symbol.%

\vshrink{-0.7em}
\paragraph{Bayesian modelling.} 
Suppose we are given %
i.i.d.~samples $\{z_i\}_{i=1}^n$ from an unknown distribution $p_{\theta_0}$ and wish to learn a %
$\hat p_n\approx p_{\theta_0}$. 
Standard Bayesian modelling requires a parameter space $\Theta$, a likelihood function $p(z\mid\theta)$ and a prior $\pi$ over $\Theta$. We can then compute (or approximate) the posterior $\pi(d\theta\mid z_{1:n}) \propto \pi(d\theta) \prod_{i=1}^n p(z_i\mid\theta)$, The posterior defines the predictive distribution $\pi(z_{n+1}\in\cdot\mid z_{1:n}) = \int \pi(d\theta\mid z_{1:n}) p(z\in\cdot\mid \theta)$ that provides the learned approximation for $p_{\theta_0}$. It also quantifies predictive uncertainty %
through the variation in $\pi(\cdot\mid z_{1:n})$. 

When $\pi$ is ``correctly specified'', predictors derived from the posterior generally enjoy good theoretical guarantees. 
One way to justify such predictors %
is through their ability to %
minimise various average-case losses where data is sampled from the prior predictive distribution: 
for instance, the posterior predictive density minimises %
the %
loss %
$
\cL_{\log}(\hat f_n) := 
\EE_{\theta_0\sim\pi, (z_{1:n},z_{n+1})\sim p_{\theta_0}} \log \hat f_n(z_{n+1}; z_{1:n}). 
$
As the loss functional is defined w.r.t.~training and test data $(z_{1:n}, z_{n+1})$ 
from the prior predictive distribution, such statements are only relevant when %
$\pi$ is correctly specified to model the true data distribution. %

All Bayesian models are correctly specified for some tasks, %
but they do not necessarily cover the present one. 
In many cases, specifying models based on vague subjective beliefs or computational considerations can lead to disappointing performance. 
A classical example is the Bayesian Lasso, where %
the Laplace prior does not define a sparse posterior \citep{lykou2013bayesian}. 
Bayesian NNs arguably %
provide %
another example:
the convenient %
$\cN(0, \alpha I)$ prior %
can lead to undesirable consequences \citep{fortuin2021bayesian} despite its %
connection to the 
widely adopted $\ell_2$ regularisation. %
In such cases, the user faces an apparent dilemma: %
choose a prediction algorithm %
best suited for the task or have access to Bayesian uncertainty. 
Such issues---coupled with the challenges in inference---motivate the search of alternative methods for uncertainty quantification. 

\vshrink{-0.7em}
\paragraph{Martingale posteriors.} We review the ideas of martingale posteriors (MPs, \cite{fong_martingale_2021,holmes_statistical_2023}) which provides the basis of our work. 
Suppose we have observations $z_{1:n}$ and a suitable algorithm $\cA$ which, for any $j\ge n$, maps any $j$ observations $z_{1:j}\in\cZ^{\otimes j}$ to a (deterministic or random) parameter $\cA(z_{1:j})\in\Theta$. Consider a sequence of data and parameter samples defined recursively as follows:
\begin{equation}\label{eq:seq-est-orig}%
\EstParam[j] \gets \cA(z_{1:n}\cup \EstData[n+1:j]), ~ \EstData[j+1] \sim p_{\EstParam[j]}, ~~~
\text{for } j = n, \ldots 
\end{equation}
Informally, with reasonable choices of $\cA$ we expect the resulted $\{\EstParam[j]: j>n\}$ to converge a.s.~to a random $\EstParam[\infty]$ 
w.r.t.~a \emph{suitably chosen semi-metric} $d$, because after observing infinite samples the parameter uncertainty should vanish.\footnote{
For overparameterised models this is true if $d$ measures ``relevant differences'' between $p_\theta$ and $p_{\theta'}$ (Rem.~\ref{rem:identifiability}).
} The variation in the distribution $\EstParam[\infty]\vert z_{1:n}$ arises from the missingness of %
the true observations $\{z_j\}_{j=n+1}^\infty$ which, if observed, would have enabled us to identify $\theta_0$ w.r.t.~$d$. 
Thus, this distribution reflects the \emph{epistemic uncertainty} \citep{der2009aleatory} in the 
learning process 
and fulfils a similar role as the Bayesian posterior $\pi(\theta\vert z_{1:n})$ \citep{kendall2017uncertainties}. 

The above formulation is justified in part through the fact that it generalises Bayesian posteriors: $\EstParam[\infty]\mid z_{1:n}$ will distribute as the Bayesian posterior %
if we define $\cA(z_{1:j})$ to sample from $\pi(\theta\mid z_{1:j})$; see \cite{fong_martingale_2021}. 
More generally, as long as $\cA$ is such that $\{\EE(\EstParam[j]\vert z_{1:n}, \hat z_{n+1:j})\}_{j=n}^\infty$ defines a bounded martingale w.r.t.~some vector semi-norm $\|\cdot\|$, it will follow from Doob's theorem \citep{doob1949application} that $\EstParam[N]$ converges a.s.~to a $\EstParam[\infty]$ in this $\|\cdot\|$. 
The distribution $\EstParam[\infty]\mid z_{1:n}$ is thus called a \emph{martingale posterior}. 

\begin{remark}[supervised learning]\label{rem:supervised-learning}
The above %
can be extended to cover supervised learning %
where $z_i=(x_i,y_i)$ and the model parameter $\theta$ only determines $p(y\,\vert\, x)$: 
in \eqref{eq:seq-est-orig} we can sample %
$\hat x_{j+1}$ from an external distribution %
(e.g., a generative model, the empirical measure defined by $x_{1:n}\cup \hat x_{n+1:j}$, or unlabelled data if available), and $\hat y_{j+1}\sim p_\theta(\cdot\mid x=\hat x_{j+1})$. 
\end{remark}

\begin{remark}[identifiability and semi-norm]\label{rem:identifiability}
$\theta_0$ will not be identifiable 
in overparameterised models if we consider conventional choices of %
$\|\cdot\|$ (e.g., Euclidean norm for NN parameters). But the framework %
can still apply if we can determine suitable \emph{semi-norms} over $\Theta$, 
or replace the parameter space %
with %
equivalence classes of parameters that define the same \emph{prediction function}, which in turn determines the likelihood. 
Such semi-norms will allow us to focus on the differences between parameters that are \emph{relevant to the purpose of prediction; %
for this goal there is no need to distinguish between parameters that define the same likelihood.}\footnote{Past works on ``function-space inference'' \citep[e.g.,][]{sun2018functional,wang2018function,ma2019variational,burt2020understanding} advocated for the restriction to similar semi-metrics.} %
As a concrete example, 
in certain wide NN models the prediction function is determined by a linear map of a transformed parameter \citep{jacot2018neural,lee2019wide}; 
we can then use that linear map to define $\|\cdot\|$.
\end{remark}

\vshrink{-0.9em}
\paragraph{Martingales for machine learning?} %
The MP framework relieves the requirement for %
an explicitly and correctly specified %
prior, as long as the user can express their prior knowledge in the form of an algorithm $\cA$. 
Nonetheless, there is still the requirement that $\cA$ %
define a martingale. 
Past works have explored various choices of $\cA$, including nonparametric resampling and copula-based algorithms \citep{fong_martingale_2021} and purpose-built NN models that satisfy this requirement \citep{lee2022martingale,ghalebikesabi2023quasi}. 
Yet it is %
unclear how common ML algorithms, such as approximate empirical risk minimisation (ERM) on general NN models, can be adapted for this purpose. In this work we bridge this gap, building on %
the observation that online gradient descent (GD) defines a martingale \citep{holmes_statistical_2023}: for 
\ifdefined\doublecolumn
\begin{equation}\label{eq:gd-martingale}
\!\!\EstParam[j+1] := \EstParam[j] + \eta_j \nabla_\theta\log p_{\EstParam[j]}(\EstData[j+1]), \text{ where } \EstData[j+1]  \sim p_{\EstParam[j]},
\end{equation}
\else
\begin{equation}\label{eq:gd-martingale}
\EstParam[j+1] := \EstParam[j] + \eta_j \nabla_\theta\log p_{\EstParam[j]}(\EstData[j+1]), \text{ where } \EstData[j+1]  \sim p_{\EstParam[j]},
\end{equation}
\fi
we have $\EE(\EstParam[j+1]\mid z_{1:j})=\EstParam[j]$. 
We will start from the observation 
that a natural gradient variant of \eqref{eq:gd-martingale} enjoys desirable properties and connects to sequential maximum likelihood estimation (MLE) (\S\ref{sec:ex-expfam}); %
the latter perspective will allow us to derive algorithms for high-dimensional models (\S\ref{sec:ex-lingauss}) and, from a methodological point of view, DNN models (\S\ref{sec:method}). 

Another unaddressed question is how MPs can be justified theoretically, beyond the somewhat vague belief that the imputations from a suitable $\cA$ may ``approximate the missing data well''. 
While previous work \citep{fong_martingale_2021} established consistency for specific MPs, 
such a result does not fully justify the uncertainty estimates from the MPs, as they cannot guarantee the MP credible sets will contain the true parameter in any finite-sample scenario. 
Moreover, the intuition that imputations may approximate the missing data well is challenging in the small-sample regime, in which case the estimate $\cA(z_{1:n})$ is still a poor approximation to %
$\theta_0$; yet it is in this regime where predictive uncertainty is most needed. 
In the next section we %
address this question, starting from the basic postulation %
\eqref{eq:multitask-main}. %

\vshrink{-0.25em}
\section{Martingale Posteriors with Near-Optimal Algorithms}\label{sec:theory}
\vshrink{-0.25em}

This section presents our theoretical contributions. %
We will state our result formally in \S\ref{sec:main-results}. %
It %
can be informally summarised as follows: for algorithms that define approximate MPs, %
satisfy stability conditions and 
are \emph{sample efficient} on %
a task distribution $\pi$ \emph{in the sense of \eqref{eq:multitask-main}}, %
the resulted MP %
will be close to the Bayesian posterior defined by $\pi$ in a Wasserstein distance. %
It follows that the MP will provide useful uncertainty estimates on new tasks sampled from $\pi$, which is valuable when explicit knowledge of $\pi$ is not available and thus cannot be used %
to construct the (optimal) Bayesian posterior. 

As discussed in \S\ref{sec:intro}, our conceptual setup covers %
generic ML algorithms such as %
approximate MLE on DNN models: they are generally considered efficient on a %
variety of tasks that, loosely speaking, may represent samples from $\pi$, and the present task may be assumed to 
also fall into this category. 
While our theorem will not cover practical DNN models, we illustrate in \S\ref{sec:theory-examples} how it justifies similar %
algorithms on examples that cover high-dimensional, overparameterised models and the small-sample regime. The examples provide valuable insight to the algorithm's behaviour in more complex settings.

\vshrink{-0.2em}
\subsection{Setup and Main Result}\label{sec:main-results}

\vshrink{-0.3em}
\paragraph{Analysis setup.}
Our analysis covers simplified scenarios that 
nonetheless capture interesting aspects of applications. 
We focus on \emph{deterministic, online} algorithms 
$\{\Alg[j]: \Theta\times \cZ \mapsto\Theta\}$ 
that define (approximate) MPs by 
\ifdefined\doublecolumn
$\hat p_{mp,n} := \mathrm{Law}(\EstParam[N]\mid z_{1:n}),$ where 
\begin{equation}\label{eq:MP-for-analysis}
\EstParam[j+1] := \Alg[j+1](\EstParam[j], \EstData[j+1]), ~~~
\EstData[j+1] \sim p_{\EstParam[j]} \tag{\ref{eq:seq-est-orig}'}
\end{equation}
\else
\begin{equation}\label{eq:MP-for-analysis}
\hat p_{mp,n} := \mathrm{Law}(\EstParam[N]\mid z_{1:n}), ~~~\text{where}~~
\EstParam[j+1] := \Alg[j+1](\EstParam[j], \EstData[j+1]) ~~\text{and}~~
\EstData[j+1] \sim p_{\EstParam[j]} \tag{\ref{eq:seq-est-orig}'}
\end{equation}
\fi
are defined %
for $n\le j<N$ starting from an initial estimate $\EstParam[n]$.  
This covers the GD algorithm \eqref{eq:gd-martingale} which serves as an important example to motivate our assumptions. 
We allow \eqref{eq:MP-for-analysis} to be truncated at some $N<\infty$, which may make the efficiency assumption easier to validate at the cost of an increased error. 
It is helpful to view $N$ as a growing function of $n$, or substitute $N=\infty$ for simplicity.

We assume the existence of a vector semi-norm $\|\cdot\|$ that, informally speaking, measures the ``relevant differences'' between parameters that we are interested in (see Rem.~\ref{rem:identifiability}). 
Our goal is to show that on average and w.r.t.~this $\|\cdot\|$, the 2-Wasserstein distance between $\hat p_{mp,n}$ and the unknown posterior $\pi_n := \pi(\cdot\mid z_{1:n})$ has a higher order than the spread of the latter, defined through its \emphU{radius} $\BayesError[j]$:
\ifdefined\doublecolumn
\begin{align*}
\BayesError[j]^2 := \EE_{\theta_0\sim\pi,z_{1:j}\overset{iid}{\sim} p_{\theta_0}}\EE_{\theta_{p,j}\sim\pi(\cdot\mid z_{1:j})}\|\theta_{p,j} - \BayesParam[j]\|^2, \numberthis\label{eq:bayes-err-defn} \\ 
\end{align*} 
where $\BayesParam[j] := \EE_{\theta\sim \pi(\cdot\mid z_{1:j})}\theta$ 
\else
\begin{equation}\label{eq:bayes-err-defn}
\BayesError[j]^2 := \EE_{\theta_0\sim\pi,z_{1:j}\overset{iid}{\sim} p_{\theta_0}}\EE_{\theta_{p,j}\sim\pi(\cdot\mid z_{1:j})}\|\theta_{p,j} - \BayesParam[j]\|^2, ~~~\text{where}~~
\BayesParam[j] := \EE_{\theta\sim \pi(\cdot\mid z_{1:j})}\theta
\end{equation} 
\fi
denotes the posterior mean. %
Importantly, 
in the above, %
$\theta_{p,j}$ and $\theta_0$ are conditionally i.i.d.~given $z_{1:j}$, so $\BayesError^2$ also equals the (expected, squared) \emphU{error rate} of the estimator $\BayesParam[j]$ \citep{xu2022minimum}, which minimise the above error. 
\ifdefined\doublecolumn
We thus define the average ``\emphU{excess error}'' of $\Alg$ as 
\begin{align*}
\ExcessError[j]^2 &:= \EE_{\theta_0\sim\pi,z_{1:j}\overset{iid}{\sim} p_{\theta_0}}(\|\followParam[j] - \theta_0\|^2 - \|\BayesParam[j] - \theta_0\|^2) 
\numberthis\label{eq:excess-err-defn}
\end{align*}
where $\followParam := \Alg[j](\followParam[j-1], z_j)$ is defined recursively by applying $\Alg$ to the same set of $z_{1:j}$. 
Note the subtrahend above equals $\BayesError^2$.
\else
We hence define the average ``\emphU{excess error}'' incurred by $\Alg$ as 
\begin{equation}\label{eq:excess-err-defn}
\ExcessError[j]^2 := \EE_{\theta_0\sim\pi,z_{1:j}\overset{iid}{\sim} p_{\theta_0}}(\|\followParam[j] - \theta_0\|^2 - \|\BayesParam[j] - \theta_0\|^2)
= \EE_{\theta_0\sim\pi,z_{1:j}\overset{iid}{\sim} p_{\theta_0}}\|\followParam[j] - \theta_0\|^2 - \BayesError[j]^2, 
\end{equation}
where $\followParam := \Alg[j](\followParam[j-1], z_j)$ is defined recursively by applying $\Alg$ to the same set of $z_{1:j}$. %
\fi

We now state our assumptions. We first require 
$\Alg$ to define an approximate martingale w.r.t.~$\|\cdot\|$: 
\vshrink{-0.15em}
\begin{assumption}[approximate martingale]\label{asm:approx-martingale}
    There exists $\delta>0$ s.t.~for all $j\ge n$ and $\theta\in\Theta$, we have 
\ifdefined\doublecolumn
$
\|\EE_{\hat z\sim p_\theta} \EstDelta(\theta, \hat z)\|^2 \le j^{-2(1+\delta)}  \BayesError[j]^2,
$
where 
$
\EstDelta[j](\theta, z) := \Alg[j](\theta, z) - \theta.  
$
\else
$$
    \|\EE_{\hat z\sim p_\theta} \EstDelta(\theta, \hat z)\|^2 \le j^{-2(1+\delta)}  \BayesError[j]^2, ~~~~
    \text{where}~~
\EstDelta[j](\theta, z) := \Alg[j](\theta, z) - \theta.  
$$
\fi
\end{assumption}
\vshrink{-0.15em}
Now we introduce our first assumption on stability. %
For the GD algorithm \eqref{eq:gd-martingale}, its condition~\emph{(i)} merely requires $\nabla_\theta\log p_\theta(z)$ to be Lipschitz continuous w.r.t.~$\theta$ and $z$. %
\vshrink{-0.15em}
\begin{assumption}[stability I]\label{asm:stable}
There exist a norm $\|\cdot\|_z$ over $\cZ$, $\iota>0, L_1,L_2>0$ and $\eta_j\le  j^{-(1+\iota)/2}$ s.t.~for all 
$n\le j<N$, $\theta,\theta'\in\Theta, z,z'\in\cZ$, we have 
\ifnum\isPreprint=0
\begin{enumerate}[leftmargin=*,topsep=1pt,partopsep=1pt,parsep=1pt,label=(\roman*)]
\else
\begin{enumerate}
\fi
    \item $\|\EstDelta(\theta,z) - \EstDelta(\theta',z)\|^2 \le \eta_j^2 L_1^2 \|\theta-\theta'\|^2, ~
\|\EstDelta(\theta,z) - \EstDelta(\theta,z')\|^2 %
\le \eta_j^2 L_2^2 \|z-z'\|_z^2. %
$
    \item 
Let $W_{2,z}$ denote the 2-Wasserstein distance w.r.t.~$\|\cdot\|_z$. Then either \emph{(a)}
$W_{2,z}^2(p_\theta, p_{\theta'}) \le C_\Theta \|\theta - \theta'\|^2,$ or \emph{(b)}
$W_{2,z}^2(p_\theta, p_{\theta'}) \le C_\Theta \|\theta - \theta'\|$ and $\eta_j\le j^{-(3+\iota)/4}.$
\end{enumerate}
\end{assumption}
\vshrink{-0.2em}
The following condition characterises \emphU{efficiency in the sense of \eqref{eq:multitask-main}}: 
it requires that for sample sizes \emph{up to $N$}, 
the ``excess error'' \eqref{eq:excess-err-defn} incurred by $\Alg$ has a higher polynomial order. %
When $\eta_j\asymp j^{-1},\iota=1$ as in all examples in 
\S\ref{sec:theory-examples}, 
it is satisfied as long as $\ExcessError^2 \lesssim j^{-s'} \BayesError^2$ 
for an arbitrarily small $s'>0$. %
\vshrink{-0.1em}
\begin{assumption}[efficiency]\label{asm:replaced-efficiency}
There exist $s \in (0, \min\{\delta, \iota\})$ and a sequence $\{\nu_j\}\to 0$ s.t.~for all $n\le j\le N$, we have 
$
\ExcessError[j]^2 \le j^{-(1-\iota+s)} \nu_j \BayesError[j]^2.
$
\end{assumption}
\vshrink{-0.3em}
The following is a further condition on stability. %
For the GD algorithm \eqref{eq:gd-martingale}, 
equivalent conditions have appeared in previous work analysing its convergence \citep[H6]{moulines2011non}.%
\vshrink{-0.1em}
\begin{assumption}[stability II]\label{asm:martingale-divergence}
There exist $\Calg, \Calg'\ge 0$, %
$\{H_{\theta,j}\in\RR^{d\times d}\}_{\theta\in\Theta,j\in\mb{N}}$ 
s.t.~for all $\theta, \theta'\in\Theta$ and $j\in\mb{N}$, we have 
$
\|\EE_{z'\sim\PP_{\theta'}} \EstDelta[j+1](\theta, z') - \eta_j H_{\theta,j}(\theta'-\theta)\| 
\le \Calg\eta_j\|\theta'-\theta\|^2, %
$
$%
\|H_{\theta,j}\|_{op}^2 \le \Calg'$. %
\end{assumption}
\vshrink{-0.3em}
The following conditions 
are rather mild for regular parametric models in the large-sample regime ($\BayesError[n]^2 \asymp d/n$, $n\ge d^{1/(\iota-s)}$). They may also hold in the pre-asymptotic regime if $\Calg$ is small, as we show in 
\S\ref{sec:ex-expfam}. 
\begin{assumption}[miscellaneous conditions]\label{asm:conventions} 
\begin{inlineEnum}
\iItem\label{it:convention-scale}~For all $j\ge n$ we have $\ExcessError[j] \le 1, ~\BayesError[j]\ge j^{-1}$. 
\iItem\label{it:convention-ident}~$\lim_{j\to\infty}\BayesError[j] = 0$. $\{\ExcessError\}$ is non-increasing. 
\iItem\label{it:convention-asymp}~$\Calg\sum_{j\ge n} j^{1+s} \eta_j^2 \BayesError[j]^4 \le \nu_n \BayesError[n]^2$. %
\end{inlineEnum}
\end{assumption}

\vshrink{-0.8em}
\paragraph{Main result.} Our main result is the following:
\begin{theorem}[proof in App.~\ref{app:proof-thm-main}]\label{thm:param-alt}
Let $\pi_{n}, \hat p_{mp,n}$ be defined %
as above, %
and $W_{2,\theta}$ be the 2-Wasserstein distance w.r.t.~$\|\cdot\|$.  
Under Asm.~\ref{asm:approx-martingale}-\ref{asm:conventions}, %
there exists some $C>0$ determined by $(C_\Theta,\Calg,\Calg',L_1,L_2)$ s.t.~for $\chi_n = C/(sn^s) \to 0$ we have 
\ifdefined\doublecolumn
\begin{align*}
&\phantom{=}\EE_{\theta_0\sim\pi, z_{1:n}\overset{iid}{\sim} p_{\theta_0}} W_{2,\theta}^2(\pi_{n}, \hat p_{mp, n})  \\
&\le 
    2e^{\chi_n} ((\chi_n+\nu_n)\BayesError[n]^2 + \ExcessError[n]^2) + 2\ExcessError[N]^2 + \BayesError[N]^2
     \numberthis\label{eq:main-nonasymp}.
\end{align*}
\else
\begin{align*}
\EE_{\theta_0\sim\pi, z_{1:n}\overset{iid}{\sim} p_{\theta_0}} W_{2,\theta}^2(\pi_{n}, \hat p_{mp, n}) &\le 
    2e^{\chi_n} ((\chi_n+\nu_n)\BayesError[n]^2 + \ExcessError[n]^2) + 2\ExcessError[N]^2 + \BayesError[N]^2
     \numberthis\label{eq:main-nonasymp}.
\end{align*}
\fi
Consequently, if $N\gg n$ is sufficiently large so that $\BayesError[N]\ll \BayesError[n]$, we have 
\begin{equation}\label{eq:main-asymp}
 \EE_{\theta_0\sim\pi, z_{1:n}\overset{iid}{\sim} p_{\theta_0}} W_{2,\theta}^2(\pi_{n}, \hat p_{mp, n}) \ll \BayesError[n]^2. 
\end{equation}
\end{theorem}
\vshrink{-0.4em}
\Cref{thm:param-alt} provides an average-case bound on the 2-Wasserstein distance between the MP $\hat p_{mp,n}$ and the Bayesian posterior $\pi_n$. 
Such Wasserstein distance bounds justify the use of the MP to approximate credible sets defined by $\pi_n$: %
as we prove in App.~\ref{app:theory-additional-cb}, it follows from \eqref{eq:main-asymp} that 
any MP credible set %
can be ``enlarged'' by an amount of $o(\BayesError[n])$ w.r.t.~$\|\cdot\|$ to contain a Bayesian credible set with an asymptotically equivalent nominal level, and  
the modification is asymptotically negligible compared with the ``average-case spread'' of $\pi_n$, as measured by $\BayesError[n]$. 
Consequently, we can see that the MP will provide useful uncertainty estimates, whenever $\pi_n$ can be assumed to do so.  

\vshrink{-0.2em}
\subsection{Examples}\label{sec:theory-examples}
\vshrink{-0.2em}

\ifdefined\doublecolumn
\subsubsection{Exponential Family and Sequential MLE}\label{sec:ex-expfam}
\else
\subsubsection{Exponential Family Models and Sequential MLE}\label{sec:ex-expfam}
\fi
\vshrink{-0.1em}

Let $\bar p_\eta(z) \propto e^{\eta^\top T(z) - A(\eta)}$ be an exponential family model \citep{wainwright2008graphical} with natural parameter $\eta$, and $\theta(\eta) := \EE_{z\sim \bar p_\eta} T(z)$ denote the mean parameter. Then $\theta = \nabla_\eta A$, 
and we can use %
$p_\theta := \bar p_{(\nabla A)^{-1}(\theta)}$ to denote the model distribution corresponding to $\theta$. 
Consider the sequential MLE algorithm: for any set of $n$ observations 
$\{z_i\}_{i=1}^n$ it returns 
$\hat\theta_n = \frac{1}{n}\sum_{i=1}^n T(z_i)$. It can be equivalently expressed as 
\begin{equation}\label{eq:seq-MLE-EF}
\Alg[j](\EstParam[j-1], z_{j}) := \EstParam[j-1] + j^{-1}(T(z_j) - \EstParam[j-1]). 
\end{equation}
Note that \eqref{eq:seq-MLE-EF} is equivalent to \emph{natural gradient} with step-size $%
j^{-1}$ \citep{amari2016information} and thus generalises \eqref{eq:gd-martingale}. 

We choose $\pi$ to be a conjugate prior determined by the following density for %
$\eta$:
$
\bar\pi(\eta) \propto e^{\eta^\top \theta_\pi - \alpha A(\eta)}.
$ $\alpha>0, \theta_\pi\in\RR^d$
are the prior hyperparameters. 
We impose the following assumptions which are rather mild: $n+\alpha>2, \alpha = O(1), \BayesError^2 = O(d/n)$ where $d = \dim\theta$, the function $T$ is $L$-Lipschitz.  Then 
it can be readily verified that 
Asm.~\ref{asm:approx-martingale}, \ref{asm:stable}~(i), \ref{asm:martingale-divergence} hold with any $\delta>0, 
\eta_j = (j+1)^{-1}, L_1=1, L_2=L, 
H_{\theta,j}= I,\Calg=0,\Calg'=1$, and  Asm.~\ref{asm:conventions} holds when $n\gtrsim \sqrt{d}$. 
We prove in App.~\ref{app:ex-expfam-deriv} that 
Asm.~\ref{asm:replaced-efficiency} holds for all $s<\min\{1,\delta\}$ and $\nu_l\le 2\alpha l^{-1+s}$.

Validation of Asm.~\ref{asm:stable}~(ii) is more challenging, due to a somewhat lack of understanding of Wasserstein distance properties for exponential family models. 
We first note that it can be verified on a case-by-case basis by studying transport plans; in this way we can verify that the Gaussian model $p_\theta = \cN(\theta, \Sigma_0)$, and the exponential model $p_\theta = \mrm{Exp}(\theta)$ satisfy its \emph{(a)} with $C_\Theta=O(\|\Sigma_0\|_{op}^{-1})$ and $C_\Theta=1$, respectively, and the Bernoulli model satisfies its \emph{(b)} with $C_\Theta=8$. 
Another scenario %
where a similar condition holds is when %
$\sup_{z\in\cZ}\|T(z)\|$ is bounded and the eigenvalues of the Fisher information matrices are %
bounded from both sides. See App.~\ref{app:ex-expfam-deriv} for a detailed discussion. 

When the assumption holds, Thm.~\ref{thm:param-alt} will establish the bound %
$
W_{2,\theta}^2(\hat p_{mp,n}, \pi_n) \lesssim %
n^{-1/2} \BayesError[n]^2.
$
Note \emph{this applies to the pre-asymptotic regime} $\sqrt{d}\lesssim n\lesssim d$ when the estimation error %
is $\|\hat \theta_n - \theta_0\|\gtrsim 1$. 

\vshrink{-0.2em}
\subsubsection{Regularised Algorithms in High Dimensions} %
\vshrink{-0.2em}
\label{sec:ex-lingauss} 

The above example involves unregularised MLE %
which is known to %
perform poorly %
on some high-dimensional problems. %
We now present %
a high-dimensional 
example where a regularised variant of the MP %
enjoys good guarantees. %
This example further connects to Gaussian processes (GP) regression. 

\vshrink{-0.5em}
\paragraph{A linear-Gaussian inverse problem.}
Let $(\cH,\cZ)$ be two %
Hilbert spaces for $\theta$ and $z$, respectively, and $A:\cH\to\cZ$ be a Hilbert-Schmidt operator.
Suppose $z_{1:n}$ is generated by 
$
z_i\mid\theta\sim\cN_\cZ(A\theta, I)
$
where $\cN_\cZ$ denotes the shifted iso-normal process on $\cZ$ \citep[see e.g.,]{van2008reproducing}.
We define an MP using %
\ifdefined\doublecolumn
\begin{equation}\label{eq:lingauss-alg}
\Alg[j](\theta, z) := \theta + \eta_j G_j \nabla\log p(z;\theta)
\end{equation}
where  
$\eta_j = j^{-1}, ~G_j = (A^\top A + j^{-1} I)^{-1},$
\else
\begin{equation}\label{eq:lingauss-alg}
\Alg[j](\theta, z) := \theta + \eta_j G_j \nabla\log p(z;\theta) ~~~
\text{ where }~~\eta_j = j^{-1}, ~G_j = (A^\top A + j^{-1} I)^{-1},
\end{equation}
\fi
and compare %
with the posterior $\pi_n$ defined by %
$
    \pi = \cN_\cH(0, I).
$
The setup is closely related to the classical %
inverse problems defined by %
white noise \citep{cavalier2008nonparametric}. 
Following a convention in that literature, we assume the singular values $s_i(A)\asymp i^{-\beta}$, %
and adopt the norm %
$\|\theta-\theta'\| = \|(A^\top A)^{\alpha/2} (\theta-\theta')\|_\cH$ 
where $\beta>1/2, \alpha\in\RR$ are problem parameters. 
When $\alpha=1$, we can view 
the problem as regression in a Sobolev space with $\|\cdot\|$ equivalent %
to the $L_2$ norm.
See App.~\ref{app:wn} for details. 

As we prove in App.~\ref{app:wn}, all assumptions in \S\ref{sec:main-results}
hold with the above 
$\eta_j$, %
all choices of $(\{\nu_j\}, \iota, \delta, s)$ 
and $L_1=L_2=C_\Theta=\Calg'=1,\Calg=0$. 
Thm.~\ref{thm:param-alt} thus applies and gives a bound of $\cO(\BayesError[n]^2/n)$. Note \emph{the result does not depend on the extrinsic dimensionality of $\theta$}.%

\begin{remark}
Note that $\{\Alg\}$ would produce the same output as the posterior mean had we applied it to %
$\pi$-generated data. 
However, the result above is non-trivial, because the samples used to define the MP, $\{\EstData[j]\}_{j=n+1}^\infty$, are quite different from samples from the posterior predictive distribution: the latter is defined by a mixture of parameters, the full posterior, whereas $\{\EstData[j]\}$ is defined by a single point estimate. 
It is thus interesting that $\EE_\pi W_2^2(\pi_n, \hat p_{mp,n})/\BayesError^2$ is bounded by a dimension-free factor. 
\end{remark}

\vshrink{-0.75em}
\paragraph{Connections to GP regression.} The above example connects to GP regression through its connection to inverse problems that are %
asymptotically equivalent to regression %
\citep{cavalier2008nonparametric}. Alternatively, we
can observe that if we set $\cH$ to be a reproducing kernel Hilbert space, %
$\pi$ will reduce to the respective standard GP prior, %
and the operator $A: \cH\ni f\mapsto (f(x_1),\ldots, f(x_n))$ is Hilbert-Schmidt; %
hence, the above derivations 
should apply to GPs. %

We refer readers to App.~\ref{app:wn} 
for a detailed discussion of the above, %
where we also note that %
\eqref{eq:lingauss-alg} can be used for GP inference. 
However, the following algorithm provides a more practical alternative:
\ifdefined\doublecolumn
\begin{align*}
&\textstyle
\EstParam[j+1] :=\argmin_{\theta\in\cH}\Bigl( \sum_{i=1}^j (f_{\EstParam[j]}(x_i) - f_\theta(x_i))^2 + 
\\ &\hspace{5em}
\textstyle
(f_\theta(\hat x_{j+1}) - \hat y_{j+1})^2 + \frac{1}{n}\|\theta - \EstParam[j]\|_\cH^2
\Bigr), 
\numberthis\label{eq:gp-alg-spo}
\end{align*}
\else
\begin{equation}
\EstParam[j+1] :=\argmin_{\theta\in\cH}\, \sum_{i=1}^j (f_{\EstParam[j]}(x_i) - f_\theta(x_i))^2 + 
(f_\theta(\hat x_{j+1}) - \hat y_{j+1})^2 + \frac{1}{n}\|\theta - \EstParam[j]\|_\cH^2, \label{eq:gp-alg-spo}
\end{equation}
\fi
where $f_\theta$ denotes the regression function defined by $\theta$, 
$\hat x_{j+1}$ is set according to Rem.~\ref{rem:supervised-learning}, 
$\hat y_{j+1}\sim p(\hat y_{j+1}\mid f(X)=\EstParam[j], \hat x_{j+1})$, and 
with a slight abuse of notation we use $(x_i,y_i)$ to refer to the $i$-th (real or synthetic) observation received by the algorithm. 
As we verify in App.~\ref{app:wn}, Eq.~\eqref{eq:gp-alg-spo} is based on the same principle of iterative maximum-a-posteriori estimation as \eqref{eq:lingauss-alg}. %

Similar to some previous works on GP inference \citep{osband2018randomized,he2020bayesian,pearce20a}, we can implement \eqref{eq:gp-alg-spo} using random feature approximations for $\cH$; 
the resulted algorithm can also be applied to overparameterised random feature models that represent a simplified %
model for DNNs \citep{lee2019wide}. 
It is also worth noting a line of theoretical work on \emph{multi-task learning} \citep{tripuraneni2020theory,du2020few,tripuraneni2021provable,wang2022fast}, which proved in a stylised setting that it is possible to learn an approximation of $\cH$ that performs well on i.i.d.~test tasks; 
thus the premise and implications of Thm.~\ref{thm:param-alt} may hold true, 
which will provide a non-trivial (albeit stylised) example where predictive uncertainty quantification can be aided by pretraining data. 
In App.~\ref{app:gp-mtl} we confirm this through numerical simulations on a similar setup. %
Finally, from a methodological perspective, \eqref{eq:gp-alg-spo} is also interesting because 
as we discuss shortly, it is connected to a \emph{function-space Bregman divergence} of the likelihood loss \citep{bae_if_2022}, which motivates the use of similar objectives in broader scenarios. 

\section{MP-Inspired Uncertainty %
for General Algorithms}\label{sec:method}
\vshrink{-0.1em}

\S\ref{sec:theory-examples}
illustrates the efficacy of the MP \eqref{eq:MP-for-analysis} for uncertainty quantification %
when it is instantiated with a sequential MLE algorithm or its regularised variants. 
The results suggest that similar procedures should be broadly applicable, even to models beyond the scope of the analysis. We now discuss the implementation of such an MP-inspired scheme.

\vshrink{-0.5em}
\ifdefined\doublecolumn
\paragraph{An 
``iterative parametric bootstrap'' scheme.} 
\else
\paragraph{From MLE/MP to an 
``iterative parametric bootstrap'' scheme.} %
\fi
As \eqref{eq:MP-for-analysis} is based on sequential sampling and refitting, 
it is natural to generalise the procedure as follows: 
\vshrink{-0.6em}
\begin{algorithm}[H]
    \caption{~MP-inspired uncertainty quantification}\label{alg:main}
\begin{enumerate}[leftmargin=0.6cm,topsep=0pt]%
\item\label{it:alg1} Initialisation: $D_n := z_{1:n}, \EstParam[n] \gets \cA_0(D_n)$
\item\label{it:alg2} for $j\gets n,n+1,\ldots,n+\lfloor N/\Delta n\rfloor$
\begin{enumerate}[leftmargin=0.5cm]
    \item Sample $\EstData[n_j: n_j+\Delta n]\sim p_{\EstParam[j]}$;  $D_{j+1}\gets D_j\cup\EstData[n_j:n_j+\Delta n]$
    \item $\EstParam[j+1] \gets \cA(D_{j+1}; \EstParam[j])$ 
\end{enumerate}
\item Repeat \ref{it:alg1}--\ref{it:alg2} for $K$ times; use the resulted $\{\EstParam[n+\lfloor N/\Delta n\rfloor]^{(k)}\}_{k=1}^K$ to form an ensemble predictor
\end{enumerate}
\vshrink{-0.1em}
\end{algorithm}
\vshrink{-0.85em}

In the above, $(\cA_0(D), \cA(D; \theta))$ denote a general estimation algorithm. %
The analysis in \S\ref{sec:theory} justifies the use of algorithms that are connected to sequential MLE, and 
loosely suggests that any $\cA$ may be used if it is sample efficient in the sense of \eqref{eq:multitask-main}. To accelerate the computational process, we allow $\cA$ to resume from the previous iteration's optimum $\theta$ if %
possible. 
Compared with \eqref{eq:MP-for-analysis},  %
we also modify the procedure to process $\Delta n>1$ samples at each iteration. %

Alg.~\ref{alg:main} has a form similar to \emph{parametric bootstrap} \citep{efron2012bayesian}, 
which correspond to setting $\Delta n=N=n$ %
and \emph{discarding} the original dataset $\{z_{1:n}\}$ when estimating $\EstParam[j+1]^{(k)}$. 
With the differences in Alg.~\ref{alg:main} 
we may expect to achieve better performance.  
This is suggested by the analysis in \S\ref{sec:theory} which may become applicable at $\Delta n=1$, and we will also support this claim with experiments and theoretical examples. 

\vshrink{-0.5em}
\paragraph{A modified objective for DNNs.} %
Many ML algorithms can be directly plugged into Alg.~\ref{alg:main}. 
For DNN-based estimation algorithms, however, it may be preferable to modify the base algorithm to explicitly model the effect of early stopping:  
while DNNs are often trained to minimise a (regularised) empirical risk, due to early stopping the resulted 
$\EstParam[j]$ may not reach the optimum of its respective objective w.r.t.~$D_j$. When processing new samples, %
it can be desirable to avoid further optimisation on the part of the training loss that corresponds to $D_j$. 
For this purpose we adopt the modification in \cite{bae_if_2022}: 
suppose the original objective for $\EstParam[j+1]$ has the form of $
\sum_{z\in D_{j+1}} \ell(f(z;\theta), z)$, where $f(z;\theta)$ denotes the output from the DNN, 
we adopt the following modified algorithm for $\EstParam[j+1]$,
\ifdefined\doublecolumn
\begin{align*} &\textstyle 
\!\!\!\cA(D_{j+1};\EstParam[j]) := \arg\!\min_{\theta} \bigl[
\sum_{z\in D_j} \bar\ell_B(f(z;\theta), z; f(z;\EstParam[j])) 
\\&\textstyle\hspace{9.4em}
 + \sum_{l=n_j}^{n_j+\Delta n}\ell(f(\EstData[l];\theta), \EstData[l])\bigr],
 \numberthis\label{eq:nn-objective}
\end{align*}
\else
\begin{align*} %
\cA(D_{j+1};\EstParam[j]) := \arg\min_{\theta} 
\sum_{z\in D_j} \bar\ell_B(f(z;\theta), z; f(z;\EstParam[j])) 
 + \sum_{l=n_j}^{n_j+\Delta n}\ell(f(\EstData[l];\theta), \EstData[l]), %
 \numberthis\label{eq:nn-objective}
\end{align*}
\fi
where 
$
\bar\ell_B(f, z; \bar f) := \ell(f,z) - \ell(\bar f, z) - \nabla_f \ell(\bar f, z) (f-\bar f)
$ is a function-space Bregman divergence. %
As long as $\ell(f, z)$ is convex w.r.t.~the \emph{function value} $f$ (e.g., if $\ell$ is the square loss or cross-entropy loss), the first term of \eqref{eq:nn-objective} is always minimised by the old $\EstParam[j]$, thus retaining the regularisation effect of early stopping. 
As an example, when $\ell$ is the squared loss for regression, \eqref{eq:nn-objective} will have the form of \eqref{eq:gp-alg-spo} (modulo regularisation). 
\eqref{eq:nn-objective} can be augmented with explicit regularisers if desired, and the resulted algorithm $\cA$ can be plugged into Alg.~\ref{alg:main}. %
We discuss implementation details in App.~\ref{app:impl-details}. 

\vshrink{-0.5em}
\paragraph{Comparison to classical bootstrap.} Alg.~\ref{alg:main} is broadly similar to bootstrap aggregation \citep{breiman1996bagging}: both build an ensemble of model parameters by estimating on perturbed versions of the training set. However, in contrast to 
classical bootstrap schemes which only have asymptotic guarantees, Alg.~\ref{alg:main} can be justified in the small-sample regime (\S\ref{sec:theory}). 
App.~\ref{app:cmp-bagging} further 
presents concrete examples where Alg.~\ref{alg:main} has a more desirable theoretical behaviour %
when the training data is not sufficiently informative. %
This is consistent with %
\S\ref{sec:exp} where we find our %
method to perform better empirically. 
Broadly similar limitations for nonparametric bootstrap are also known in various contexts \citep{nixon2020why,davidson2010wild}. 

While we have focused on uncertainty quantification for deterministic algorithms that do not maintain any notion of parameter (i.e., epistemic) uncertainty, it is worth noting that Alg.~\ref{alg:main} can also be applied to ``fully Bayesian'' algorithms that sample %
from the posterior, in which case it will not overestimate the epistemic uncertainty; see \cite[\S 2.1]{fong_martingale_2021}, or \S\ref{sec:bg}.\footnote{
A major difference between our work and \cite{fong_martingale_2021} is that we allow the use of deterministic estimation algorithms, which is justified by \S\ref{sec:theory}. \cite{fong_martingale_2021} requires the algorithm $\cA$ to satisfy coherence conditions (e.g., defining c.i.d.~%
samples); this precludes choices of $\cA$ such as GD or %
MLE, and implies that $\cA$ already maintains a coherent notion of epistemic uncertainty (App.~\ref{app:related-work}).
} This is in stark contrast to conventional bootstrap, which %
will overestimate parameter uncertainty given %
such $\cA$. %

\vshrink{-0.25em}
\section{Experiments}\label{sec:exp}
\vshrink{-0.2em}

In this section we evaluate the proposed method empirically across a variety of ML tasks. %
Additional simulations are presented in App.~\ref{app:exps}, which provide more direct validation of the claims in \S\ref{sec:theory}.

\paragraph{Hyperparameter learning for GP regression.} We investigate whether the proposed method could alleviate overfitting in GP hyperparameter learning \citep[\S 5.1]{williams2006gaussian}. 
We instantiate our method (\texttt{IPB}) 
using empirical Bayes (\texttt{EB}) as the base estimation algorithm, 
and compare it 
with nonparametric bootstrap (\texttt{BS}) 
and vanilla ensemble (\texttt{Ens}) based on initialisation randomness, both applied to \texttt{EB} as well.  
We adopt GP models with a Mat\'ern-3/2 kernel and a Gaussian likelihood; hyperparameters include 
a vector-valued kernel bandwidth \citep{neal1996bayesian} %
and the likelihood variance. 
We evaluate %
on %
9 UCI datasets %
used in %
\citep{sun2018functional,wang2018function,ma2019variational,salimbeni2017doubly,dutordoir2020sparse} and subsample $n\in \{75, 300\}$ observations for training. 
We report the following metrics: root mean-squared error (RMSE), negative log predictive density (NLPD) and continuous ranked probability score (CRPS). 
All experiments are repeated on 50 random train/test splits.
For space reasons, we defer full details and 
results to App.~\ref{app:exp-gp}, and report the average rank of each method %
in Table~\ref{tbl:gp-main}. We can see that the proposed method achieves the best overall performance. %

\ifdefined\doublecolumn
\vshrink{-1em}
\begin{table}[h] 
    \centering
\caption{GP experiment: average rank across all datasets for each metric. %
Boldface denotes the best method. See Table~\ref{tbl:gp-full} in appendix for full results, including statistical significance tests. 
}\label{tbl:gp-main}
\resizebox{.95\linewidth}{!}{%
    \begin{tabular}[ht]{ccccccccc} \toprule
\multirow{2}[2]{*}{Metric} &  \multicolumn{4}{c}{$n=75$} & \multicolumn{4}{c}{$n=300$} \\
         \cmidrule(lr){2-5}  \cmidrule(lr){6-9} 
        & EB	&  BS	&  Ens	&  IPB	&  EB	&  BS	&  Ens	&  IPB	
        \\ \midrule 
        RMSE 	& $3.1$	& $2.7$	& $2.4$	& $\mathbf{1.4}$   
            & $2.9$	& $3.0$	& $2.0$	& $\mathbf{1.1}$ \\ 
        NLPD	& $3.0$	& $2.0$	& $2.6$	& $\mathbf{1.6}$   
            & $2.7$	& $3.0$	& $2.2$	& $\mathbf{1.1}$ \\
        CRPS	& $3.0$	& $2.3$	& $2.6$	& $\mathbf{1.4}$   
            & $2.7$	& $3.3$	& $2.1$	& $\mathbf{1.1}$ \\
\bottomrule\end{tabular}
}
\vshrink{-0.1em}
\end{table}
\else
\vshrink{-0.3em}
\begin{table}[h] \small
    \centering
\caption{GP experiment: average rank across all datasets for each metric. %
Boldface denotes the best method.
See Table~\ref{tbl:gp-full} in appendix for full results, including statistical significance tests. 
}\label{tbl:gp-main}
\vspace{0.2em}
    \begin{tabular}[ht]{ccccccccc} \toprule
\multirow{2}[2]{*}{Metric} &  \multicolumn{4}{c}{$n=75$} & \multicolumn{4}{c}{$n=300$} \\
         \cmidrule(lr){2-5}  \cmidrule(lr){6-9} 
        & EB	&  BS	&  Ens	&  IPB	&  EB	&  BS	&  Ens	&  IPB	
        \\ \midrule 
        RMSE 	& $3.1$	& $2.7$	& $2.4$	& $\mathbf{1.4}$   
            & $2.9$	& $3.0$	& $2.0$	& $\mathbf{1.1}$ \\ 
        NLPD	& $3.0$	& $2.0$	& $2.6$	& $\mathbf{1.6}$   
            & $2.7$	& $3.0$	& $2.2$	& $\mathbf{1.1}$ \\
        CRPS	& $3.0$	& $2.3$	& $2.6$	& $\mathbf{1.4}$   
            & $2.7$	& $3.3$	& $2.1$	& $\mathbf{1.1}$ \\
\bottomrule\end{tabular}
\vshrink{-0.2em}
\end{table}
\fi

\paragraph{Classification with GBDT and AutoML algorithms.} We now turn to classification and consider %
two base %
algorithms: \emph{(i)} 
gradient boosting decision trees \citep[\emph{GBDT}s,][]{friedman2001greedy} implemented as in XGBoost \citep{chen2016xgboost}, and \emph{(ii)} AutoGluon \citep{agtabular}, an AutoML system that 
aggregates a range of tree and DNN models. %
Both are highly competitive approaches that outperform conventional deep learning methods on tabular data \citep{grinsztajn2022tree,shwartz2022tabular}, yet neither has a natural Bayesian counterpart. 
Our method fills in this important gap, enabling us to mitigate overfitting and quantify uncertainty %
based on Bayesian principles. 

We evaluate on 30 OpenML \citep{bischl2017openml} datasets chosen by \cite{hollmann2022tabpfn}. 
For each algorithm, we apply our method (\texttt{IPB}) and compare with bootstrap aggregation (\texttt{BS}) and the base algorithm without additional aggregation. 
Alg.~\ref{alg:main} is implemented by sampling $\hat x_{n+i}$ from the empirical distribution of past inputs. All hyperparameters, including $(\Delta n, N)$ in Alg.~\ref{alg:main}, are determined using log likelihood on a validation set. 
Experiments are repeated on 10 random train/test splits. 
Full details are deferred to App.~\ref{app:exp-tree}. 

Table~\ref{tbl:tree-main} reports the average test accuracy and negative log likelihood (NLL) across all datasets. We can see that 
for both choices of base algorithms, our method achieves better \emphU{predictive performance} than the base algorithm as well as its bootstrap variant.  
Full results are deferred to App.~\ref{app:exp-tree}, where we further demonstrate that the improvement is consistent across all datasets, and that our method produces informative \emphU{uncertainty estimates} for the feature importance scores from GDBT.

\ifdefined\doublecolumn
\vshrink{-0.8em}
\begin{table}[ht]
    \centering 
\caption{Classification experiment: average ranks across 30 OpenML datasets. 
Boldface indicates the best result within each group of methods.
See App.~\ref{app:exp-tree} for full results and statistical significance tests.
}\label{tbl:tree-main}
\resizebox{.99\linewidth}{!}{%
\begin{tabular}[h]{ccccccc}
    \toprule 
\multirow{2}{*}[-0.2em]{Metric} & \multicolumn{3}{c}{GDBT} & \multicolumn{3}{c}{AutoML} \\ 
\cmidrule(lr){2-4}  \cmidrule(lr){5-7}
    & (Base) & + BS & + IPB & (Base) & + BS & + IPB 
    \\ \midrule 
    NLL
    & $4.77$ &  $4.33$ &  $\mathbf{3.20}$ & $3.60$ & $3.03$ & $\mathbf{2.07}$ \\
    Accuracy
    &  $4.87$ &  $4.43$ & $\mathbf{3.23}$ & $3.50$ & $2.50$ & $\mathbf{2.47}$ \\
    \bottomrule
\end{tabular}
}
\vshrink{-0.6em}
\end{table}

\else
\begin{table}[ht]\small
    \centering 
\setlength{\tabcolsep}{5pt}
\caption{Classification experiment: average test metrics and ranks across 30 OpenML datasets. 
Boldface indicates the best result within each group of methods.
Ranks are calculated by sorting across all six methods. See App.~\ref{app:exp-tree} for full results, including statistical significance tests.
}\label{tbl:tree-main}
\vspace{0.2em}
\resizebox{.98\linewidth}{!}{%
\begin{tabular}[h]{ccccccc}
    \toprule 
\multirow{2}{*}[-0.2em]{Metric} & \multicolumn{3}{c}{GDBT (XGBoost)} & \multicolumn{3}{c}{AutoML (AutoGluon)} \\ 
\cmidrule(lr){2-4}  \cmidrule(lr){5-7}
    & (Base) & + BS & + IPB & (Base) & + BS & + IPB 
    \\ \midrule 
    NLL	/ Avg.~Rank  
    & $0.215$ / $4.77$ & $0.207$ / $4.33$ & $\mathbf{0.200}$ / $\mathbf{3.20}$ & $0.215$ / $3.60$ & $0.190$ / $3.03$ & $\mathbf{0.185}$ / $\mathbf{2.07}$ \\
    Accuracy / Avg.~Rank	
    & $90.4$ / $4.87$ & $90.7$ / $4.43$ & $\mathbf{90.9}$ / $\mathbf{3.23}$ & $91.0$ / $3.50$ & $91.3$ / $2.50$ & $\mathbf{91.5}$ / $\mathbf{2.47}$ \\
    \bottomrule
\end{tabular}
}
\vshrink{-0.6em}
\end{table}
\fi

\paragraph{Interventional density estimation with diffusion models.} 
Finally, we present a set of NN-based experiments on %
the estimation of interventional distributions \citep{pearl2009causality} given a %
causal graph. 
Such a task can be seen as conditional density estimation but involves distribution shifts induced by the intervention. 
Recent works demonstrated the efficacy of deep generative models \citep{sanchez2022vaca,khemakhem2021causal,chao2023interventional} on this task. 
We are interested in whether our algorithm could lead to further improvements by better accounting for predictive uncertainty, which can be especially relevant here due to the distribution shift present. 

We instantiate Alg.~\ref{alg:main} using %
diffusion models following \cite{chao2023interventional}, and employ the modified objective \eqref{eq:nn-objective}. %
We evaluate on two sets of datasets: \emph{(i)} 8 synthetic datasets in \cite{chao2023interventional}; \emph{(ii)} a set %
of real-world fMRI datasets constructed by \cite{khemakhem2021causal}. %
In both cases we repeat all experiments 30 times, %
using independently sampled train/validation splits and initialisation for NN parameters. 
See App.~\ref{app:exp-dj} for full details. 

For the synthetic datasets, we compute the %
maximum mean discrepancy (MMD) w.r.t.~the ground truth on a grid of %
queries following \cite{chao2023interventional}. 
We compare with other ensemble methods applied to the same model: parametric (\texttt{PB}) and nonparametric (\texttt{BS}) bootstrap, deep ensemble (\texttt{Ens}), and 
the method of \citet[\texttt{NTKGP}]{he2020bayesian}. %
We choose these baselines because \texttt{Ens} has demonstrated strong performance in previous benchmarks %
\citep[e.g.,][]{gustafsson2020evaluating,ovadia2019can}, 
and \texttt{NTKGP} is motivated from a wide NN 
setup similar to \S\ref{sec:ex-lingauss}. 
As shown in Table~\ref{tbl:dj-synth-main}, 
the proposed method (\texttt{IPB}) achieves the best \emphU{predictive performance} across all datasets. Full results are deferred to App.~\ref{app:exp-dj}, where we further evaluate \emphU{uncertainty quantification} through the coverage of credible/confidence intervals;  %
we %
find our method generally provides the best coverage, followed by the bootstrap baselines. 

\ifdefined\doublecolumn
\begin{table}[ht] \small
\vshrink{-0.5em}
    \centering 
\caption{Density estimation: average rank across all synthetic datasets. Boldface indicates the best result. See App.~\ref{app:exp-dj} for full results and significance tests.}
\label{tbl:dj-synth-main}
\begin{tabular}[h]{cccccc}
\toprule
$n$ & PB & Ens. & NTKGP & BS & IPB \\ \midrule 
$100$	& $3.6$	& $1.9$	& $5.0$	& $3.1$	& $\mathbf{1.0}$	\\
$1000$	& $4.0$	& $1.9$	& $5.0$	& $2.4$	& $\mathbf{1.2}$	
\\ \bottomrule
\end{tabular}
\end{table}
\else
\begin{table}[ht] \small
\vshrink{-0.5em}
    \centering 
\caption{Interventional density estimation: average rank across all synthetic datasets. Boldface indicates the best result. See App.~\ref{app:exp-dj} for full results and significance tests.}
\vspace{0.2em}
\label{tbl:dj-synth-main}
\begin{tabular}[h]{cccccc}
\toprule
$n$ & PB & Ens. & NTKGP & BS & IPB \\ \midrule 
$100$	& $3.6$	& $1.9$	& $5.0$	& $3.1$	& $\mathbf{1.0}$	\\
$1000$	& $4.0$	& $1.9$	& $5.0$	& $2.4$	& $\mathbf{1.2}$	
\\ \bottomrule
\end{tabular}
\end{table}
\fi

On the fMRI datasets, we report the median absolute error %
following \cite{khemakhem2021causal,chao2023interventional}, as well as CRPS which better evaluates %
the estimation quality for the entire interventional distribution. 
\ifdefined\doublecolumn
We compare with the flow-based method of \cite[\texttt{Flow}]{khemakhem2021causal} and the nonlinear baseline therein (\texttt{ANM}), 
\else
We compare with the flow-based method of \cite[\texttt{Flow}]{khemakhem2021causal} and the baselines therein (\texttt{Linear}, \texttt{ANM}), 
\fi
as well as the same diffusion model combined with deep ensemble (\texttt{D+Ens}) and nonparametric bootstrap (\texttt{D+BS}). As shown in Table~\ref{tbl:fmri}, our method (\texttt{D+IPB}) achieves the best predictive performance.

\begin{table}[ht]\small
\vshrink{-0.5em}
    \centering 
\caption{%
Results for the fMRI datasets. Boldface indicates the best result 
($p<0.05$ in a $Z$ test). 
}\label{tbl:fmri}
\vspace{0.2em}
\ifdefined\doublecolumn
\setlength{\tabcolsep}{3.6pt}
\resizebox{.99\linewidth}{!}{%
\begin{tabular}[h]{ccccccc}
\toprule 
Metric &  ANM & Flow & D+Ens & D+BS & D+IPB \\ 
\midrule
CRPS 
&         $.551${\tiny $\pm .01$}
&         $.546${\tiny $\pm .02$}
&         $.520${\tiny $\pm .00$}
& $\mathbf{.518}${\tiny $\pm .00$}
& $\mathbf{.518}${\tiny $\pm .00$}
\\ 
Abs.~Err
&         $.655${\tiny $\pm .01$}
& $\mathbf{.605}${\tiny $\pm .02$}
&         $.609${\tiny $\pm .01$}
&         $.611${\tiny $\pm .01$}
& $\mathbf{.604}${\tiny $\pm .00$}
\\ \bottomrule
\end{tabular}
}
\vshrink{-0.2em}
\else
\begin{tabular}[h]{ccccccc}
\toprule 
Metric & Linear & ANM & Flow & D+Ens & D+BS & D+IPB \\ 
\midrule
CRPS 
&         $.738${\tiny $\pm .10$}
&         $.551${\tiny $\pm .01$}
&         $.546${\tiny $\pm .02$}
&         $.520${\tiny $\pm .00$}
& $\mathbf{.518}${\tiny $\pm .00$}
& $\mathbf{.518}${\tiny $\pm .00$}
\\ 
Abs.~Err
&         $.658${\tiny $\pm .03$}
&         $.655${\tiny $\pm .01$}
& $\mathbf{.605}${\tiny $\pm .02$}
&         $.609${\tiny $\pm .01$}
&         $.611${\tiny $\pm .01$}
& $\mathbf{.604}${\tiny $\pm .00$}
\\ \bottomrule
\end{tabular}
\vshrink{-0.75em}
\fi
\end{table}

\section{Conclusion}\label{sec:conclusion}
\vshrink{-0.4em}

We studied %
uncertainty quantification using general ML algorithms, starting from the postulation that commonly used 
algorithms may be near-Bayes optimal on an unknown task distribution. 
We proved in simplified settings that 
it is possible to recover the unknown but optimal Bayesian posterior 
by constructing a martingale posterior, %
and proposed a novel method which is applicable across NN and non-NN models.  Experiments %
confirmed the efficacy of the method.

Our work has various limitations, which we discuss in detail in App.~\ref{app:addi-disc}. %
Briefly, it would be interesting to investigate the use of ML algorithms that satisfy weaker conditions for stability and efficiency, as well as %
stochastic algorithms that may have an imperfect notion of parameter uncertainty. 
We hope that our results demonstrate the potential of the algorithmic %
perspective for Bayesian uncertainty quantification, %
and that it may inspire further investigation in this direction. 

\bibliographystyle{IEEEtranN}
\bibliography{main}

\appendix
\onecolumn

\addcontentsline{toc}{section}{Appendix} %
\part{\Large Appendix} %
\parttoc %
\section{Additional Discussions}\label{app:all-disc}

\subsection{Discussion, Limitation and Future Work}\label{app:addi-disc}

\paragraph{Broader context of the theoretical contributions.}
We presented an analysis of MPs defined by sample-efficient estimation algorithms 
and investigated their application to modern ML algorithms. 
Our work has various limitations which we discuss shortly. 
However, we first clarify on the broader context of the theoretical contributions, and the main direction we hope to contribute to. 

The theoretical analysis aims at provide better justification for the use of general ML algorithms in quantifying parameter (i.e., epistemic) uncertainty. 
While it is possible to quantify ``subjective uncertainty'' using any base algorithm, 
e.g., by plugging it into Alg.~\ref{alg:main}, the resulted uncertainty will not always be useful: 
the base algorithm could be grossly misspecified for the present data distribution $p_0$, or 
the uncertainty estimates could also be \emph{incoherent} in which case downstream decision-making may be uniformly suboptimal regardless of what $p_0$ is \citep{heath1978finitely,savage1972foundations}. 
For these reasons, the user should seek to 
provide additional %
justification for their choices of ML algorithm and uncertainty quantification scheme, %
beyond the tautological %
argument %
that the result represents their subjective uncertainty; 
just as a user of standard Bayesian models should justify their %
choices through additional conceptual %
reasoning or empirical diagnostics \citep{gelman2020bayesian}. 

The end goal of the analysis is to allow 
users to justify their choices by reasoning about the algorithm's estimation performance on similar tasks. %
The reasoning process could be grounded in empirical evidence derived from 
real or synthetic datasets. 
It can also be conceptual, as a thought experiment that allows the user to elucidate  %
their algorithmic choices. 
In its weakest form, a result of this form will still allow user to understand that they can obtain approximately coherent uncertainty estimates as long as the base algorithm can be assumed to be near optimal %
w.r.t.~any hypothetical task distribution; this is in the spirit of \citet{dawid1999prequential}. Our result is also a step forward from \citet{fong_martingale_2021}, as we allow for a wider range of %
base algorithms that do not necessarily define a coherent predictive distribution on their own. 

The analysis assumes the base algorithm is near-optimal for point estimation; 
it is reasonable to ask how we expect to improve over such an algorithm. 
As shown in \S\ref{sec:intro} and \S\ref{sec:exp}, by better accounting for epistemic uncertainty we can still improve its predictive performance, which  
be also viewed as achieving near-optimality w.r.t.~a more stringent criterion (e.g., 
from square loss for parameter estimation to log loss for prediction).\footnote{
Moreover, note that compared with the base algorithm, ensemble prediction employs a different action space, so near-optimality w.r.t.~the same loss function could also be a stronger requirement.  
} 
We further emphasise that 
the task of \emph{epistemic uncertainty quantification is fundamentally more challenging} than 
prediction (of a single test sample): it can be viewed as modelling the joint distribution of $(z_{n+1}, z_{n+2}, \ldots)$ as opposed to the marginal distribution of $z_{n+1}$.\footnote{
The naive estimate $p_{\hat\theta_n}\otimes p_{\hat\theta_n}\ldots$ is suboptimal from a Bayesian perspective: the optimal (posterior) predictive distribution is correlated. It is also uncalibrated \citep{johnson_experts_2024}, which is relevant beyond the Bayesian perspective. 
} 
The practical utility of epistemic uncertainty has been extensively discussed. 
Here we note the following example, %
which is closely related to the %
joint modelling view above: suppose we want to 
model the average effect of a policy deployed to a population of individuals distributed as $p_0$.

\paragraph{Limitation and future work.}
As we noted in \S\ref{sec:theory}, the 
analysis intends to provide intuition by studying simplified scenarios, and its assumptions can be restrictive for practical applications. 
We first note that some 
restrictions are merely made to simplify presentation, and we expect that they can be relaxed with some effort.  For example, 
Asm.~\ref{asm:stable} and Asm.~\ref{asm:martingale-divergence} only need to hold in a neighbourhood around the true $\theta_0$; 
it should also be straightforward to provide a conditional analogue of \Cref{thm:param-alt} that does not average over $z_{1:n}\sim \pi$ if we modify the definition of $\BayesError$ to be conditional on the observed data.\footnote{
The relaxation will allow us to understand the behaviour of the MP on $\pi$-null sets, as long as we can reason about the true posterior's behaviour on such events. 
Such a discussion is important in classical statistics, since the prior $\pi$ can be misspecified: it is imposed by the user, who needs to know its (analytical) form and be able to conduct approximate inference with it. 
It appears less relevant in our motivating setup, where $\pi$ is assumed to be ``correctly specified'' and the user does not need to have exact knowledge about it. 
}

A main technical limitation in \S\ref{sec:theory} is the restriction to \eqref{eq:MP-for-analysis}: 
the requirement that $\EstParam[n+1]$ does not depend on $z_{1:n}$ except through $\EstParam[n]$ 
will rule out many practical algorithms. 
We note that for regular parametric models, online natural gradient 
has the form of \eqref{eq:MP-for-analysis} and always provides a near-optimal estimator \citep[\S 12.1.7]{amari2016information}; 
for high-dimensional models, preconditioned GD may fulfil a similar purpose. %
The ultimate purpose of \eqref{eq:MP-for-analysis} is to ensure stability: 
it guarantees that the \emph{internal state} of $\Alg$ can be summarised into a tractable space---the parameter space---so that further assumptions (\ref{asm:stable}) could quantify stability. %
Weaker notions of algorithmic stability %
have been extensively studied in literature %
\citep{bousquet2000algorithmic,rogers1978finite,liu2017algorithmic,bassily2020stability}, 
but it certainly requires substantial effort to bring them into our framework. 
It is interesting to note that Bayesian algorithms (that maps $z_{1:j}$ to a sample from $\pi(\theta\mid z_{1:j})$) 
can always be viewed as an online algorithm with a form similar to \eqref{eq:MP-for-analysis}: its ``state'' can be summarised by the posterior distribution, given which it becomes independent of past data. %

Through \eqref{eq:MP-for-analysis} we also restrict to deterministic algorithms. %
In some scenarios it may be preferable to employ stochastic algorithms, based on which we can construct a better approximation to the unknown posterior mean. 
As discussed in \S\ref{sec:method}, the MP scheme can be applied to fully Bayesian algorithms which produce a stochastic parameter estimate \citep{fong_martingale_2021}. 
In combination with our results it indicates that the MP is ``robust'' at two extremes where the base algorithm either quantify no parameter uncertainty at all or maintains a fully coherent notion of parameter uncertainty. %
It may thus be reasonable to expect that MPs can also be constructed out of base algorithms with an imperfect notion of parameter uncertainty, e.g., those based on approximate Bayesian inference. 
Our proof appears to suggest that any variation in $\Alg[j](\EstParam[j-1], \EstData[j])\mid \EstData[\le j]$, %
would have a higher-order effect, as guaranteed by the stability of the algorithm (see in particular the application of Asm.~\ref{asm:martingale-divergence}). 

The efficiency assumption is central to the analysis. 
It is natural to expect that similar conditions %
may be unavoidable for results like \eqref{eq:main-asymp}. 
The assumption could be more easily satisfied if we restrict to smaller $N$ (\S\ref{app:gp-mtl}) or weaker choices of $\|\cdot\|$. 
Conceptual examples for the latter include semi-norms that focus on the comparison between likelihood functions indexed by parameters (\Cref{rem:identifiability}) and semi-norms that ignore differences between \emph{nuisance parameters} \citep[Ch.~25]{van2000asymptotic}. 
It would be interesting if predictive efficiency could be quantified through more general means than vector semi-norms. 
It would also be interesting to investigate whether prediction algorithms could define an approximately coherent notion of uncertainty (e.g., approximating a model that defines conditionally identically distributed samples \citep{Berti2004,fong_martingale_2021}) in a broader range of scenarios. 

A main limitation with our methodology is the need to specify a distribution of inputs for supervised learning tasks. 
\Cref{rem:supervised-learning} discussed several choices; nonparametric resampling appears to be effective in our experiments, 
and for high-dimensional structured inputs we may employ pretrained generative models. 
We also note that this is a shared limitation with previous works on function-space Bayesian inference for deep models \citep{sun2018functional,wang2018function,ma2019variational}. 
For large-scale NN models the need to maintain an ensemble of parameters would also be limiting; it would be interesting to explore the use of parameter-efficient finetuning methods \citep{ding2023parameter,dusenberry2020efficient,yang2023bayesian} for this issue. 

\subsection{Related Work}\label{app:related-work}

Our work is motivated by challenges of designing and implementing Bayesian counterparts for ML methods. 
As discussed in \S\ref{sec:intro}, NN methods may constitute an important example, due in part to the challenges in inference and prior specification. 
Another issue is the choice of likelihood: 
applications in computer vision and natural language processing often involve loss functions that do not have a likelihood interpretation \citep{lin2017focal,li2019dice}, 
and even when a likelihood-based objective leads to efficient point predictors, 
its suitability for Bayesian NNs can still be debatable if the application involves 
human-annotated datasets \citep{aitchison2020statistical} or data augmentation \citep{nabarro2022data}.\footnote{
See also the works of \citet{wenzel2020good,izmailov2021bayesian} who reported performance issues with Bayesian NNs (with Gaussian priors) in the presence of data augmentation.
}
Compared with versatility and flexibility %
of non-Bayesian deep learning, these issues suggest that 
in typical deep learning applications, it can often be easier to express the 
``prior knowledge'' about what method is best suited for a given problem through algorithms, rather than through explicitly defined Bayesian models.

Our work provides an efficient %
ensemble method for uncertainty quantification. 
Many ensemble methods have been proposed for NN models \citep[to name a few]{lakshminarayanan2017simple,osband2018randomized,wang2018function,liu2016stein,d2021repulsive,wang2021scalable}. 
Our method stands out for its applicability beyond NN models, 
while it also retains advantages over the bootstrap aggregation method---known for a similar trait---by more effectively leveraging the parametric model when it is available (App.~\ref{app:cmp-bagging}). 
It may be interesting to build an ensemble of ensemble predictors using Alg.~\ref{alg:main}. %

The GP example in \S\ref{sec:ex-lingauss} is connected to the ensemble algorithms in \cite{osband2018randomized,he2020bayesian,pearce20a}, which are designed for DNNs but motivated from the same GP regression setting. As observed in \cite{he2020bayesian}, the GP example is relevant in a deep learning context given the connection between ultrawide NNs and GPs \citep{lee2019wide}. 
While GP regression serves as an interesting motivating example, 
the ultrawide NNs in that literature represent %
an oversimplified model which does not allow for feature learning \citep{chizat2019lazy},  and 
should not be viewed as a ``correct prior'' for NNs \citep{aitchison2020bigger}. 
Yet to ensure a match to the GP posterior, those ensemble methods involve design choices that may not be generally beneficial, such as an $\ell_2$ regularisation with a fixed $n^{-1}$ scaling. 
Our method is motivated from a more general perspective, but we also compare with \cite{pearce20a,he2020bayesian} empirically; see \Cref{app:toy-exp-gp} and \Cref{tbl:dj-synth-main}. 
We also note that the specific problem of (conjugate) GP inference is by now well-understood; there exist algorithms with good statistical and computational guarantees \citep{burt2019rates,nieman2022contraction}. 

We focus on uncertainty quantification for near-Bayes optimal algorithms. 
This is closely related to %
recent works that explicitly train predictive models on a mixture of synthetic or real datasets 
so that they may approach the Bayes-optimal predictor \citep{finn2017model,garnelo2018conditional,muller_transformers_2023}. 
Our work is different in its applicability to models not explicitly trained in this way, 
and importantly we provide concrete theoretical guarantees for epistemic uncertainty quantification. 
As discussed in \S\ref{sec:intro} and App.~\ref{app:addi-disc}, 
epistemic uncertainty quantification is a more difficult task than (single-sample) prediction, and 
algorithms that are near-optimal for prediction may have no sense of epistemic uncertainty at all (e.g., MLE). 
It is generally interesting to investigate the quantification of epistemic uncertainty using pretrained predictive models. 
Note that neural processes \citep{garnelo2018neural} have a coherent notion of epistemic uncertainty, 
but different from our approach it is unclear if they can recover the true Bayesian posterior defined by the pretraining distribution. 
However, it is interesting to note the connection \citep{rao1971projective} between Kolmogorov extension theorem, 
the key invariance property of neural processes, and Doob's theorem which underlies the construction of MPs. 

We reviewed previous work on martingale posteriors in \S\ref{sec:bg}, and  
our methodology is most related to \cite{fong_martingale_2021} and \cite{holmes_statistical_2023}. 
\cite{fong_martingale_2021} imposes a coherence condition (see their condition 2) that requires the base algorithm to define the same predictive distribution as the MP. The MP is thus a tool for inference that \emph{reveals} the epistemic uncertainty in the base algorithm. This %
is a non-trivial accomplishment, since the algorithm is accessed as a black box; but 
the coherence requirement does rule out the use of common algorithms %
such as sequential MLE with a non-categorical likelihood.  
\cite{holmes_statistical_2023} studies more general algorithms %
beyond the coherence case, but the only theoretical guarantee provided is that the MP defined by \eqref{eq:gd-martingale} may have a variance scaling of $\cO(1/n)$. 
This does not cover non-GD algorithms, and does not justify the application of GD to multidimensional models ($\dim\theta>1$) as there is no guarantee about the shape of the covariance. 
By introducing the postulation \eqref{eq:multitask-main}
we are able to cover a broader range of algorithms and provide more complete justification for all of them. 
The postulation is related to the works of \citet{dawid1999prequential,skouras_efficient_1998,xu2022minimum}. 

Our result is also related to the work of \citet{efron2012bayesian} who connected parametric bootstrap to a specific Bayesian posterior defined by the Jeffreys prior \citep{jeffreys1998theory}. However, 
the Jeffreys prior has %
counterintuitive behaviours when $\dim\theta>1$ \citep[see e.g.,][]{syversveen1998noninformative}, and cannot be defined for infinite-dimensional models as in \S\ref{sec:ex-lingauss}. 
There is also a literature on statistical inference with GD and bootstrap resampling (see \cite{lam_resampling_2023} and references therein), which studies similar but different algorithms to the example \eqref{eq:gd-martingale}. 
Such works have the different goal of recovering the sampling distribution for regular parametric models ($d<\infty$ does not grow w.r.t.~$n$), which is not relevant beyond that setting (see \Cref{app:wn}). 

\subsection{Comparison with Bootstrap Aggregation}\label{app:cmp-bagging}

The proposed method is broadly similar to bootstrap aggregation (bagging) methods: 
both build an ensemble of model parameters by estimating on perturbed versions of the training set. 
Bagging can be implemented using parametric or nonparametric bootstrap. 
In practice, parametric bootstrap is rarely used in ML, possibly because the algorithm discards the training observations in resampling which is considered undesirable; it also performs worse in our experiments.  
Here we present two simplified examples which may provide additional insight. 

\begin{example}[comparison to nonparametric bootstrap]\label{ex:bnp-cmp}
Suppose %
$z_{1:n}\sim \cN(\theta_0, I)$ with $d:=\dim z_i$ satisfying $n\ll d\ll n^2$. Let 
Alg.~\ref{alg:main} be defined with $\Delta n=1, N\gg n$ and the sequential MLE algorithm as $\cA$. 
It follows by \S\ref{sec:ex-expfam} that $\PP(\hat\theta_N\,\vert\, z_{1:n})=\cN(\hat\theta_n, C_n)$ for some $C_n\sim %
I$. This distribution quantifies a non-trivial amount of uncertainty in the $(d-n)$-dimensional null space of the empirical covariance $\frac{1}{n}\sum_{i=1}^n (z_i - \bar z_i)(z_i - \bar z_i)^\top$. 
In contrast, the sampling distribution of nonparametric bootstrap has no variation in this subspace, falsely indicating %
complete confidence in the subspace where the data does not provide any information at all. 
\end{example}

\begin{example}[comparison to parametric bootstrap]
Consider a two-dimensional dataset generated by $z_{i,1}\sim\mrm{Bern}(1-\epsilon), z_{i,2}\vert z_{i,1}\sim \cN(\theta_{z_1=z_{i,1}}, 1)$. With $n = \lfloor\epsilon^{-1}/2\rfloor$ the expected number of samples with $z_{i,1}=0$ is $<1$, %
so there should be substantial uncertainty about $\theta_{z_1=0}$. %
Yet parametric bootstrap may underestimate the uncertainty: the probability of a resampled dataset $D_n^{(k)}$ containing no samples with $z_{i,1}=0$ is $(1-\epsilon^{-1})^n \sim e^{-1/2}$, in which case there may not be any meaningful variation in the respective estimate, %
$\hat\theta^{(k)}_{z_1=0}$ e.g., if the estimation algorithm applies a small regularisation. %
In contrast, our method with $N\gg n$ will update all $\hat\theta^{(k)}$ with probability $1-(1-\epsilon^{-1})^N\to 1$. 
\end{example}

The above examples are clearly oversimplified. In practice, initialisation randomness in optimisation will also contribute to the uncertainty estimates and may help narrow the gap between these procedures %
\citep{lakshminarayanan2017simple}. 
Still, the examples illustrated how our method has a more direct impact on the final uncertainty estimates, especially in aspects of the parameter which the training data is not informative about. 

\section{Deferred Proofs}\label{app:proofs}

In the proofs we adopt the following additional notations: we use $\EE_\pi$ to denote the expectation w.r.t.~data sampled from the prior predictive distribution; formally, for any $j\in\mb{N}$ and any integrable function $g: \cZ^{\otimes j}\to\RR$ we define 
$
\EE_\pi g(z_{1:j}) := \EE_{\theta_0\sim\pi, z_{1:j}\overset{iid}{\sim} p_{\theta_0}} g(z_{1:j}).
$ 
For all $j\ge n$, define 
$$
\BayesData[j+1] \sim \pi(z_{j+1} \mid z_{1:n}, \BayesData[n+1:j]), ~~
\BayesParam[j] := \EE_{\theta\sim\pi(\theta\mid z_{1:n}, \BayesParam[n+1:j])} \theta. 
$$
Note that when $z_{1:n}$ follow the prior predictive distribution, $(\theta_0, z_{1:n}\cup \BayesData[n+1:j], \BayesParam[j])$ will have the same distribution as the random variables $(\theta_0, z_{1:j}, \BayesParam[j])$ defined in \eqref{eq:bayes-err-defn}. Thus, for such $z_{1:n}$,  
$\BayesError[j]^2$ will continue to represent the mean square error of $\BayesParam[j]$ and the squared radius of the Bayesian posterior, as stated in the text. 
We use $\cF_j$ to denote the $\sigma$-algebra generated by ``all observations up to iteration $j$'', including $\{z_{1:n}, \EstData[n+1:j], \BayesData[n+1:j]\}$ as well as an additional set of $\{\followData[n+1:j]\}$ that will be defined shortly. 
Define $$
\EE_j := \EE(\cdot\mid\cF_j), ~~~~ \BayesDelta[j] := \BayesParam[j] - \BayesParam[j-1].
$$ 
We will also make frequent use of the inequality 
\begin{align}
  \|a+b\|^2 = \|a\|^2 + \|b\|^2 + 2\<\delta^{1/2}a, \delta^{-1/2}b\> 
  \le (1+\delta)\|a\|^2 + (1+\delta^{-1})\|b\|^2, \label{eq:basic}
\end{align}
which holds for all vector semi-norms, $a,b$ and $\delta>0$. 
In particular, this implies $\|a+b\|^2 \le 2(\|a\|^2+\|b\|^2)$. It also follows that, for any 
$\{\cF_j\}$-adapted $\{a_j\}$ and any collection of random $\{b_j\}$, 
\begin{align}
  \EE_j\|a_j + b_j\|^2 &= \|a_j\|^2 + \EE_j\|b_j\|^2 + 2\<\delta^{1/2}a_j, \EE_j\delta^{-1/2}b_j\>  \nonumber \\ 
  &\le (1+\delta)\|a_j\|^2 + \EE_j\|b_j\|^2 + \delta^{-1}\|\EE_j b_j\|^2. \label{eq:basic-rvec}
\end{align}

\subsection{Proof for Theorem~\ref{thm:param-alt}}\label{app:proof-thm-main}

By \cref{asm:replaced-efficiency} 
it suffices to prove \eqref{eq:main-nonasymp}. 
Observe that the following always holds: %
    \begin{align*}
    \EE_\pi W_2^2(\pi_n, \hat p_{mp,n}) &\le \EE_\pi\|\EstParam[N] - \BayesParam[\infty]\|^2   \\ 
    &= \EE_\pi\|\EstParam[N] - \BayesParam[N]\|^2 + \EE_\pi\|\BayesParam[N] - \BayesParam[\infty]\|^2  \\ 
    &\le 2\EE_\pi(\|\EstParam[N] - \followParam[N]\|^2 + \|\followParam[N] - \BayesParam[N]\|^2) +  \EE_\pi\|\BayesParam[N] - \BayesParam[\infty]\|^2 \\ 
    &= 2\EE_\pi\|\EstParam[N] - \followParam[N]\|^2 + 2\ExcessError[N]^2 + \BayesError[N]^2.\numberthis\label{eq:PA-goal-new}
    \end{align*}
    Thus, to prove \eqref{eq:main-nonasymp} it suffices to establish that %
\begin{equation}\label{eq:PA-goal-full}
  \EE_\pi\|\EstParam[N] - \followParam[N]\|^2 \le e^{\chi_n}((\chi_n + \nu_n)\BayesError[n]^2 + \ExcessError[n]^2)
\end{equation}
where $\chi_n$ is to be defined below. 

To prove \eqref{eq:PA-goal-full}, 
we will construct a sequence of couplings between $\{\EstData\}$ and $\{\BayesData\}$ which determines a joint distribution for $(\EstParam[N], \BayesParam[N])\mid\cF_n$ that allows \eqref{eq:PA-goal-new} to be bounded as desired. 
For this purpose, we will introduce an additional r.v.~$\followData[j+1]$ s.t.~$\PP(\followData[j+1]\in\cdot\mid\cF_j) = \PP_{\followParam[j]}(\cdot)$, and 
couple $(\EstData,\BayesData)$ through 
the joint distribution $
\PP(\followData[j+1], \EstData[j+1], \BayesData[j+1]\mid\cF_j) = \PP(\followData[j+1]\mid\cF_j) \PP(\EstData[j+1]\mid \followData[j+1],\cF_j) \PP(\BayesData[j+1]\mid\followData[j+1],\cF_j)
$ with the last two terms determined by various optimal transport plans. 

Let $s>0$ be defined in \cref{asm:replaced-efficiency}. For any $n\le j<N$, 
consider the decomposition 
\begin{align*}
&\phantom{=} \EE_j\|\EstParam[j+1] - \followParam[j+1]\|^2 \\
&= 
  \EE_j\|\EstParam[j] + \EstDelta[j](\EstParam[j], \EstData[j+1]) - 
    (\followParam[j] + \EstDelta[j](\followParam[j], \followData[j+1]) 
    -  \EstDelta[j](\followParam[j], \followData[j+1]) + \EstDelta[j](\followParam[j], \BayesData[j+1]))\|^2 \\ 
&\overset{\eqref{eq:basic-rvec}}{\le} (1+j^{-(1+s)})\EE_j\|\EstParam[j] - \followParam[j] - (\EstDelta[j](\followParam[j], \followData[j+1]) - \EstDelta[j](\followParam[j], \BayesData[j+1]))\|^2 \\ 
&\hspace{2em} + \EE_j\| \EstDelta[j](\EstParam[j], \EstData[j+1]) - \EstDelta[j](\followParam[j], \followData[j+1])\|^2 %
+ j^{1+s}(\|\EE_j(\EstDelta[j](\EstParam[j], \EstData[j+1])- \EstDelta[j](\followParam[j], \followData[j+1]))\|^2)
    \\ 
&\le (1+j^{-(1+s)})\EE_j\|\EstParam[j] - \followParam[j] - (\EstDelta[j](\followParam[j], \followData[j+1]) - \EstDelta[j](\followParam[j], \BayesData[j+1]))\|^2 \\ 
&\hspace{2em} + \EE_j\| \EstDelta[j](\EstParam[j], \EstData[j+1]) - \EstDelta[j](\followParam[j], \followData[j+1])\|^2 %
+ j^{1+s}(2\|\EE_j \EstDelta[j](\EstParam[j], \EstData[j+1])\|^2 + 2
\|\EE_j \EstDelta[j](\followParam[j], \followData[j+1])\|^2)
    \\ 
&=: (1+j^{-(1+s)})A_j + B_j + j^{1+s} C_j. \numberthis\label{eq:new-proof-intermediate}
\end{align*}
We will bound the three terms in turn. 

\underline{For $C_j$}, we note that since $s<\delta$ (Asm.~\ref{asm:replaced-efficiency}), Asm.~\ref{asm:approx-martingale} also holds for $\delta = s$, and thus we have 
\begin{equation}
 j^{1+s} C_j \le 
 2j^{-(1+s)} \BayesError[j]^2.  \label{eq:NPI-C}
\end{equation}

\underline{For $B_j$}, first note that by \cref{asm:stable}~(i) we have
\begin{align*}
B_j &\le 2 (
    \EE_j \|\EstDelta(\EstParam[j], \EstData[j+1]) - \EstDelta(\followParam[j], \EstData[j+1])\|^2 + 
    \EE_j \|\EstDelta(\followParam[j], \EstData[j+1]) - \EstDelta(\followParam[j], \followData[j+1])\|^2) \\ 
&\le 2\eta_j^2(L_1^2 \|\EstParam[j] - \followParam[j]\|^2 + L_2^2 \EE_j \|\EstData[j+1] - \followData[j+1]\|_z^2) \numberthis\label{eq:NPI-B0}.
\end{align*}
Let $\PP(\EstData[j+1]\mid \cF_j,\followData[j+1])$ be defined by the optimal transport plan that minimises the transport cost above. Recall that \cref{asm:stable}~(ii) states that one of the following must hold:
\begin{align}
W_{2,z}^2(p_\theta, p_{\theta'}) &\le C_\Theta \|\theta - \theta'\|^2, ~~~\text{ or} \label{eq:A3}\\
W_{2,z}^2(p_\theta, p_{\theta'}) &\le C_\Theta \|\theta - \theta'\|, ~~\eta_j\le j^{-(3+\iota)/4}.\tag{\ref{eq:A3}'}\label{eq:A3-weaker}
\end{align}
If \eqref{eq:A3} holds, the above will be bounded by $2\eta_j^2 (L_1^2 + L_2^2 C_\Theta) \|\EstParam[j] - \followParam[j]\|^2$, and we have $\eta_j \le j^{-(1+\iota)/2}$. Otherwise, by \eqref{eq:A3-weaker} we have $j^{1/4}\eta_j \le j^{-(1+\iota)/2}$ and 
\begin{align*}
2\eta_j^2 L_2^2 \EE_j \|\EstData[j+1] - \followData[j+1]\|_z^2 &
    \le 2\eta_j^2 L_2^2 C_\Theta \|\followParam[j] - \EstParam[j]\| 
    \\ &
    = L_2^2 C_\Theta \cdot 2j^{-1/2}(j^{1/4}\eta_j) \cdot (j^{1/4}\eta_j) \|\followParam[j] - \EstParam[j]\|
\\ &\le L_2^2 C_\Theta \cdot\bigl(
    (j^{-1/2}(j^{1/4}\eta_j))^2 + 
    (j^{1/4}\eta_j \|\followParam[j] - \EstParam[j]\|)^2
\bigr)  %
\\ &= L_2^2 C_\Theta \cdot(j^{1/4}\eta_j)^2 \bigl(\|\followParam[j] - \EstParam[j]\|^2 + j^{-1}\bigr) 
\\ &\le L_2^2 C_\Theta \cdot(j^{1/4}\eta_j)^2 (\|\followParam[j] - \EstParam[j]\|^2 +  \BayesError[j]^2).  && \mcomment{Asm.~\ref{asm:conventions}}
\end{align*}
Define $\eta'_j := j^{-(1+\iota)/2}$, then in both cases we have 
\begin{equation}\label{eq:NPI-B1}
2\eta_j^2 L_2^2 \EE_j \|\EstData[j+1] - \followData[j+1]\|_z^2 \le 
L_2^2 C_\Theta
\eta'^2_j (\|\followParam[j] - \EstParam[j]\|^2 +  \BayesError[j]^2).
\end{equation}
Plugging back to \eqref{eq:NPI-B0} we have 
\begin{align*}
B_j
&\le 
    2 \eta'^2_j (L_1^2 + L_2^2 C_\Theta)(\|\EstParam[j] - \followParam[j]\|^2 + \BayesError[j]^2). \numberthis\label{eq:NPI-2}
\end{align*}
\underline{For $A_j$}, we first use \eqref{eq:basic-rvec} to bound it as 
\begin{align*}
A_j &\le  \EE_j((1+j^{-(1+s)})\|\EstParam[j] - \followParam[j]\|^2 + \|\EstDelta[j](\followParam[j], \followData[j+1]) - \EstDelta[j](\followParam[j], \BayesData[j+1])\|^2) 
\\ & \hspace{4em}
+ j^{1+s}
  \|\EE_j(
\EstDelta[j](\followParam[j], \followData[j+1]) - 
    \EstDelta[j](\followParam[j], \BayesData[j+1]) 
  )\|^2 
\\
&\le
(1+j^{-(1+s)})\|\EstParam[j] - \followParam[j]\|^2 + \EE_j\|\EstDelta[j](\followParam[j], \followData[j+1]) - \EstDelta[j](\followParam[j], \BayesData[j+1])\|^2 
\\ & \hspace{4em}
+ j^{1+s} C_j + 
2j^{1+s}
  \|\EE_j
    \EstDelta[j](\followParam[j], \BayesData[j+1]) 
  \|^2. 
  \numberthis\label{eq:alt-decomp-1}
\end{align*}
We now bound the second and last terms above. For the second term we introduce our coupling between $(\followData[j+1], \BayesData[j+1])\mid \cF_j$ as follows. Recall the conditional distribution $\BayesData[j+1]\mid\cF_j$ can be represented as $\theta\sim \pi(\theta\mid\cF_j), ~\BayesData[j+1]\sim p_{\theta}$; we thus define $\PP(\BayesData[j+1]\mid \cF_j,\EstData[j+1])$ through 
\begin{equation}\label{eq:alt-bayes-coupling}
\theta\sim \pi(\theta\mid\cF_j), ~~ \BayesData[j+1]\mid(\theta, \followData[j+1]) \sim \Gamma_{p_{\followParam[j]}\to p_{\theta}}(\cdot\mid \followData[j+1]),
\end{equation}
where $\Gamma_{P\to Q}$ denotes the conditional probability derived from the optimal transport plan from $P$ to $Q$. Clearly this preserves both marginal distributions as required, and we have 
\begin{align*}
\EE_j\|\EstDelta[j](\followParam[j], \followData[j+1]) - \EstDelta[j](\followParam[j], \BayesData[j+1])\|^2 
    &\le \eta_j^2 L_2^2 \EE_j\|\followData[j+1] - \BayesData[j+1]\|_z^2  && \mcomment{Asm.~\ref{asm:stable}~(i)} \\ 
    &\overset{\eqref{eq:alt-bayes-coupling}}{\le} \eta_j^2 L_2^2 \EE_{\theta\sim\pi(\cdot\mid\cF_j)} W_2^2(p_{\followParam[j]}, p_\theta). 
\end{align*}
Repeating the proof for \eqref{eq:NPI-B1} we find the above is bounded as 
\begin{align*}
\eta_j^2 L_2^2 \EE_{\theta\sim\pi(\cdot\mid\cF_j)} W_2^2(p_{\followParam[j+1]}, p_\theta) 
    &\le L_2^2 C_\Theta \eta'^2_j(\EE_{\theta\sim\pi(\cdot\mid\cF_j)} \|\followParam[j] - \theta\|^2 + \BayesError[j]^2) \\ 
    &= L_2^2 C_\Theta \eta'^2_j(\ExcessError[j]^2 %
     + 2\BayesError[j]^2), \numberthis\label{eq:alt-decomp-11}
\end{align*}
where the last line follows from the fact that $\theta\mid\cF_j\overset{d}{=}\BayesParam[\infty]\mid\cF_j$. Now, turning to the last term of \eqref{eq:alt-decomp-1}, we have 
\begin{align*}
    &\phantom{=}\|\EE_j \EstDelta[j](\followParam[j], \BayesData[j+1])\|^2 \\ 
     &\overset{\eqref{eq:alt-bayes-coupling}}{=} 
    \|\EE_{\theta\sim\pi(\cdot\mid\cF_j)}\EE_{z\sim p_\theta}
        \EstDelta[j](\followParam[j], z) 
    \|^2  \\
    &= 
    \|\EE_{\theta\mid\cF_j}\EE_{z\mid\theta}
        (\EstDelta[j](\followParam[j], z) - \eta_j H_{\followParam[j]}(\theta-\followParam[j])
        +  \eta_j H_{\followParam[j]}(\theta-\followParam[j]))
    \|^2 \\
    &\le 
    2\|\EE_{\theta\mid\cF_j}\EE_{z\mid\theta}
        (\EstDelta[j](\followParam[j], z) - \eta_j H_{\followParam[j]}(\theta-\followParam[j]))\|^2
        +  2\|\mblue{\EE_{\theta\mid\cF_j}}\eta_j H_{\followParam[j]}(\mblue{\theta}-\followParam[j])
    \|^2 \\
    &\le 
    2(\mgreen{\EE_{\theta\mid\cF_j}\|\EE_{z\sim p_\theta}
        \EstDelta[j](\followParam[j], z) - \eta_jH_{\followParam[j]}(\theta-\followParam[j])\|})^2
        +  2\|\eta_j H_{\followParam[j]}
            (\mblue{\BayesParam[j]}-\followParam[j])\|^2   \\
    &\le
    2(\EE_{\theta\mid\cF_j} \mgreen{\Calg\eta_j\|\followParam[j] - \theta\|^2})^2 + 
    2\Calg'\eta_j^2 \ExcessError[j]^2 
    && \mcomment{Asm.~\ref{asm:martingale-divergence}}
    \\ 
    &= 
      2\eta_j^2 \Calg^2(\ExcessError[j]^2 + \BayesError[j]^2)^2  +
    2\Calg'\eta_j^2 \ExcessError[j]^2  
    \le 
    4\eta_j^2(\Calg'^2\ExcessError[j]^2  + \Calg\BayesError[j]^4).  && \mcomment{Asm.~\ref{asm:conventions}~\ref{it:convention-scale}}
    \numberthis\label{eq:AD1-new}
    \end{align*}
Plugging \eqref{eq:AD1-new} and \eqref{eq:alt-decomp-11} into \eqref{eq:alt-decomp-1}, we have 
\begin{align*}
    A_j &\le 
  (1+j^{-(1+s)})\|\EstParam[j] - \followParam[j]\|^2 +  
    \eta_j'^2 L_2^2 C_\Theta (\ExcessError[j]^2  + 2\BayesError[j]^2) \\ & \hspace{8em} 
    + 
  8j^{1+s} \eta_j^2 (\Calg'\ExcessError[j]^2 + \Calg\BayesError[j]^4)  + j^{1+s} C_j\\ 
  &\le
  (1+j^{-(1+s)})\|\EstParam[j] - \followParam[j]\|^2 +  
     C_\Theta' (\eta_j'^2\BayesError[j]^2 + 
  j^{1+s}\eta_j^2 (\ExcessError[j]^2  + \Calg\BayesError[j]^4)) + j^{1+s} C_j,
  \numberthis\label{eq:AD1-resolved}
  \end{align*}
where the constant $C_\Theta'$ is determined by $L_1,L_2,C_\Theta$ and $\Calg'$.
Plugging \eqref{eq:AD1-resolved}, \eqref{eq:NPI-2} and \eqref{eq:NPI-C} into \eqref{eq:new-proof-intermediate} and taking expectation, we find 
\begin{align*}
    \EE_\pi\|\EstParam[j+1]-\followParam[j+1]\|^2 &\le 
    (1 + 2j^{-(1+s)} + \eta_j'^2 C_\Theta'')\EE_\pi \|\EstParam[j]-\followParam[j]\|^2 \\ 
    &\hspace{2em}+ %
C_\Theta'' (\eta_j'^2\BayesError[j]^2 + j^{1+s}\eta_j^2 (\ExcessError[j]^2  + \Calg\BayesError[j]^4)) %
+ 4 j^{-(1+s)} \BayesError[j]^2 %
\end{align*}
where $C_\Theta''$ is a constant similarly determined by $(L_1,L_2,C_\Theta,\Calg')$. 
Define $\Delta\chi_j := 2j^{-(1+s)} +  C_\Theta''\eta_j'^2, \chi_l := \sum_{j=l}^N \Delta\chi_j$. Then 
$\chi_l \lesssim 1/(sn^s) + 1/(\iota n^\iota)\lesssim 1/(sn^s)$ as claimed, and we have 
\begin{align*}
&\phantom{=} \EE_\pi\|\EstParam[j+1]-\followParam[j+1]\|^2  \\
&\le
     e^{\Delta \chi_j} \EE_\pi \|\EstParam[j]-\followParam[j]\|^2 + 
C_\Theta'' (\eta_j'^2\BayesError[j]^2 + j^{1+s}\eta_j^2 (\ExcessError[j]^2  + \Calg\BayesError[j]^4)) + 4 j^{-(1+s)} \BayesError[j]^2,  \\
&\phantom{=}
\EE_\pi\|\EstParam[N]-\followParam[N]\|^2 \\ 
 &\le e^{\chi_n} \biggl(
    \EE_n\|\EstParam[n] - \followParam[n]\|^2 + 
    \sum_{j=n}^{N} 
  C_\Theta'' (\eta_j'^2\BayesError[j]^2 + 
j^{1+s}\eta_j^2 (\ExcessError[j]^2  +\Calg\BayesError[j]^4)) + 4 j^{-(1+s)} \BayesError[j]^2
\biggr) \\ 
&\le 
 e^{\chi_n} %
    (\EE_n\|\EstParam[n] - \followParam[n]\|^2 + 
      C(\chi_n +  \nu_n)\BayesError[n]^2),
\end{align*}
where the last inequality follows by Asm.~\ref{asm:replaced-efficiency}, \ref{asm:conventions}~\ref{it:convention-asymp} and the constant $C$ is determined by $C_\Theta''$.  
This completes the proof.\qed

\subsection{Discussion of Credible Set Approximations}\label{app:theory-additional-cb}

We prove the following statement which substantiates the claim made below \Cref{thm:param-alt}:

\begin{corollary}\label{claim:w2-cred-set}
For any $A\subset\Theta$ and $\delta>0$, define the ``enlarged'' set 
$$
A_\delta := \{\theta'\in\Theta: ~~\exists\theta\in A~s.t.~\|\theta-\theta'\|\le \delta\}. 
$$
Then 
\begin{enumerate}[label=(\roman*),leftmargin=*]
\item Let $\epsilon>0,\gamma\in(0,1)$ be arbitrary, $(p,q)$ be any pair of distributions over $\Theta$ s.t.~$W_{2,\theta}(p,q) \le \epsilon, $
  and $A_\gamma\subset\Theta$ be any set s.t.~$p(A_\gamma)=1-\gamma$. 
  Then for any $t>0$, we have $q(A_{\gamma,t^{-1/2}\epsilon})\ge 1-\gamma-t$. 
\item When \eqref{eq:main-asymp} holds, there exist some $\delta_n\ll \BayesError[n]^2$ s.t.~the following statement holds on a $\cF_n$-measurable event with probability $\to 1$: for all $\gamma\in (0,1)$ and $\cF_n$-measurable $A_\gamma\subset\Theta$ s.t.~$\hat p_{mp,n}(A_\gamma)=1-\gamma$, we have $\pi_n(A_\gamma)\ge 1-\gamma-t_n$ where $t_n=o_n(1)$. 
\end{enumerate}
\end{corollary}
\begin{proof}
(i): by definition of $W_{2,\theta}$ there exists a distribution $\Gamma(\theta_p, \theta_q)$ s.t.~the marginal distributions for $\theta_p$ and $\theta_q$ are $p$ and $q$ respectively, and $\EE_{\Gamma}\|\theta_p-\theta_q\|^2 \le \epsilon^2$. Thus, 
\begin{align*}
q(A_{\gamma,\delta}) &= \Gamma(\theta_q\in A_{\gamma,\delta})\ge \Gamma(\theta_p\in A_{\gamma,\delta}, \|\theta_p-\theta_q\|\le t^{-1/2}\epsilon) \\
&\ge p(A_\gamma) - \Gamma(\|\theta_p-\theta_q\|> t^{-1/2}\epsilon)  
\overset{(a)}{\ge} 1-\gamma - \frac{\EE_\Gamma\|\theta_p-\theta_q\|^2}{\epsilon^2} \ge 1-\gamma - t.
\end{align*}
In the above, (a) follows by Chebyshev's inequality. 

(ii) Define 
$\omega_n := (\EE_\pi W_2^2(\hat p_{mp,n},\pi_n))^{1/2}$ so that $\omega_n\ll \BayesError[n]$ by \eqref{eq:main-asymp}. 
Another application of Chebyshev's inequality yields 
$\PP_{\pi}(W_2(\hat p_{mp,n},\pi_n)\le \omega_n^{1/2} \BayesError^{1/2}) = 1-o_n(1)$. 
Restricting to this event and applying (i) with $t \gets \omega_n^{-1/4}\BayesError[n]^{1/4}$ completes the proof.
\end{proof}

\subsection{Deferred Proofs in Section~\ref{sec:theory-examples}}\label{app:examples-derivations}

\subsubsection{Proof for the claims in \Cref{sec:ex-expfam}}\label{app:ex-expfam-deriv}

The following claim immediately implies that in the setting of \S\ref{sec:ex-expfam} 
\Cref{asm:replaced-efficiency} holds for all $s<\min\{1,\delta\}$ and $\nu_l\le 2\alpha l^{-1+s}$, as claimed.
\begin{claim}
  In the setting of Sec.~\ref{sec:ex-expfam} we have $\ExcessError[j]^2 \le 2\alpha j^{-1} \BayesError[j]^2.$ 
\end{claim}

\begin{proof}
It follows by our choice of $\pi$ that 
$$
\BayesParam[j] = \frac{j \followParam[j] + \theta_\pi}{j+\alpha} = 
\BayesParam[j-1] + \frac{1}{j+\alpha}(\BayesData[j] - \BayesParam[j-1]). 
$$
To bound $\BayesError$ we use the above representation, and the fact that $\{\BayesParam\}$ define a martingale; it follows that 
\begin{align*}
\BayesError[j]^2 &= \EE_\pi \|\BayesParam[j] - \BayesParam[\infty]\|^2 
=  \sum_{k=j}^\infty\EE_\pi \|\BayesParam[k] - \BayesParam[k+1]\|^2 
=  \sum_{k=j}^\infty\EE_\pi \frac{\|T(\BayesData[k+1]) - \BayesParam[k]\|^2}{(k+\alpha)^2}.
\end{align*}
Observe that $\PP(\BayesData[k+1]\in dz\mid \BayesParam[k]) = \int \PP_{\tilde\theta_k}(dz) \pi_{k,\BayesParam[k]}(d\tilde\theta_k)$, where $\pi_{k,\BayesParam[k]}(d\theta) = \pi(\theta\mid \BayesData[\le k])$ is the posterior measure, and is \emph{determined by} the posterior mean $\BayesParam[k]$: the posterior for natural parameter is $\pi(\eta\mid \BayesData[\le k]) \propto \exp((k+\alpha)\eta^\top \BayesParam[k] - (k+\alpha) A(\eta))$, and $\pi(\theta\mid \BayesData[\le k])$ is merely its pushforward by $\nabla A$. 
Therefore, we have $\BayesData[k+1]\indep \BayesParam[k]\mid \tilde\theta_k$, and 
\begin{align*}
\EE\|T(\BayesData[k+1]) - \BayesParam[k]\|^2 &= %
\EE\|T(\BayesData[k+1]) - \tilde\theta_k\|^2 + \EE\|\tilde\theta_k - \BayesParam[k]\|^2 + 
\EE\<(T(\BayesData[k+1]) - \tilde\theta_k\mid \tilde\theta_k, \cancel{\BayesParam[k]}), \tilde\theta_k - \BayesParam[k]\> \\
&\overset{(i)}{=} \EE\|T(\BayesData[k+1]) - \tilde\theta_k\|^2 + \EE\|\tilde\theta_k - \BayesParam[k]\|^2 \\ 
&\ge \EE\|T(\BayesData[k+1]) - \tilde\theta_k\|^2 \\
&\overset{(ii)}{=} \EE_{\theta\sim\pi,z\sim\PP_\theta}\|T(z) - \theta\|^2 =: V_\pi.
\end{align*}
In the above, (i) holds because $\tilde\theta_k$ is the mean parameter for $\BayesData[k+1]$, 
and (ii) holds because the marginal distributions for all posterior samples $\tilde\theta_k$ equal the prior. Plugging back, we find 
\begin{align*}
\BayesError[j]^2 &\ge  \sum_{k=j}^\infty \frac{V_\pi}{(k+\alpha)^2} \ge 
\frac{1}{j+\alpha}V_\pi. 
\end{align*}

For $\ExcessError$, we have 
\begin{align*}
\EE_\pi\|\EstParam[j] - \theta_0\|^2 &= 
    \EE_{\theta\sim\pi, z_{1:j}\sim\PP_{\theta}^{\otimes j}}(\EE(\|\EstParam[j] - \theta\|^2\mid\theta)) \\ 
&= 
    \EE_{\theta\sim\pi, z_{1:j}\sim\PP_{\theta}^{\otimes j}}\biggl(\EE\biggl(\biggl\|
    \frac{1}{j}\sum_{k=1}^j T(z_k)
    - \theta\biggr\|^2~\bigg|~\theta\biggr)\biggr) \\ 
&= \EE_{\theta\sim\pi, z\sim\PP_{\theta}}  \frac{\|T(z)-\theta\|^2}{j} = \frac{1}{j} V_\pi,
\end{align*}
where the last equality follows from conditional independence. 
It thus follows that 
$$
\ExcessError[j]^2 %
\le \frac{\alpha}{j(j+\alpha-1)}V_\pi \le \frac{\alpha}{j} \cdot \frac{j+\alpha}{j+\alpha-1} \BayesError[j]^2 
\le \frac{2\alpha}{j}\BayesError[j]^2.
$$
This completes the proof.
\end{proof}

\begin{claim}\label{claim:bounded-Fisher}
Let $F_\theta$ denote the Fisher information matrix for $p_\theta$. %
In the setting of Sec.~\ref{sec:ex-expfam}, 
\Cref{thm:param-alt} holds if %
$(\|T\|_\infty:=\sup_{z\in\cZ}\|T(z)\|, \sup_{\theta}\lambda_{max}(F_\theta), \sup_{\theta}\lambda_{min}^{-1}(F_\theta))$ are all bounded.  
\end{claim}
\begin{proof}[Proof for \Cref{claim:bounded-Fisher}]
Observe that \Cref{thm:param-alt} will continue to hold if we replace all occurrences of $z$ with $T(z)$ (and the norm $\|\cdot\|_z$ with $\|\cdot\|$) in its proofs and assumptions: this is because both the MP and the Bayesian posterior only depend on $z$ through $T(z)$. 
Therefore, to prove the claim it suffices to establish \Cref{asm:stable}~(ii)--or Eq.~\eqref{eq:A3-weaker}--after the replacement. The equation holds because 
\begin{align*}
  &\phantom{=}W_2^2(T_{\#} p_\theta, T_{\#} p_{\theta'})  \\
  &\le 2\sup_{z,z'<\infty}\|T(z)-T(z')\|^2 D_{TV}(T_{\#} p_\theta, T_{\#} p_{\theta'}) && \mcomment{\citealp{villani2009optimal}, Theorem 6.15} 
  \\ 
  &\le  8\|T\|_\infty^2 D_{TV}(p_\theta, p_{\theta'}) \\ 
  &\le 8\|T\|_\infty^2 \sqrt{\mathrm{KL}(p_\theta, p_{\theta'})/2} && \mcomment{Pinsker's inequality} \\ 
  &=  8\|T\|_\infty^2 \sqrt{A(\eta') - A(\eta) - \nabla A(\eta)^\top (\eta'-\eta)} \\ 
  &\le  4\sqrt{2} \|T\|_\infty^2 (\sup_{\tilde\eta}\|\nabla^2 A(\tilde\eta)\|_{op})^{1/2} \|\eta-\eta'\| \\ 
  &\le 4\sqrt{2} \|T\|_\infty^2 \sup_{\tilde\eta}\|\nabla^2 A(\tilde\eta)\|_{op}^{1/2} 
  (\sup_{\tilde\eta'}\|(\nabla^2 A(\tilde\eta'))^{-1}\|_{op}) \|\theta-\theta'\|.
\end{align*}
In the above, $\eta=(\nabla A)^{-1}(\theta), \eta'=(\nabla A)^{-1}(\theta')$ are the respective natural parameters, $T_{\#}$ denotes the pushforward measure, the LHS is the replaced LHS of \eqref{eq:A3-weaker}, and the coefficients in the RHS are bounded by assumptions, in particular because 
$\nabla^2 A(\eta) = F_\theta^{-1}$. %
This completes the proof.
\end{proof}

We note that it should be possible to replace the uniform boundedness conditions with their local counterparts (that only holds in a neighbourhood of $\theta_0$); the resulted conditions can be used to establish a conditional version of the theorem (which can be easily proved by adapting the existing proof). We omit the discussion for brevity.

Finally, we substantiate on the claims about specific exponential family models: 
for Gaussian model \eqref{eq:A3} holds because the transport plan is $z\mapsto z+\theta'-\theta$; for $\{Exp(\theta)\}$ \eqref{eq:A3} holds by considering the transport plan $z\mapsto \frac{\theta'}{\theta} z$. For the Bernoulli model we can establish \eqref{eq:A3-weaker} using the first two inequalities in the above proof. 

\subsubsection{Deferred proofs and additional discussion for Section \ref{sec:ex-lingauss}}\label{app:wn}

\newcommand{\regCov}[1]{\hat\Sigma_{#1}}

\paragraph{Connection to nonparametric inverse problems and regression.} \Cref{sec:ex-lingauss} is closely connected to the following inverse problem:
\begin{equation}\label{eq:nip}
\bar z_n = A \theta_0 + n^{-1/2} W, ~~\text{  where  } W\sim \cN_\cZ(0, I).
\end{equation}
Indeed, we can recover the above problem by setting $\bar z_n := \frac{1}{n}\sum_{i=1}^n z_i$. 
The latter is the classical (nonparametric) linear inverse problem; %
see \citet{cavalier2008nonparametric} for a review. 
Strictly speaking, our setup is different from \eqref{eq:nip} as we observe $\{z_i\}$, but \emph{the difference is irrelevant} to our discussion, since we can verify that 
both the MP and %
the Bayesian posterior %
only depend on $\{z_i\}$ through $\bar z_n$ and are thus applicable to \eqref{eq:nip}.

When $\alpha=1$, the problem can be equivalently stated as $\bar z_n = \theta'_0 + n^{-1/2} W$ where $\theta'_0 := A\theta_0$; and the norm of interest becomes $\|\hat\theta-\theta_0\| = \|A\hat\theta - \theta'_0\|_\cZ$. This is the signal-in-white noise problem which is asymptotically equivalent to regression \citep{brown1996asymptotic}. The prior $\pi$ for $\theta$ corresponds to the GP\footnote{see \citet{van2008reproducing} for a definition of GPs in Hilbert spaces.} prior $\pi' := \cN_\cZ(0, AA^\top)$ for $\theta'$. Such priors are ``infinitesimally weaker'' than assuming $\theta'_0$ to live in $S^{2\beta-1} := \{\theta' = \sum_i i^{-(2\beta-1)/2}a_i \psi_i\text{ for some }\{a_i\}\in\ell_2(\mb{N})\}$ where $\{\psi_i\}$ denotes the left singular vectors of $A$, as $\theta'\sim\pi'$ will fall into $S^{2\beta-1-\epsilon}$ a.s.~for all $\epsilon>0$ \citep{van2008reproducing}. The spaces $S^{(\cdot)}$ are 
known as \emph{Sobolev classes} \citep[see e.g.,][]{cavalier2008nonparametric} %
and can recover the $L_2$-Sobolev spaces for suitable choices of $\beta$ and $\{\psi_i\}$. 

\paragraph{Inapplicability of MLE / natural gradient.} 
For both \eqref{eq:nip} and the data generating process in \Cref{sec:ex-lingauss}, the MLE $\hat\theta_n$ satisfies $A\hat\theta_n = \bar z_n = \frac{1}{n}\sum_{i=1}^n z_i$. When $\alpha=1$, the estimation error $\|\hat\theta_n - \theta_0\|$ thus equals the \emph{dimensionality} of $\cZ$, and is unbounded if the dimensionality is so; 
the same applies to the natural gradient algorithm with $\eta_j = j^{-1}$ due to its exact equivalence to MLE in this scenario. 
In contrast, the Bayesian estimator have a bounded error (see \eqref{eq:bayes-error-wn} below) due to its regularisation effect. 

\paragraph{Validating the assumptions for the linear-Gaussian MP.} Observe that the posterior equals $$
\pi(\theta\mid z_{\le j}) = \cN(\theta\mid \regCov{j}^{-1} A^\top \bar z_j, (j\regCov{j})^{-1}),
$$
where $
\regCov{j} := A^\top A + j^{-1} I, ~
\bar z_j := \frac{1}{j}\left(\sum_{i=1}^n z_i + \sum_{i=n+1}^j \BayesData[i]\right),
$ and 
$A^\top$ denotes the adjoint. And we have 
\begin{equation}\label{eq:bayes-error-wn}
\BayesError[j]^2 = \mrm{Tr} ((A^\top A)^\alpha (j\regCov{j})^{-1})
    = \sum_{i=1}^\infty \frac{s_i^{2\alpha} }{j s_i^2 + 1}
    \asymp %
     j^{-1} + %
     j^{-\alpha} m_j, 
\end{equation}
where $m_j := \max\{m\in\mb{N}: s_m^2 \ge j^{-1}\}\asymp j^{1/2\beta}$. 
We have introduced the Hilbert spaces $\cH,\cZ$ and defined the parameter norm $\|\theta\| := \|(A^\top A)^{\alpha/2}\theta\|_\cH =: \|S \theta\|_\cH$. 
In instantiating the theorem we will set the data norm as 
$
\|z\|_z := \|(AA^\top)^{(\alpha-1)/2} z\|_\cZ.
$

We now verify the assumptions in turn. 
\begin{enumerate}[leftmargin=*]
\item \Cref{asm:approx-martingale} holds for all $\delta>0$ because $\Alg$ defines an exact martingale.
\item \Cref{asm:stable} holds because for its~(i), we have
\begin{align*}
\|\EstDelta(\theta, z) - \EstDelta(\theta', z)\|^2 &= 
\|S(\EstDelta(\theta, z) - \EstDelta(\theta', z))\|_\cH^2 \\ 
&=
    \|j^{-1} g_j(A^\top A) A^\top A S(\theta-\theta')\|_\cH^2 \le j^{-2} \|\theta-\theta'\|^2, \\
\|\EstDelta(\theta, z) - \EstDelta(\theta, z')\|^2 &= 
\|S\cdot j^{-1} g_j(A^\top A) A^\top (z-z')\|_\cH^2  \\ 
&\le j^{-2} \|(A^\top A) g_j(A^\top A)\|_{op}^2 \|(AA^\top)^{(\alpha-1)/2}(z-z')\|_\cZ^2 
\le j^{-2}\|z-z'\|_z^2.
\end{align*}
And for its condition~(ii), 
\begin{align*}
W_2^2(p_\theta, p_{\theta'};\|\cdot\|) &= \|A\theta - A\theta'\|^2 = \|\theta-\theta'\|^2.%
\end{align*}
\item To verify \cref{asm:replaced-efficiency} we first prove that 
$$
\EstDelta[j](\BayesParam, \BayesData) = \BayesDelta[j].
$$ 
This is because there exist independent rvs 
$
e_i \sim \cN(0, \sigma^2 I), ~\Delta e_i \sim \cN(0, j^{-1} A \regCov{j}^{-1} A^\top)
$
s.t.~for $
\bar e_i := e_i + \Delta e_i,
$ we can have 
\begin{align*}
\BayesDelta[j] &=
    \regCov{j}^{-1} A^\top \left(\frac{j-1}{j}\bar z_{j-1} + \frac{1}{j}(A\BayesParam[j] + \bar e_j)\right) - \regCov{j-1}^{-1} A^\top \bar z_{j-1} %
= j^{-1} \regCov{j}^{-1}A^\top \bar e_j %
= \EstDelta(\BayesParam, \BayesData). %
\end{align*}
Since we also have $\followParam[n] = \BayesParam[n]$, it follows by induction that $\followParam[j] = \BayesParam[j]$ for all $j\ge n$. 
Thus, $\ExcessError \equiv 0$, and the assumption holds for $\nu_l \equiv 0$. 
\item \Cref{asm:martingale-divergence} holds for $\Calg=0,\Calg'=1$ and $\eta_j=j^{-1}$ because 
$$
\EE_{z'\sim\PP_{\theta'}} \EstDelta[j](\theta, z') = j^{-1} \underbrace{g_j(A^\top A) A^\top A}_{=: H_{\theta,j}} (\theta'-\theta).
$$
\item \Cref{asm:conventions} holds when $\alpha=1$ since $\BayesError[j]^2\asymp j^{-1+1/2\beta}$. 
It also holds for a range of $\alpha$ depending on the value of $\beta$. 
\end{enumerate}

\paragraph{(Non-asymptotic) connections to GP regression.}
Consider a GP model with input space $\cX$, prior $\pi_{gp} = \mc{GP}(0, k)$ and likelihood $p(y\mid f(x))=\cN(f(x), 1)$. 
Let $\bar\cH$ be the reproducing kernel Hilbert space (RKHS) defined by $k$, $\{(x_1,y_1),\ldots,(x_n,y_n)\}$ be the training data, and $K := (k(x_i,x_j))_{ij}\in\RR^{n\times n}$ be the %
Gram matrix. 
Introduce the notations %
$f(X) := (f(x_1);\ldots;f(x_n))\in\RR^n$ and 
$Y := (y_1;\ldots;y_n)\in\RR^n$. 
Let $\cH\subset\bar\cH$ be the subspace spanned by $\{k(x_i,\cdot)\}_{i=1}^n$ with the inherited norm. Then we can identify the projection of any $f\in\bar\cH$ onto $\cH$ with $f(X)$, and its norm satisfies $\|f(X)\|_{\cH}^2 = f(X)^\top K^{-1} f(X)$. Let $\cZ=\RR^n$ be equipped with the Euclidean norm.  
We substitute the remaining quantities in \cref{sec:ex-lingauss} as follows:
$$
\theta = f(X), ~~~
A\theta = %
\frac{1}{\sqrt{n}} f(X), ~~~
\frac{1}{n}\sum_{i=1}^n z_i = \frac{1}{\sqrt{n}} Y.
$$
Then it is clear that $\theta$ follows the prior $\pi$ and the conditional distribution $\frac{1}{n}\sum_{i=1}^n z_i \mid \theta$ equals that defined by the likelihood  in \cref{sec:ex-lingauss}, and  
we can readily verify that the %
posterior in Sec.~\ref{sec:ex-lingauss} for $\theta=f(X)$ equals the GP marginal posterior. 
Following \cref{sec:ex-lingauss}, we can consider an MP defined by 
\eqref{eq:lingauss-alg} and $\EstData[j] \sim \cN(\EstParam, n^{-1} I)$, which provides a high-quality approximation to the GP marginal posterior. 

As noted above, on $\{z_j\}$ sampled from the prior predictive distribution 
\eqref{eq:lingauss-alg} has a behaviour equivalent to sequential posterior mean estimation which, for linear-Gaussian Bayesian models, is equivalent to sequential maximum-a-posteriori (MAP) estimation. Based on the same idea of sequential MAP estimation we can derive the update rule \eqref{eq:gp-alg-spo} for GP regression. Note that 
\eqref{eq:gp-alg-spo} and \eqref{eq:lingauss-alg} are not an exact match because the GP MAP also depends on the sampled $\hat x_j$. (If we continue the analogy above, \eqref{eq:gp-alg-spo} can be viewed as an MAP in a Bayesian model where we impute at all $n$ input locations simultaneously in each iteration, and scale the resulted log likelihood by $1/\sqrt{n}$.) 
Nonetheless, we expect their behaviour to be similar. A separate analysis for \eqref{eq:gp-alg-spo} may be possible, but we forego this discussion given the rich literature on GP inference. Instead, we refer readers to \Cref{app:toy-exp-gp} for an empirical evaluation for \eqref{eq:gp-alg-spo}. 

\begin{remark}\label{rem:gp-ood}
The above discussion restricted to the marginal posterior $f(X)\mid (X,Y)$ and does not cover predictive uncertainty in out-of-distribution (OOD) regions. 
We note that for models that define continuous prediction functions, the uncertainty for $f(X)$ always translates to some uncertainty in OOD regions due to the continuity constraint; the MP will also provides additional uncertainty if we sample $\hat x_j$ from the OOD regions. However, an equally important source of OOD uncertainty is from the model's \emph{initialisation randomness}, which can be fully characterised in the GP example above. 

To see this, consider an MP defined by \eqref{eq:gp-alg-spo} and the choice of $\hat x_{j+1} \sim \mrm{Unif}\{x_{1:n}, \hat x_{n+1:j}\}$. We claim that the resulted algorithm will fully retain the initialisation randomness for %
uncertainty in OOD regions. Formally, for any $f\in\bar\cH$, or an interpolating RKHS which %
cover all GP samples \citep{steinwart2019convergence}, and any $x_*\in\cX$, we can decompose 
$f(x_*) = f_{\parallel}(x_*) + f_{\perp}(x_*)$ by projecting $f =: f_{\parallel} + f_{\perp}$ into $\cH$ and its orthogonal complement.
Then %
the GP posterior for $f_\parallel$ and $f_\perp$ are then independent, and the latter is equivalent to the prior; this is because the likelihood is independent of $f_\perp$. %
The MP update admits a similar factorisation for the same reason, and thus any initialisation randomness will be retained in the MP, and an exact match to the GP posterior can be possible if we initialise based on the GP prior. 
\end{remark}

\section{Implementation Details for Algorithm~\ref{alg:main}}\label{app:impl-details}

\paragraph{Choices of $\Delta n$ and $N$.} 
If the base algorithm is ``correctly specified'' for the problem as hypothesised, we should ideally choose $\Delta n$ and $N$ to match the exact martingale posterior ($\Delta n=1, N\to\infty$) as close as possible, but computational constraints may prevent an exact match. 
A larger $\Delta n$ or a smaller $N$ generally leads to an underestimation of uncertainty. 

We note that no adjustment is needed if, as 
in many applications, the goal is merely to improve predictive performance by better accounting for epistemic uncertainty, 
since the algorithm can still account for a substantial proportion of the uncertainty; 
and similar underestimation issues may also emerge in the applications of approximate Bayesian inference to %
complex models, when due to computational constraints we cannot recover the exact posterior. 
Nonetheless, for the construction of credible sets, 
we provide a rule of thumb to %
compensate for this effect by analysing simplified settings. 
Specifically, consider the natural GD algorithm
\begin{equation}\label{eq:natgd-martingale}
\EstParam[j+1] := \EstParam[j] + (j+1)^{-1} F_{\EstParam[j]}^{-1} \nabla_\theta \log p_{\EstParam[j]}(\EstData[j+1]),
\end{equation}
where $F_\theta$ denotes the Fisher information matrix. 
Suppose $n/\Delta_n\in\mb{N}$ for simplicity, then 
the covariance of the parameter ensemble from \Cref{alg:main} is 
\begin{equation}\label{eq:alg1-cov}
\sum_{j'=n/\Delta_n}^\infty \frac{\Delta_n }{((j'+1)\Delta_n)^2} F_{\EstParam}^{-1}
\approx 
\sum_{j'=n/\Delta_n}^\infty \frac{\Delta_n }{((j'+1)\Delta_n)^2} F_{\theta_0}^{-1}
\sim \left(\frac{1}{n+\Delta_n} - \frac{1}{N+\Delta_n}\right) F_{\theta_0}^{-1}.
\end{equation}
The exact MP has covariance $\sim n^{-1}F_{\theta_0}^{-1}$, so 
to match the exact MP it suffices to inflate the covariance by a factor 
$\sim \frac{\Delta n}{n} + \frac{n}{N}$. 
The same inflation applies to credible sets for linear functionals of the parameter which, 
for linear-in-parameter regression models, include  
pointwise credible intervals for the true regression function. 
Note that the same adjustment applies to any GD algorithms with a step-size of $\eta_j\sim j^{-1}$, which is generally related with sequential ERM algorithms (and thus Alg.~\ref{alg:main}) as shown in \Cref{sec:theory-examples}. 
And the above discussion is relevant in a deep learning context if we consider ultrawide NNs \citep{lee2019wide}. 

In reality, we expect the adjustment to produce conservative credible sets for NN-based algorithms, since it also (unnecessarily) inflates the initialisation randomness. 
However, the scale of the adjustment is generally small, and together with the unadjusted credible sets they can provide a two-sided bound for the predictive uncertainty. 

In our experiments we adopt $N\asymp n\asymp \Delta n$ where the ratios $(N/n, n/\Delta n)$ are in the range of $[1, 10]$, and determine the adjustment scale by explicitly numerical approximation of the ratio between the coefficient of \eqref{eq:alg1-cov} and $n^{-1}$. 
For base algorithms that are potentially misspecified we determine the ratio through cross validation. 

\paragraph{Early stopping for NN-based algorithms.} 
While the objective \eqref{eq:nn-objective} always prevent overfitting to past samples, we still need to determine the number of optimisation iterations for the new samples $\EstData[n_j:n_j+\Delta n]$. 
In our experiments we use a simple strategy: 
we use a validation set to determine the number of iterations $L$ for estimation on the $n$ real samples, 
and optimise for $L \Delta n/n$ iterations when ``finetuning'' on (each group of) $\Delta n$ synthetic samples. Other optimisation hyperparameters are also kept consistent across the initial estimation and finetuning. 

\section{Experiment Details and Additional Results}\label{app:exps}

This section provides full details for the experiments in the text, and two additional experiments on GP inference. 

\subsection{Toy Experiment: 1D Gaussian Process Regression}\label{app:toy-exp-gp}

We first evaluate the proposed method on a toy GP regression task, to understand its behaviour and 
complement the GP discussion in \Cref{sec:ex-lingauss}. 

\begin{figure}[htb]
\centering 
\includegraphics[width=0.85\linewidth,clip,trim={0 0.25cm 0 0.25cm}]{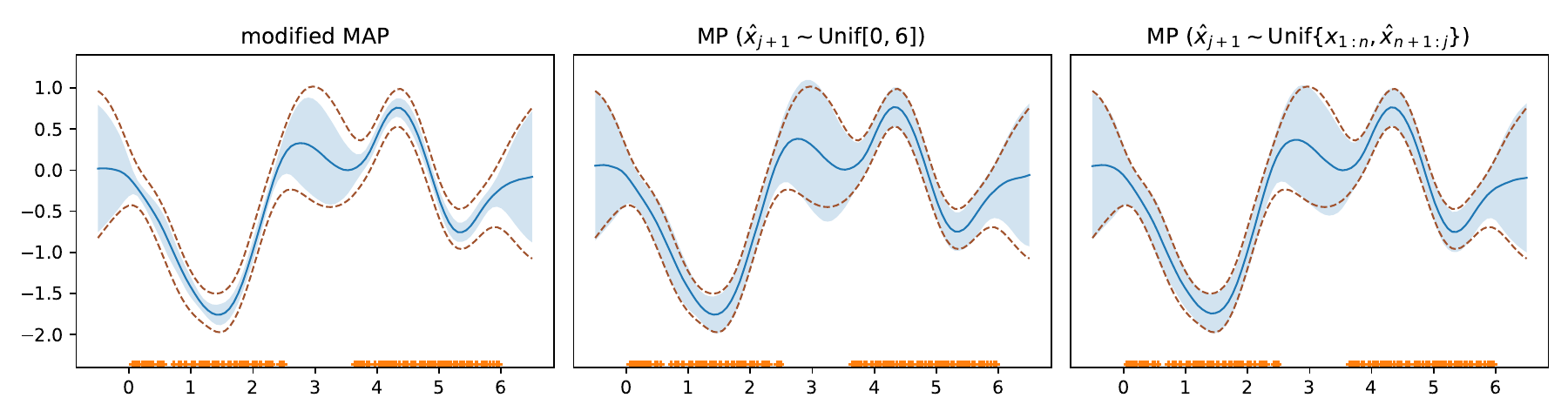}
\caption{GP inference on the Snelson dataset: visualisation of the approximate MP defined by Eq.~\eqref{eq:gp-alg-spo}, compared with the ensemble predictors defined by a modified MAP estimator with similar initialisation randomness (Eq.~\eqref{eq:map-anchoring}). Solid line and shade indicate the mean estimate and $80\%$ pointwise credible intervals (CIs) for the true regression function. Dashed line indicates the $80\%$ CIs from the exact posterior. Dots at bottom indicate the location of training inputs. 
}\label{fig:gp-toy}
\end{figure}

\paragraph{Experiment setup.} We instantiate \Cref{alg:main} using \eqref{eq:gp-alg-spo} as the estimation algorithm and 
random Fourier approximation \citep{rahimi2007random} for the RKHS. 
We adopt the one-dimensional Snelson dataset \citep{snelson2008flexible} and remove the samples with input within the $[0.4, 0.6]$ quantile to create an out-of-distribution region for visualisation. 
We adopt a Mat\'ern-$3/2$ kernel with bandwidth $1$ approximated with $400$ random Fourier features, and specify a Gaussian likelihood with variance $\sigma^2=0.64$. We set $N=6n, \Delta n=0.05n$ in \Cref{alg:main}, and consider two choices for $\hat x_j$: (i) uniform sampling from $[0, 6]$, and (ii) nonparametric resampling as in \Cref{rem:gp-ood}. 
We compare with an ensemble of modified MAP predictors, proposed by \citet{pearce20a}:
\begin{equation}\label{eq:map-anchoring}
\hat f_n := \argmin_{f} \sum_{i=1}^n (f(x_i)-y_i)^2 + \frac{\sigma^2}{n}\|f-\tilde f_0\|_\cH^2, ~~~\text{ where }\tilde f_0\sim\mc{GP}(0, k_x),
\end{equation}
and $k_x$ denotes the Mat\'ern kernel. 
Compared with standard MAP estimation, the random $\tilde f_0$ provides an additional source of initialisation randomness which is 
also needed for the MP to match the exact Bayesian posterior in out-of-distribution regions (\Cref{rem:gp-ood}). 
\eqref{eq:map-anchoring} is also analogous to the deep ensemble method \citep{lakshminarayanan2017simple} in which epistemic uncertainty is similarly derived solely from initialisation randomness. %
For all methods we compute the closed-form optima. 

\paragraph{Results and discussion.}
\Cref{fig:gp-toy} visualises the predictive uncertainty from the MP, the modified MAP ensemble, and the exact posterior. We can see that the MP produces a close match to the GP posterior, as expected in \Cref{sec:ex-lingauss}; 
and the results are highly consistent across the two choices of samplers for $\hat x_j$. 
In contrast, \eqref{eq:map-anchoring} underestimates uncertainty, especially in in-distribution regions. 
While conjugate GP inference is a well-studied problem, 
the above result suggests that in more general scenarios, the uncertainty derived from our method may also have a more desirable behaviour than that from methods relying solely on initialisation randomness. 
We will observe such results in the DNN experiments in \Cref{app:exp-dj}.

\subsection{Synthetic Multi-Task Learning Experiment}\label{app:gp-mtl}

We now turn to a synthetic setup where the MP defined by \eqref{eq:gp-alg-spo} is instantiated with a kernel learned from multi-task data. 

\paragraph{Background: few-shot multi-task learning in a stylised setting.} The setup is inspired from a line of theoretical work \citep{tripuraneni2020theory,du2020few,tripuraneni2021provable,wang2022fast} that studied multi-task learning in a stylised setting and showed that,  
given a number of i.i.d.~pretraining tasks sampled from a task distribution $\pi$, 
it is possible to learn a linear representation space (i.e., a finite-dimensional RKHS) that allows for sample-efficient learning on identically distributed test tasks. 
These results suggest that in such settings our theoretical analysis may guarantee the approximate recovery of the optimal posterior $\pi_n = \pi(\cdot\mid z_{1:n})$, since 
the base algorithm \eqref{eq:gp-alg-spo} instantiated with a learned RKHS may satisfy the efficiency assumption (Asm.~\ref{asm:replaced-efficiency}) in \S\ref{sec:theory}, following which the discussions in \S\ref{sec:ex-lingauss} will apply. 
Such a result will provide an interesting stylised example where the challenge of uncertainty quantification can be addressed by exploiting pretraining data.

Previous works \citep{du2020few,wang2022fast} showed that in certain regimes test-time prediction using the learned RKHS attains order-optimal errors. 
Our Asm.~\ref{asm:replaced-efficiency} requires the prediction error to be first-order optimal \emph{up to a sample size} of $N>n_{test}$. Thus, we expect it to hold in scenarios closer to \emph{few-shot learning}, where the test task has a smaller sample size. 
We will validate both Asm.~\ref{asm:replaced-efficiency} and the conclusion of \Cref{thm:param-alt} empirically, on a synthetic data distribution inspired by \citep{wang2022fast}. 

\paragraph{Experiment setup.} We consider regression tasks with additive noise and known variance. 
All tasks share a latent feature space $\bar\cX$, and are determined by a feature-space prediction function 
$\bar g: \bar\cX\to\RR$. 
Each task defines a data distribution $p_{\bar g}(x,y)$ as follows:
\begin{equation}
\bar g\sim \mc{GP}(0, \bar k), ~~\bar x = \begin{bmatrix}
    \bar x_{true} \\ \bar x_{spurious}
\end{bmatrix} \sim \cN(0, I), 
~~y\mid \bar x, \bar g\sim \cN(\bar g(\bar x_{true}), \sigma_0^2), 
~~x = \Phi(\bar x). 
\end{equation}
In the above, 
$\bar x$ denotes the unobserved latent features, 
$\bar k$ is a reproducing kernel in the latent space, and the function $\Phi$ is the same across all tasks. 
Representation learning thus amounts to learning the composition of 
the feature-space kernel $\bar k$ and the feature extraction function $\Phi^{-1}$. 
We note that both the values of $(\bar k, \Phi)$ and their structural form (e.g., the fact that $\bar k$ is an RBF kernel, or $\Phi$ is defined by a DNN with a certain architecture) are unknown to the learner. 
Instead, the learner simply invokes the algorithm in \cite{wang2022fast} on the pretraining dataset, which trains a DNN model with $m$ prediction heads (one for each pretraining task) and defines a kernel $\hat k$ using the linear predictions as the feature map. At test time, the learner invokes the base prediction algorithm \eqref{eq:gp-alg-spo} with the RKHS $\cH$ defined by $\hat k$. 

We generate $m$ pretraining tasks, each with $n_{pret}$ observations, and an identically distributed test task with $n_{test}$ observations. We define $\Phi$ as a randomly initialised multi-layer perceptron (MLP) with $3$ hidden layers and a width of $128$, and instantiate the kernel learning algorithm in \cite{wang2022fast} using an MLP with $4$ hidden layers and a width of $256$. The MLPs are defined with swish activation. %
(We note that the MLP model in kernel learning is not guaranteed to be correctly specified since it needs to model the inverse of $\Phi$.) 
We set $\dim \bar x_{true}=1, \dim\bar x_{spurious}=3, \dim x=10$ and $\bar k$ to be an RBF kernel with bandwidth set to the input median. 
We vary $m\in \{100,200,400\}, n_{pret}\in\{5,10,20,40\}\times 100$ and $n_{test}\in \{5,10,20,40\}$. 
For kernel learning, the MLP is optimised using the AdamW optimiser \citep{loshchilov2018fixing}, with learning rate determined from $\{1,5,10,50\}\times 10^{-4}$ and number of iterations from $\{1,2,4\}\times 1000$ 
based on validation loss; other optimisation hyperparameters follow the default in Optax \citep{deepmind2020jax}. 
Given the learned kernel we compute \eqref{eq:gp-alg-spo} in closed form. 
We implement Alg.~\ref{alg:main} using $\Delta n=\max\{1, 0.05n\}$ and $N=12n$. 

For evaluation, we generate inputs as $\{x_{eval,i} := \Phi([\bar x_{true, i}; 0])\}$ where $\{\bar x_{true,i}\}$ denote a linearly spaced grid of 10 points from $-2.25$ to $2.25$. 
$\{x_{eval,i}\}$ determine an empirical $L_2$ (semi-)norm $\|\cdot\|$ for the regression function $g$; we validate our theoretical claims against this choice of $\|\cdot\|$. 
We also report the average coverage rate of the pointwise 90\% credible intervals for $\{g(x_{eval,i})\}$. 

\begin{figure}[htbp]
\centering 
\subfigure[$n_{test}=20, m=200$, varying $n_{pret}$]{
\includegraphics[width=0.9\linewidth]{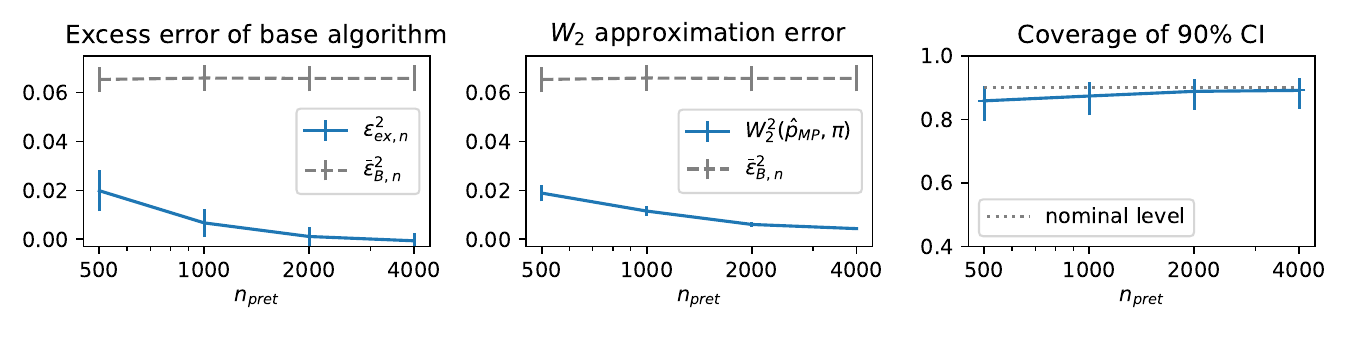}
}
\subfigure[$n_{pret}=2000, m=200$, varying $n_{test}$]{
\includegraphics[width=0.9\linewidth]{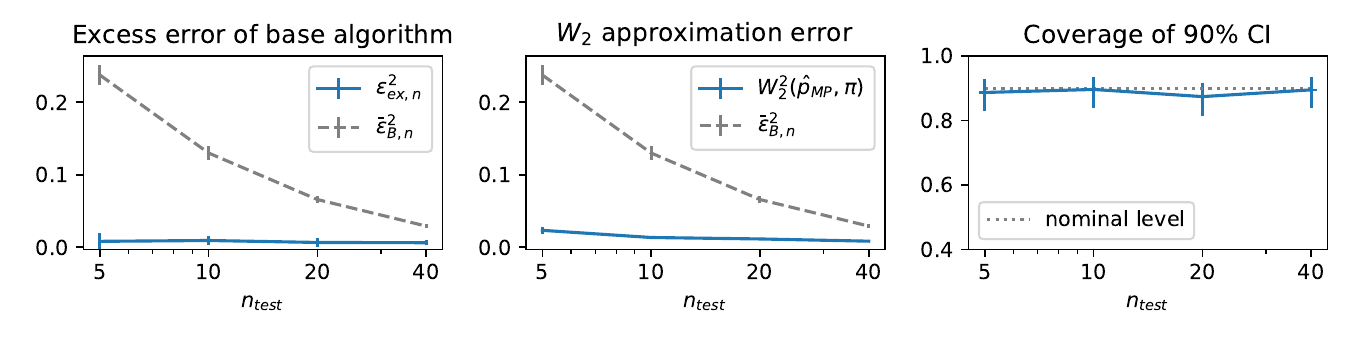}
}
\subfigure[$n_{pret}=2000, n_{test}=20$, varying $m$]{
\includegraphics[width=0.9\linewidth]{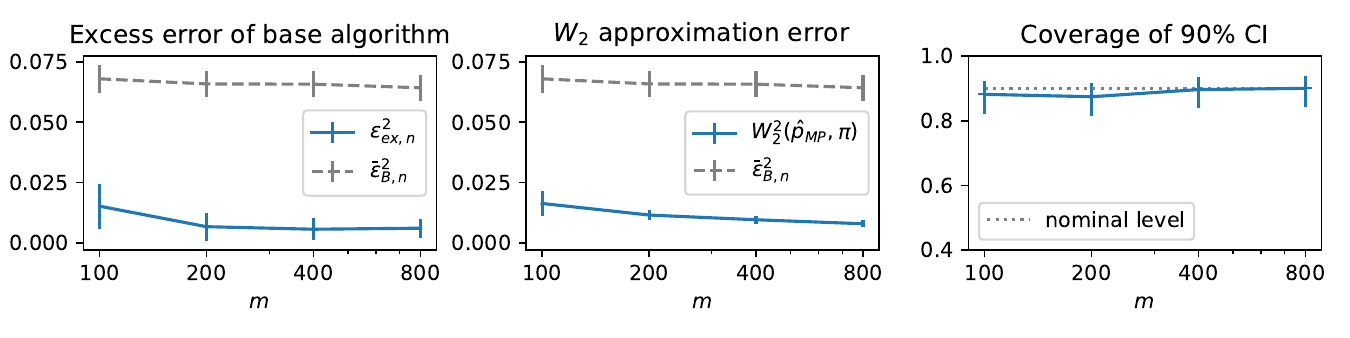}
}
\caption{Multi-task learning simulation: results with varying choices of $(m,n_{pret},n_{test})$. Plotted are the mean and 95\% confidence interval (CI) for each metric. CIs are computed on 160 replications using normal approximation (first two subplots) or the Wilson score (last subplot).
}\label{fig:mtl-sim-asymp}
\end{figure}

\paragraph{Results and discussion.} The results are summarised in \Cref{fig:mtl-sim-asymp}. We can see that  
as we increase the pretraining sample size ($n_{pret}$), the task diversity ($m$), or move closer to a few-shot scenario ($n_{test}$), 
the ratio $\ExcessError[n_{test}] / \BayesError[n_{test}]$ vanishes, indicating \Cref{asm:replaced-efficiency} becomes more applicable; 
and as predicted by \Cref{thm:param-alt} the Wasserstein distance between the MP and the Bayesian posterior becomes vanishing compared with the spread ($\BayesError[n_{test}]$) of the latter. 
In such cases the coverage rate of the MP credible intervals also matches their nominal level, in line with the discussion below \Cref{thm:param-alt}. 
These results validate the analysis in \S\ref{sec:theory} 
in a multi-task learning setting. %

\subsection{Hyperparameter Learning for Gaussian Processes}\label{app:exp-gp}

\paragraph{Setup details.} To implement \Cref{alg:main}, we sample $\hat x_{n+i}$ from a kernel density estimate and $\hat y_{n+i}\mid \hat x_{n+i}$ from the GP's %
marginal predictive distribution, and use $\Delta n=0.25n, N=4n$. %
In preliminary experiments we find that 
a larger choice of $N$ or a smaller choice of $\Delta n$ 
appears to lead to diminishing improvements for performance; thus we adopt this choice for simplicity. 
For all methods, 
we implement the base empirical Bayes algorithm with the L-BFGS-B optimiser \citep{zhu1997algorithm} using a step-size of $0.05$ and $1600$ iterations, and build an ensemble of $K=16$ predictors.  

The hyperparameter learning process has a high variation across randomly sampled training sets due to the small sample sizes. 
Therefore, we use Wilcoxon signed-rank tests to check for statistically significant improvement, and  
account for ties in computing the ranks for \Cref{tbl:gp-main}, by defining the rank of each method as the number of methods that significantly outperform it as determined by the Wilcoxon test.  

\paragraph{Full results and discussion.} Full results are shown in \Cref{tbl:gp-full}. As we can see, our method consistently improves upon the EB baseline and is competitive against the other ensemble approaches. 
Nonparametric bootstrap also demonstrates competitive performance with $n=75$, but generally underperforms the EB baseline when $n=300$. 
It is possible that the distribution of parameter estimates from bootstrap has a very high variation, which may be only beneficial when overfitting is severe. 
We note that 
the performance difference is often small compared to the standard deviation, %
but the improvement over baselines is consistent as evidenced by the Wilcoxon test.  

\begin{sidewaystable}\scriptsize
\caption{
Full results for the GP experiment: mean and standard deviation for all test metrics. Boldface indicates the best result ($p<0.05$ in a Wilcoxon signed-rank test).
}\label{tbl:gp-full}
\begin{tabular}[h]{ccccccccccccc} \toprule Dataset 	& \multicolumn{4}{c}{ RMSE }	& \multicolumn{4}{c}{ NLPD }	& \multicolumn{4}{c}{ CRPS }	\\
 \cmidrule(lr){2-5} \cmidrule(lr){6-9}\cmidrule(lr){10-13} 	& Emp.~Bayes	& Bootstrap	& Ensemble	& Proposed	& Emp.~Bayes	& Bootstrap	& Ensemble	& Proposed	& Emp.~Bayes	& Bootstrap	& Ensemble	& Proposed	\\ \midrule
\multicolumn{13}{l}{$n=75$} \\ \midrule
Boston	& $4.47$ {\tiny $ \pm0.93$}	& $\mathbf{4.39}$ {\tiny $ \pm0.79$}	& $4.53$ {\tiny $ \pm0.89$}	& $4.49$ {\tiny $ \pm0.86$}	& $3.28$ {\tiny $ \pm0.42$}	& $\mathbf{2.72}$ {\tiny $ \pm0.14$}	& $3.19$ {\tiny $ \pm0.38$}	& $3.17$ {\tiny $ \pm0.37$}	& $2.31$ {\tiny $ \pm0.38$}	& $\mathbf{2.16}$ {\tiny $ \pm0.27$}	& $2.30$ {\tiny $ \pm0.34$}	& $2.28$ {\tiny $ \pm0.33$}	\\
Concrete	& $8.16$ {\tiny $ \pm0.94$}	& $8.21$ {\tiny $ \pm0.79$}	& $8.16$ {\tiny $ \pm0.89$}	& $\mathbf{8.10}$ {\tiny $ \pm0.86$}	& $3.66$ {\tiny $ \pm14.40$}	& $\mathbf{3.47}$ {\tiny $ \pm0.08$}	& $3.55$ {\tiny $ \pm3.42$}	& $3.54$ {\tiny $ \pm1.69$}	& $4.38$ {\tiny $ \pm0.50$}	& $\mathbf{4.44}$ {\tiny $ \pm0.38$}	& $4.38$ {\tiny $ \pm0.49$}	& $\mathbf{4.31}$ {\tiny $ \pm0.47$}	\\
Energy	& $1.27$ {\tiny $ \pm0.28$}	& $1.47$ {\tiny $ \pm0.19$}	& $1.27$ {\tiny $ \pm0.27$}	& $\mathbf{1.27}$ {\tiny $ \pm0.25$}	& $1.26$ {\tiny $ \pm0.28$}	& $1.60$ {\tiny $ \pm0.14$}	& $\mathbf{1.25}$ {\tiny $ \pm0.25$}	& $\mathbf{1.24}$ {\tiny $ \pm0.24$}	& $0.55$ {\tiny $ \pm0.12$}	& $0.73$ {\tiny $ \pm0.09$}	& $0.55$ {\tiny $ \pm0.12$}	& $\mathbf{0.55}$ {\tiny $ \pm0.10$}	\\
Kin8nm	& $\mathbf{0.19}$ {\tiny $ \pm0.02$}	& $0.19$ {\tiny $ \pm0.01$}	& $\mathbf{0.19}$ {\tiny $ \pm0.02$}	& $\mathbf{0.19}$ {\tiny $ \pm0.02$}	& $\mathbf{-0.23}$ {\tiny $ \pm0.14$}	& $\mathbf{-0.23}$ {\tiny $ \pm0.06$}	& $\mathbf{-0.23}$ {\tiny $ \pm0.12$}	& $\mathbf{-0.22}$ {\tiny $ \pm0.13$}	& $\mathbf{0.11}$ {\tiny $ \pm0.01$}	& $0.11$ {\tiny $ \pm0.01$}	& $\mathbf{0.11}$ {\tiny $ \pm0.01$}	& $\mathbf{0.11}$ {\tiny $ \pm0.01$}	\\
Naval	& $0.01$ {\tiny $ \pm0.00$}	& $0.01$ {\tiny $ \pm0.00$}	& $0.00$ {\tiny $ \pm0.00$}	& $\mathbf{0.00}$ {\tiny $ \pm0.00$}	& $\mathbf{-5.05}$ {\tiny $ \pm0.12$}	& $-4.06$ {\tiny $ \pm0.12$}	& $-4.99$ {\tiny $ \pm0.11$}	& $\mathbf{-5.03}$ {\tiny $ \pm0.13$}	& $0.00$ {\tiny $ \pm0.00$}	& $0.00$ {\tiny $ \pm0.00$}	& $0.00$ {\tiny $ \pm0.00$}	& $\mathbf{0.00}$ {\tiny $ \pm0.00$}	\\
Power	& $4.54$ {\tiny $ \pm0.22$}	& $5.07$ {\tiny $ \pm0.42$}	& $4.54$ {\tiny $ \pm0.22$}	& $\mathbf{4.54}$ {\tiny $ \pm0.19$}	& $2.94$ {\tiny $ \pm0.05$}	& $3.10$ {\tiny $ \pm0.07$}	& $2.94$ {\tiny $ \pm0.05$}	& $\mathbf{2.94}$ {\tiny $ \pm0.04$}	& $2.50$ {\tiny $ \pm0.12$}	& $2.86$ {\tiny $ \pm0.22$}	& $2.50$ {\tiny $ \pm0.12$}	& $\mathbf{2.49}$ {\tiny $ \pm0.10$}	\\
Protein	& $6.03$ {\tiny $ \pm0.35$}	& $\mathbf{5.76}$ {\tiny $ \pm0.14$}	& $5.92$ {\tiny $ \pm0.32$}	& $5.92$ {\tiny $ \pm0.32$}	& $3.36$ {\tiny $ \pm0.35$}	& $\mathbf{3.17}$ {\tiny $ \pm0.05$}	& $3.22$ {\tiny $ \pm0.22$}	& $3.22$ {\tiny $ \pm0.22$}	& $3.50$ {\tiny $ \pm0.23$}	& $\mathbf{3.31}$ {\tiny $ \pm0.09$}	& $3.38$ {\tiny $ \pm0.21$}	& $3.37$ {\tiny $ \pm0.21$}	\\
Winered	& $0.76$ {\tiny $ \pm0.04$}	& $\mathbf{0.71}$ {\tiny $ \pm0.03$}	& $0.75$ {\tiny $ \pm0.04$}	& $0.74$ {\tiny $ \pm0.04$}	& $1.30$ {\tiny $ \pm2.25$}	& $\mathbf{1.08}$ {\tiny $ \pm0.07$}	& $1.19$ {\tiny $ \pm0.27$}	& $1.18$ {\tiny $ \pm0.25$}	& $0.43$ {\tiny $ \pm0.03$}	& $\mathbf{0.39}$ {\tiny $ \pm0.02$}	& $0.42$ {\tiny $ \pm0.03$}	& $0.42$ {\tiny $ \pm0.03$}	\\
Winewhite	& $0.87$ {\tiny $ \pm0.04$}	& $\mathbf{0.81}$ {\tiny $ \pm0.03$}	& $0.85$ {\tiny $ \pm0.04$}	& $0.84$ {\tiny $ \pm0.05$}	& $1.49$ {\tiny $ \pm0.28$}	& $\mathbf{1.22}$ {\tiny $ \pm0.06$}	& $1.40$ {\tiny $ \pm0.26$}	& $1.36$ {\tiny $ \pm0.28$}	& $0.50$ {\tiny $ \pm0.03$}	& $\mathbf{0.45}$ {\tiny $ \pm0.02$}	& $0.48$ {\tiny $ \pm0.03$}	& $0.48$ {\tiny $ \pm0.03$}	\\
\midrule \multicolumn{13}{l}{$n=300$} \\ \midrule
Boston	& $3.22$ {\tiny $ \pm0.52$}	& $\mathbf{3.19}$ {\tiny $ \pm0.43$}	& $3.19$ {\tiny $ \pm0.50$}	& $\mathbf{3.17}$ {\tiny $ \pm0.48$}	& $2.55$ {\tiny $ \pm0.16$}	& $\mathbf{2.42}$ {\tiny $ \pm0.09$}	& $2.54$ {\tiny $ \pm0.16$}	& $2.52$ {\tiny $ \pm0.14$}	& $1.61$ {\tiny $ \pm0.18$}	& $1.62$ {\tiny $ \pm0.13$}	& $1.60$ {\tiny $ \pm0.17$}	& $\mathbf{1.58}$ {\tiny $ \pm0.16$}	\\
Concrete	& $6.51$ {\tiny $ \pm0.41$}	& $6.77$ {\tiny $ \pm0.61$}	& $6.51$ {\tiny $ \pm0.41$}	& $\mathbf{6.47}$ {\tiny $ \pm0.42$}	& $3.24$ {\tiny $ \pm0.13$}	& $\mathbf{3.22}$ {\tiny $ \pm11.24$}	& $3.24$ {\tiny $ \pm0.13$}	& $\mathbf{3.22}$ {\tiny $ \pm0.12$}	& $3.42$ {\tiny $ \pm0.21$}	& $3.53$ {\tiny $ \pm0.29$}	& $3.42$ {\tiny $ \pm0.21$}	& $\mathbf{3.40}$ {\tiny $ \pm0.21$}	\\
Energy	& $0.60$ {\tiny $ \pm0.14$}	& $0.69$ {\tiny $ \pm0.13$}	& $\mathbf{0.58}$ {\tiny $ \pm0.15$}	& $\mathbf{0.57}$ {\tiny $ \pm0.15$}	& $0.77$ {\tiny $ \pm0.12$}	& $0.90$ {\tiny $ \pm0.07$}	& $\mathbf{0.71}$ {\tiny $ \pm0.10$}	& $\mathbf{0.70}$ {\tiny $ \pm0.11$}	& $0.29$ {\tiny $ \pm0.03$}	& $0.34$ {\tiny $ \pm0.03$}	& $0.28$ {\tiny $ \pm0.04$}	& $\mathbf{0.28}$ {\tiny $ \pm0.04$}	\\
Kin8nm	& $\mathbf{0.12}$ {\tiny $ \pm0.00$}	& $0.13$ {\tiny $ \pm0.00$}	& $\mathbf{0.12}$ {\tiny $ \pm0.00$}	& $\mathbf{0.12}$ {\tiny $ \pm0.00$}	& $\mathbf{-0.69}$ {\tiny $ \pm0.03$}	& $-0.62$ {\tiny $ \pm0.02$}	& $\mathbf{-0.69}$ {\tiny $ \pm0.03$}	& $\mathbf{-0.69}$ {\tiny $ \pm0.03$}	& $\mathbf{0.07}$ {\tiny $ \pm0.00$}	& $0.07$ {\tiny $ \pm0.00$}	& $\mathbf{0.07}$ {\tiny $ \pm0.00$}	& $\mathbf{0.07}$ {\tiny $ \pm0.00$}	\\
Naval	& $0.00$ {\tiny $ \pm0.00$}	& $0.00$ {\tiny $ \pm0.00$}	& $\mathbf{0.00}$ {\tiny $ \pm0.00$}	& $\mathbf{0.00}$ {\tiny $ \pm0.00$}	& $-7.00$ {\tiny $ \pm0.04$}	& $-6.49$ {\tiny $ \pm0.04$}	& $-7.01$ {\tiny $ \pm0.04$}	& $\mathbf{-7.01}$ {\tiny $ \pm0.04$}	& $\mathbf{0.00}$ {\tiny $ \pm0.00$}	& $0.00$ {\tiny $ \pm0.00$}	& $\mathbf{0.00}$ {\tiny $ \pm0.00$}	& $\mathbf{0.00}$ {\tiny $ \pm0.00$}	\\
Power	& $4.31$ {\tiny $ \pm0.10$}	& $4.73$ {\tiny $ \pm0.17$}	& $4.31$ {\tiny $ \pm0.10$}	& $\mathbf{4.30}$ {\tiny $ \pm0.10$}	& $2.88$ {\tiny $ \pm0.03$}	& $3.04$ {\tiny $ \pm0.04$}	& $2.88$ {\tiny $ \pm0.03$}	& $\mathbf{2.88}$ {\tiny $ \pm0.03$}	& $2.36$ {\tiny $ \pm0.04$}	& $2.68$ {\tiny $ \pm0.09$}	& $2.36$ {\tiny $ \pm0.04$}	& $\mathbf{2.36}$ {\tiny $ \pm0.04$}	\\
Protein	& $5.18$ {\tiny $ \pm0.16$}	& $5.36$ {\tiny $ \pm0.10$}	& $5.15$ {\tiny $ \pm0.14$}	& $\mathbf{5.14}$ {\tiny $ \pm0.14$}	& $3.07$ {\tiny $ \pm0.03$}	& $\mathbf{3.07}$ {\tiny $ \pm0.02$}	& $3.06$ {\tiny $ \pm0.04$}	& $\mathbf{3.06}$ {\tiny $ \pm0.04$}	& $2.93$ {\tiny $ \pm0.08$}	& $3.02$ {\tiny $ \pm0.07$}	& $2.92$ {\tiny $ \pm0.07$}	& $\mathbf{2.91}$ {\tiny $ \pm0.07$}	\\
Winered	& $0.71$ {\tiny $ \pm0.04$}	& $\mathbf{0.67}$ {\tiny $ \pm0.03$}	& $0.70$ {\tiny $ \pm0.05$}	& $\mathbf{0.69}$ {\tiny $ \pm0.05$}	& $\mathbf{0.87}$ {\tiny $ \pm0.15$}	& $0.98$ {\tiny $ \pm0.05$}	& $0.94$ {\tiny $ \pm0.12$}	& $\mathbf{0.93}$ {\tiny $ \pm0.11$}	& $0.38$ {\tiny $ \pm0.02$}	& $\mathbf{0.37}$ {\tiny $ \pm0.02$}	& $0.37$ {\tiny $ \pm0.03$}	& $\mathbf{0.37}$ {\tiny $ \pm0.03$}	\\
Winewhite	& $0.79$ {\tiny $ \pm0.02$}	& $\mathbf{0.76}$ {\tiny $ \pm0.02$}	& $0.78$ {\tiny $ \pm0.03$}	& $0.78$ {\tiny $ \pm0.03$}	& $1.13$ {\tiny $ \pm0.04$}	& $1.12$ {\tiny $ \pm0.03$}	& $1.10$ {\tiny $ \pm0.04$}	& $\mathbf{1.09}$ {\tiny $ \pm0.04$}	& $0.44$ {\tiny $ \pm0.01$}	& $\mathbf{0.42}$ {\tiny $ \pm0.01$}	& $0.43$ {\tiny $ \pm0.01$}	& $0.43$ {\tiny $ \pm0.01$}	\\
\bottomrule \end{tabular}

\end{sidewaystable}

\subsection{Classification with Boosting Tree and AutoML Algorithms}\label{app:exp-tree}

\paragraph{Deferred setup details.} We evaluate on the 30 datasets from the OpenML CC18 benchmark \citep{bischl2017openml} with $n\le 2000, \dim x\le 100, \dim y\le 10$. 
In all experiments we adopt a 60-20-20 split for train/validation/test, and determine the hyperparameters for the base algorithm using the log loss on validation set. 
We implement our method by refitting a predictor from scratch at each iteration; in other words, in \Cref{alg:main} we define both $\cA_0(D_{j+1};\hat\theta_j)$ and $\cA_0(D_n)$ as the predictor resulted by applying the base algorithm to the respective dataset. 

For the GDBT algorithm, we adopt the implementation from XGBoost and 
conduct search for the following hyperparameters: tree depth $D\in\{4,5,6,7\}$, number of boosting iterations $L\in\{50, 100, 200\}$ and learning rate $\eta\in\{10, 30, 100\}/L$. 
We also conduct early stopping using the validation set with a tolerance of 10 rounds. 
For the instantiations of our method and bagging, we build an ensemble of 50 predictors; for our method, we determine $\Delta n\in\{0.125n, 0.25n, n\}, N\in\{n, 3n\}$ based on the same validation loss. 

We use the default implementation in AutoGluon (\verb!TabularPredictor(eval_metric="log_loss")! \verb!.fit!), 
which determines the hyperparameters for the individual models based on pre-defined rules and 
uses the validation set to estimate a linear stacking model following \citet{caruana2004ensemble}. 
As the AutoML algorithm is more computation intensive, we build an ensemble of 20 predictors for our method and bagging, and set $\Delta n=N=n$ 
for our method. 

\begin{figure}
\centering 
\includegraphics[width=0.47\linewidth]{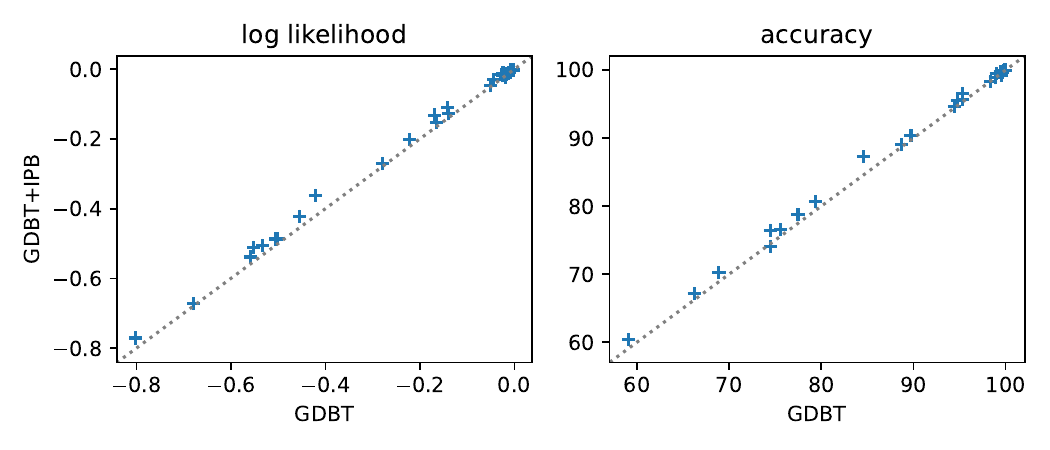}
\includegraphics[width=0.47\linewidth]{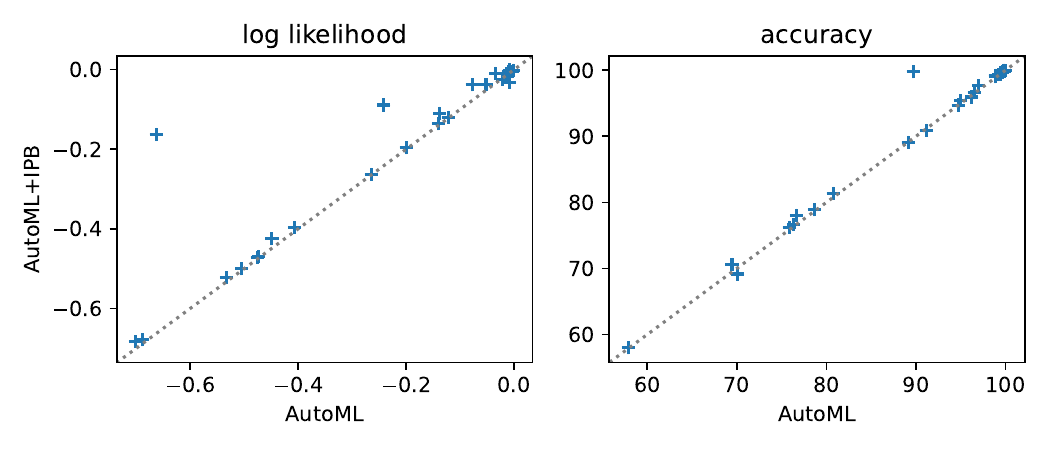}
\caption{Classification experiment: scatter plot of the test metrics (for each dataset averaged over 10 random splits; higher is better) for the base algorithm vs the proposed method. 
}\label{fig:cc18-cmp}
\end{figure}

\paragraph{Additional results.} Table~\ref{tbl:cc18-nll}--\ref{tbl:cc18-acc} report the full test metrics on all 30 datasets; for each baseline method we further conducts a Wilcoxon test to compare its distribution of loss metrics (for each dataset, averaged over 10 random splits) against that of the proposed method, and report the p-value in the respective table. As we can see, except for the test accuracy of the AutoML+bagging baseline, our method always leads to a statistically significant improvement ($p<0.05$). 

We note that the AutoGluon library recommends a more sophisticated multi-level algorithm (corresponding to \verb!.fit(presets="best_quality")!) for the best predictive performance. 
We evaluated that algorithm under identical conditions, and found it to perform better than our chosen base algorithm but worse than bagging and our method applied to the latter (average accuracy 91.1\%, NLL 0.198 in the setting of \Cref{tbl:tree-main}).  
As the algorithm also has a significantly higher computational cost, we refrain from testing our method with it, although we expect a similar improvement in performance if our method were applied. 

\Cref{fig:tree-adult-fea} visualises the uncertainty estimates for 
the information gain-based feature importance scores,  
obtained using our method on the UCI adult dataset. 
As we can see, the correlation structure of the approximate MP is informative about feature dependencies; for example, the strong negative correlation between ``marital status'' and ``relationship'' indicates that these two features are interchangeable for prediction. 

\begin{table}\centering\scriptsize
    \caption{
        Classification experiment: average negative log likelihood across random train/test splits in each dataset. 
    }\label{tbl:cc18-nll}
    \begin{tabular}[h]{ccccccc} 
        \toprule 
        \multirow{2}{*}[-0.2em]{Dataset} & \multicolumn{3}{c}{GBDT} & \multicolumn{3}{c}{AutoML} \\ 
        \cmidrule(lr){2-4}  \cmidrule(lr){5-7}
            & (Base) & + BS & + IPB & (Base) & + BS & + IPB 
            \\ \midrule
banknote-authentication	& $.002${\tiny $\pm .00$}	& $.003${\tiny $\pm .00$}	& $.003${\tiny $\pm .00$}	& $.009${\tiny $\pm .01$}	& $.002${\tiny $\pm .00$}	& $.001${\tiny $\pm .00$}	\\
blood-transfusion-service-center	& $.504${\tiny $\pm .03$}	& $.486${\tiny $\pm .02$}	& $.487${\tiny $\pm .02$}	& $.473${\tiny $\pm .02$}	& $.470${\tiny $\pm .02$}	& $.469${\tiny $\pm .03$}	\\
breast-w	& $.139${\tiny $\pm .02$}	& $.129${\tiny $\pm .02$}	& $.128${\tiny $\pm .02$}	& $.138${\tiny $\pm .03$}	& $.103${\tiny $\pm .01$}	& $.110${\tiny $\pm .02$}	\\
mfeat-karhunen	& $.012${\tiny $\pm .00$}	& $.022${\tiny $\pm .00$}	& $.012${\tiny $\pm .00$}	& $.008${\tiny $\pm .01$}	& $.086${\tiny $\pm .04$}	& $.031${\tiny $\pm .03$}	\\
mfeat-morphological	& $.018${\tiny $\pm .01$}	& $.021${\tiny $\pm .00$}	& $.014${\tiny $\pm .00$}	& $.014${\tiny $\pm .01$}	& $.030${\tiny $\pm .02$}	& $.009${\tiny $\pm .00$}	\\
eucalyptus	& $.802${\tiny $\pm .04$}	& $.786${\tiny $\pm .03$}	& $.771${\tiny $\pm .03$}	& $.689${\tiny $\pm .05$}	& $.704${\tiny $\pm .04$}	& $.679${\tiny $\pm .05$}	\\
mfeat-zernike	& $.017${\tiny $\pm .01$}	& $.021${\tiny $\pm .00$}	& $.012${\tiny $\pm .00$}	& $.241${\tiny $\pm .43$}	& $.059${\tiny $\pm .03$}	& $.089${\tiny $\pm .13$}	\\
cmc	& $.028${\tiny $\pm .01$}	& $.018${\tiny $\pm .00$}	& $.016${\tiny $\pm .00$}	& $.019${\tiny $\pm .01$}	& $.022${\tiny $\pm .01$}	& $.020${\tiny $\pm .01$}	\\
credit-approval	& $.169${\tiny $\pm .03$}	& $.159${\tiny $\pm .02$}	& $.132${\tiny $\pm .02$}	& $.122${\tiny $\pm .03$}	& $.125${\tiny $\pm .02$}	& $.120${\tiny $\pm .03$}	\\
vowel	& $.533${\tiny $\pm .02$}	& $.506${\tiny $\pm .02$}	& $.505${\tiny $\pm .02$}	& $.504${\tiny $\pm .03$}	& $.501${\tiny $\pm .02$}	& $.500${\tiny $\pm .02$}	\\
credit-g	& $.011${\tiny $\pm .00$}	& $.018${\tiny $\pm .00$}	& $.010${\tiny $\pm .00$}	& $.003${\tiny $\pm .00$}	& $.004${\tiny $\pm .00$}	& $.005${\tiny $\pm .00$}	\\
analcatdata\_authorship	& $.044${\tiny $\pm .03$}	& $.045${\tiny $\pm .02$}	& $.030${\tiny $\pm .01$}	& $.052${\tiny $\pm .04$}	& $.029${\tiny $\pm .01$}	& $.038${\tiny $\pm .02$}	\\
balance-scale	& $.421${\tiny $\pm .06$}	& $.362${\tiny $\pm .03$}	& $.361${\tiny $\pm .03$}	& $.663${\tiny $\pm .64$}	& $.137${\tiny $\pm .04$}	& $.163${\tiny $\pm .05$}	\\
analcatdata\_dmft	& $.501${\tiny $\pm .02$}	& $.490${\tiny $\pm .01$}	& $.487${\tiny $\pm .01$}	& $.476${\tiny $\pm .02$}	& $.472${\tiny $\pm .01$}	& $.471${\tiny $\pm .02$}	\\
diabetes	& $.222${\tiny $\pm .02$}	& $.205${\tiny $\pm .01$}	& $.201${\tiny $\pm .01$}	& $.200${\tiny $\pm .01$}	& $.200${\tiny $\pm .01$}	& $.196${\tiny $\pm .01$}	\\
pc4	& $.279${\tiny $\pm .02$}	& $.270${\tiny $\pm .02$}	& $.270${\tiny $\pm .02$}	& $.264${\tiny $\pm .02$}	& $.264${\tiny $\pm .02$}	& $.263${\tiny $\pm .02$}	\\
pc3	& $.019${\tiny $\pm .01$}	& $.022${\tiny $\pm .01$}	& $.024${\tiny $\pm .01$}	& $.078${\tiny $\pm .10$}	& $.026${\tiny $\pm .01$}	& $.038${\tiny $\pm .02$}	\\
kc2	& $.016${\tiny $\pm .01$}	& $.014${\tiny $\pm .01$}	& $.014${\tiny $\pm .01$}	& $.021${\tiny $\pm .02$}	& $.021${\tiny $\pm .01$}	& $.024${\tiny $\pm .02$}	\\
pc1	& $.009${\tiny $\pm .01$}	& $.003${\tiny $\pm .00$}	& $.003${\tiny $\pm .00$}	& $.001${\tiny $\pm .00$}	& $.001${\tiny $\pm .00$}	& $.001${\tiny $\pm .00$}	\\
tic-tac-toe	& $.551${\tiny $\pm .03$}	& $.536${\tiny $\pm .02$}	& $.510${\tiny $\pm .02$}	& $.448${\tiny $\pm .02$}	& $.450${\tiny $\pm .02$}	& $.424${\tiny $\pm .02$}	\\
vehicle	& $.141${\tiny $\pm .03$}	& $.117${\tiny $\pm .03$}	& $.110${\tiny $\pm .03$}	& $.119${\tiny $\pm .05$}	& $.094${\tiny $\pm .03$}	& $.095${\tiny $\pm .04$}	\\
wdbc	& $.025${\tiny $\pm .02$}	& $.015${\tiny $\pm .01$}	& $.011${\tiny $\pm .00$}	& $.034${\tiny $\pm .03$}	& $.027${\tiny $\pm .02$}	& $.009${\tiny $\pm .01$}	\\
qsar-biodeg	& $.558${\tiny $\pm .02$}	& $.543${\tiny $\pm .01$}	& $.538${\tiny $\pm .01$}	& $.533${\tiny $\pm .02$}	& $.524${\tiny $\pm .01$}	& $.523${\tiny $\pm .01$}	\\
dresses-sales	& $.678${\tiny $\pm .01$}	& $.672${\tiny $\pm .01$}	& $.672${\tiny $\pm .01$}	& $.701${\tiny $\pm .02$}	& $.683${\tiny $\pm .02$}	& $.683${\tiny $\pm .02$}	\\
mfeat-fourier	& $.025${\tiny $\pm .01$}	& $.027${\tiny $\pm .00$}	& $.019${\tiny $\pm .00$}	& $.010${\tiny $\pm .01$}	& $.031${\tiny $\pm .02$}	& $.010${\tiny $\pm .00$}	\\
MiceProtein	& $.023${\tiny $\pm .01$}	& $.025${\tiny $\pm .00$}	& $.011${\tiny $\pm .00$}	& $.008${\tiny $\pm .01$}	& $.020${\tiny $\pm .01$}	& $.002${\tiny $\pm .00$}	\\
steel-plates-fault	& $.021${\tiny $\pm .01$}	& $.024${\tiny $\pm .01$}	& $.016${\tiny $\pm .00$}	& $.010${\tiny $\pm .01$}	& $.020${\tiny $\pm .01$}	& $.005${\tiny $\pm .00$}	\\
climate-model-simulation-crashes	& $.165${\tiny $\pm .04$}	& $.158${\tiny $\pm .03$}	& $.152${\tiny $\pm .03$}	& $.140${\tiny $\pm .03$}	& $.139${\tiny $\pm .02$}	& $.136${\tiny $\pm .03$}	\\
car	& $.050${\tiny $\pm .01$}	& $.072${\tiny $\pm .01$}	& $.048${\tiny $\pm .01$}	& $.028${\tiny $\pm .01$}	& $.047${\tiny $\pm .01$}	& $.025${\tiny $\pm .01$}	\\
cylinder-bands	& $.454${\tiny $\pm .05$}	& $.429${\tiny $\pm .02$}	& $.422${\tiny $\pm .03$}	& $.407${\tiny $\pm .05$}	& $.407${\tiny $\pm .03$}	& $.396${\tiny $\pm .04$}	\\
\midrule 
Wilcoxon p-value vs IPB & 3.1e-08	& 6e-07	& -	& 2.2e-06	& 0.029	& -	\\
 \bottomrule
        \end{tabular}

\end{table}

\begin{table}\centering\scriptsize
    \caption{Classification experiment: average test accuracy across random train/test splits in each dataset.}\label{tbl:cc18-acc}
    \begin{tabular}[h]{ccccccc} 
        \toprule 
        \multirow{2}{*}[-0.2em]{Dataset} & \multicolumn{3}{c}{GBDT} & \multicolumn{3}{c}{AutoML} \\ 
        \cmidrule(lr){2-4}  \cmidrule(lr){5-7}
            & (Base) & + BS & + IPB & (Base) & + BS & + IPB 
            \\ \midrule
banknote-authentication	& $99.9${\tiny $\pm 0.1$}	& $99.9${\tiny $\pm 0.1$}	& $100.0${\tiny $\pm 0.0$}	& $99.9${\tiny $\pm 0.1$}	& $100.0${\tiny $\pm 0.1$}	& $100.0${\tiny $\pm 0.0$}	\\
blood-transfusion-service-center	& $77.5${\tiny $\pm 2.1$}	& $79.0${\tiny $\pm 1.6$}	& $78.7${\tiny $\pm 1.7$}	& $78.7${\tiny $\pm 1.1$}	& $78.7${\tiny $\pm 1.2$}	& $78.9${\tiny $\pm 1.6$}	\\
breast-w	& $95.4${\tiny $\pm 0.6$}	& $95.9${\tiny $\pm 0.9$}	& $95.7${\tiny $\pm 0.6$}	& $96.5${\tiny $\pm 0.7$}	& $96.6${\tiny $\pm 0.4$}	& $96.6${\tiny $\pm 0.5$}	\\
mfeat-karhunen	& $99.9${\tiny $\pm 0.1$}	& $99.8${\tiny $\pm 0.1$}	& $99.9${\tiny $\pm 0.1$}	& $99.8${\tiny $\pm 0.1$}	& $99.9${\tiny $\pm 0.1$}	& $100.0${\tiny $\pm 0.1$}	\\
mfeat-morphological	& $99.4${\tiny $\pm 0.4$}	& $99.6${\tiny $\pm 0.2$}	& $99.8${\tiny $\pm 0.2$}	& $99.6${\tiny $\pm 0.2$}	& $99.7${\tiny $\pm 0.2$}	& $99.8${\tiny $\pm 0.2$}	\\
eucalyptus	& $66.2${\tiny $\pm 2.2$}	& $65.9${\tiny $\pm 2.0$}	& $67.2${\tiny $\pm 2.0$}	& $69.5${\tiny $\pm 2.9$}	& $69.1${\tiny $\pm 2.9$}	& $70.6${\tiny $\pm 2.4$}	\\
mfeat-zernike	& $99.7${\tiny $\pm 0.2$}	& $99.7${\tiny $\pm 0.2$}	& $99.8${\tiny $\pm 0.2$}	& $89.7${\tiny $\pm 18.5$}	& $99.9${\tiny $\pm 0.1$}	& $99.8${\tiny $\pm 0.1$}	\\
cmc	& $99.0${\tiny $\pm 0.3$}	& $99.5${\tiny $\pm 0.1$}	& $99.5${\tiny $\pm 0.2$}	& $99.5${\tiny $\pm 0.2$}	& $99.6${\tiny $\pm 0.2$}	& $99.4${\tiny $\pm 0.2$}	\\
credit-approval	& $94.8${\tiny $\pm 0.8$}	& $95.0${\tiny $\pm 0.4$}	& $95.6${\tiny $\pm 0.8$}	& $96.2${\tiny $\pm 1.1$}	& $95.7${\tiny $\pm 0.7$}	& $95.9${\tiny $\pm 1.1$}	\\
vowel	& $74.5${\tiny $\pm 2.1$}	& $75.5${\tiny $\pm 1.7$}	& $76.5${\tiny $\pm 1.9$}	& $75.8${\tiny $\pm 1.8$}	& $75.0${\tiny $\pm 1.8$}	& $76.2${\tiny $\pm 2.0$}	\\
credit-g	& $99.8${\tiny $\pm 0.2$}	& $99.8${\tiny $\pm 0.2$}	& $99.9${\tiny $\pm 0.1$}	& $99.9${\tiny $\pm 0.1$}	& $99.9${\tiny $\pm 0.1$}	& $99.9${\tiny $\pm 0.1$}	\\
analcatdata\_authorship	& $98.9${\tiny $\pm 0.6$}	& $98.7${\tiny $\pm 0.7$}	& $98.9${\tiny $\pm 0.6$}	& $98.9${\tiny $\pm 0.6$}	& $99.1${\tiny $\pm 0.4$}	& $99.0${\tiny $\pm 0.5$}	\\
balance-scale	& $84.6${\tiny $\pm 1.9$}	& $89.2${\tiny $\pm 1.9$}	& $87.3${\tiny $\pm 1.7$}	& $95.0${\tiny $\pm 1.2$}	& $94.8${\tiny $\pm 0.9$}	& $95.4${\tiny $\pm 0.9$}	\\
analcatdata\_dmft	& $75.6${\tiny $\pm 1.9$}	& $75.6${\tiny $\pm 2.3$}	& $76.6${\tiny $\pm 2.5$}	& $76.4${\tiny $\pm 1.5$}	& $76.6${\tiny $\pm 2.1$}	& $76.6${\tiny $\pm 1.6$}	\\
diabetes	& $89.7${\tiny $\pm 1.0$}	& $90.4${\tiny $\pm 0.9$}	& $90.4${\tiny $\pm 1.0$}	& $91.1${\tiny $\pm 1.0$}	& $90.8${\tiny $\pm 1.0$}	& $90.9${\tiny $\pm 0.9$}	\\
pc4	& $88.7${\tiny $\pm 1.2$}	& $89.2${\tiny $\pm 1.1$}	& $89.0${\tiny $\pm 1.2$}	& $89.1${\tiny $\pm 1.1$}	& $89.1${\tiny $\pm 1.0$}	& $89.1${\tiny $\pm 1.0$}	\\
pc3	& $99.5${\tiny $\pm 0.3$}	& $99.6${\tiny $\pm 0.3$}	& $99.2${\tiny $\pm 0.6$}	& $99.2${\tiny $\pm 0.6$}	& $99.4${\tiny $\pm 0.5$}	& $99.3${\tiny $\pm 0.5$}	\\
kc2	& $99.6${\tiny $\pm 0.2$}	& $99.6${\tiny $\pm 0.3$}	& $99.6${\tiny $\pm 0.2$}	& $99.6${\tiny $\pm 0.3$}	& $99.7${\tiny $\pm 0.2$}	& $99.6${\tiny $\pm 0.2$}	\\
pc1	& $99.9${\tiny $\pm 0.2$}	& $99.9${\tiny $\pm 0.2$}	& $99.9${\tiny $\pm 0.2$}	& $99.9${\tiny $\pm 0.1$}	& $100.0${\tiny $\pm 0.0$}	& $99.9${\tiny $\pm 0.1$}	\\
tic-tac-toe	& $74.5${\tiny $\pm 1.7$}	& $74.0${\tiny $\pm 1.5$}	& $74.1${\tiny $\pm 1.3$}	& $76.6${\tiny $\pm 1.3$}	& $77.1${\tiny $\pm 0.9$}	& $78.1${\tiny $\pm 1.2$}	\\
vehicle	& $95.4${\tiny $\pm 1.3$}	& $95.7${\tiny $\pm 1.2$}	& $96.6${\tiny $\pm 1.1$}	& $97.0${\tiny $\pm 0.7$}	& $97.3${\tiny $\pm 0.6$}	& $97.6${\tiny $\pm 0.5$}	\\
wdbc	& $99.5${\tiny $\pm 0.3$}	& $99.5${\tiny $\pm 0.3$}	& $99.7${\tiny $\pm 0.2$}	& $99.4${\tiny $\pm 0.3$}	& $99.8${\tiny $\pm 0.2$}	& $99.7${\tiny $\pm 0.2$}	\\
qsar-biodeg	& $68.9${\tiny $\pm 1.4$}	& $69.3${\tiny $\pm 1.6$}	& $70.3${\tiny $\pm 1.5$}	& $70.1${\tiny $\pm 1.5$}	& $69.8${\tiny $\pm 1.5$}	& $69.1${\tiny $\pm 1.6$}	\\
dresses-sales	& $59.1${\tiny $\pm 1.7$}	& $60.4${\tiny $\pm 2.0$}	& $60.4${\tiny $\pm 1.9$}	& $57.9${\tiny $\pm 2.7$}	& $57.8${\tiny $\pm 2.6$}	& $58.0${\tiny $\pm 2.9$}	\\
mfeat-fourier	& $99.5${\tiny $\pm 0.2$}	& $99.6${\tiny $\pm 0.2$}	& $99.7${\tiny $\pm 0.2$}	& $99.6${\tiny $\pm 0.1$}	& $99.8${\tiny $\pm 0.1$}	& $99.7${\tiny $\pm 0.1$}	\\
MiceProtein	& $99.7${\tiny $\pm 0.2$}	& $99.5${\tiny $\pm 0.3$}	& $99.9${\tiny $\pm 0.1$}	& $99.8${\tiny $\pm 0.1$}	& $100.0${\tiny $\pm 0.0$}	& $100.0${\tiny $\pm 0.0$}	\\
steel-plates-fault	& $99.5${\tiny $\pm 0.2$}	& $99.7${\tiny $\pm 0.2$}	& $99.8${\tiny $\pm 0.1$}	& $99.8${\tiny $\pm 0.1$}	& $99.9${\tiny $\pm 0.1$}	& $99.9${\tiny $\pm 0.1$}	\\
climate-model-simulation-crashes	& $94.4${\tiny $\pm 1.5$}	& $94.3${\tiny $\pm 1.4$}	& $94.6${\tiny $\pm 1.3$}	& $94.7${\tiny $\pm 1.4$}	& $94.4${\tiny $\pm 1.6$}	& $94.7${\tiny $\pm 1.4$}	\\
car	& $98.4${\tiny $\pm 0.4$}	& $97.6${\tiny $\pm 0.4$}	& $98.4${\tiny $\pm 0.3$}	& $98.8${\tiny $\pm 0.5$}	& $98.3${\tiny $\pm 0.5$}	& $99.2${\tiny $\pm 0.4$}	\\
cylinder-bands	& $79.4${\tiny $\pm 2.3$}	& $79.7${\tiny $\pm 2.4$}	& $80.6${\tiny $\pm 1.7$}	& $80.7${\tiny $\pm 1.6$}	& $81.1${\tiny $\pm 2.0$}	& $81.4${\tiny $\pm 1.4$}	\\
\midrule 
Wilcoxon p-value vs IPB & 5e-05	& 0.0047	& -	& 0.0011	& 0.056	& -	\\
 \bottomrule
        \end{tabular}

\end{table}

\begin{figure}[htb]
    \centering 
    \includegraphics[width=0.85\linewidth]{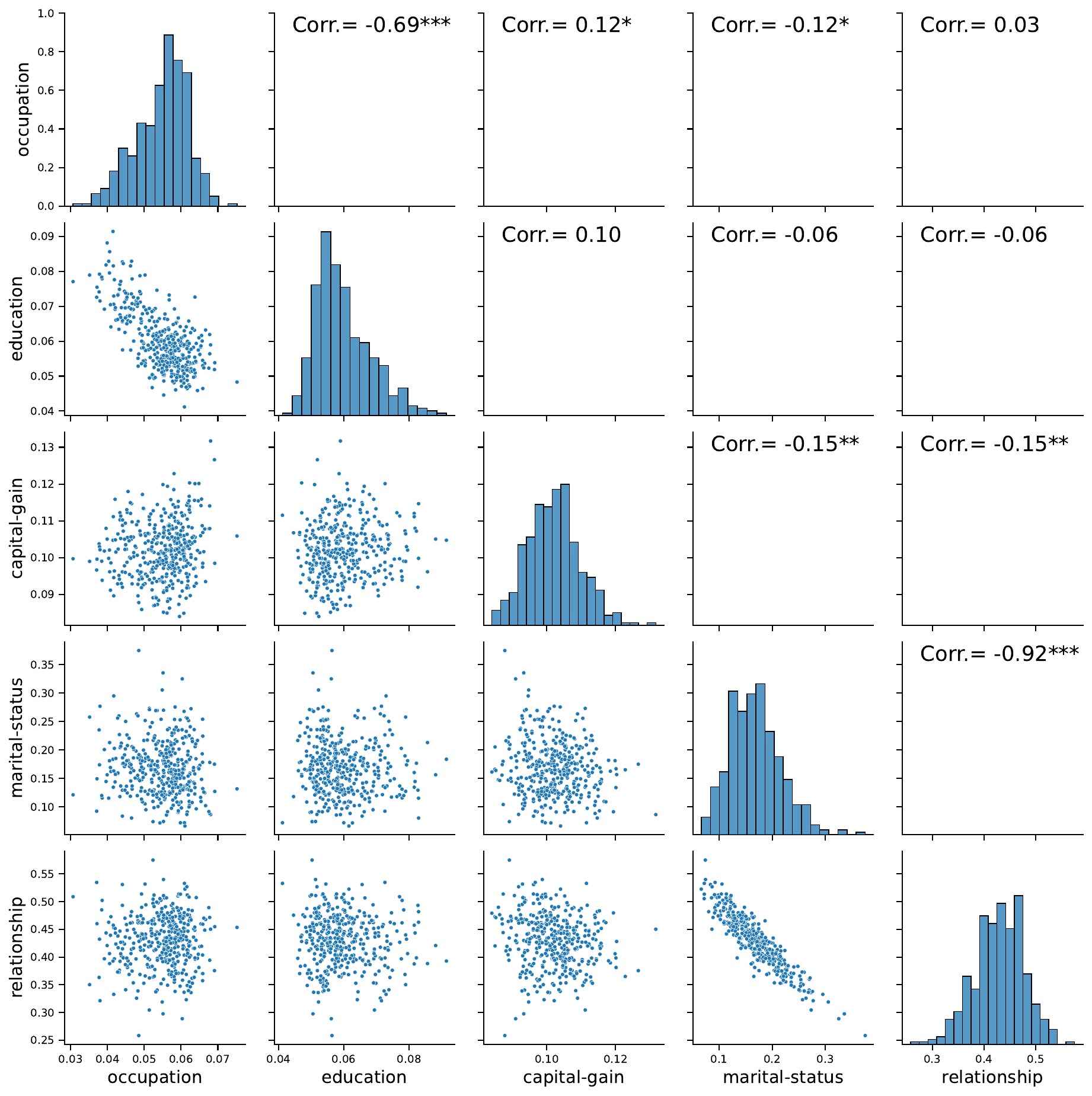}
    \caption{Classification experiment: approximate MP for the GDBT feature importance scores and their pairwise correlations. Plotted are the top 5 features in the UCI adult dataset. %
    }\label{fig:tree-adult-fea}
\end{figure}

\begin{table}[htb]
\setlength{\tabcolsep}{2pt}
\centering\scriptsize
\caption{
Interventional density estimation: full results in the setting of Table~\ref{tbl:dj-synth-main}. Reported is the estimate and $95\%$ CI for the $100\times \mathrm{MMD}^2$ metric across 30 trials. 
Boldface indicates the best result ($p<0.05$ in a Wilcoxon signed-rank test).
}\label{tbl:dj-synth-mmd-all}
\begin{tabular}[h]{ccccccccc} \toprule Method	&  chain-na	&  chain-nonlin	&  diamond-na	&  diamond-nonlin	&  triangle-na	&  triangle-nonlin	&  y-na	&  y-nonlin	\\
\midrule \multicolumn{9}{l}{$N=100$} \\ \midrule
PB	& $31.75${\tiny $\pm4.25$}	& $8.40${\tiny $\pm1.37$}	& $13.84${\tiny $\pm1.50$}	& $18.86${\tiny $\pm3.90$}	& $29.59${\tiny $\pm4.58$}	& $20.77${\tiny $\pm4.98$}	& $10.35${\tiny $\pm0.92$}	& $7.36${\tiny $\pm0.98$}	\\
Ens	& $27.40${\tiny $\pm3.55$}	& $\mathbf{6.72}${\tiny $\pm1.02$}	& $11.87${\tiny $\pm1.43$}	& $15.40${\tiny $\pm3.18$}	& $25.28${\tiny $\pm3.85$}	& $18.55${\tiny $\pm4.83$}	& $9.42${\tiny $\pm0.93$}	& $\mathbf{6.54}${\tiny $\pm0.77$}	\\
NTKGP	& $47.80${\tiny $\pm0.87$}	& $11.96${\tiny $\pm1.43$}	& $31.45${\tiny $\pm1.12$}	& $51.96${\tiny $\pm2.25$}	& $38.92${\tiny $\pm1.67$}	& $42.52${\tiny $\pm2.45$}	& $19.97${\tiny $\pm1.22$}	& $22.10${\tiny $\pm1.63$}	\\
BS	& $30.30${\tiny $\pm3.36$}	& $\mathbf{6.83}${\tiny $\pm1.05$}	& $12.81${\tiny $\pm1.40$}	& $19.88${\tiny $\pm3.54$}	& $28.21${\tiny $\pm4.81$}	& $23.09${\tiny $\pm5.21$}	& $11.54${\tiny $\pm1.73$}	& $6.76${\tiny $\pm0.78$}	\\
IPB	& $\mathbf{19.94}${\tiny $\pm2.35$}	& $\mathbf{6.31}${\tiny $\pm0.87$}	& $\mathbf{8.74}${\tiny $\pm0.92$}	& $\mathbf{9.64}${\tiny $\pm1.38$}	& $\mathbf{16.35}${\tiny $\pm1.42$}	& $\mathbf{10.02}${\tiny $\pm1.77$}	& $\mathbf{8.14}${\tiny $\pm0.76$}	& $\mathbf{6.56}${\tiny $\pm0.93$}	\\
\midrule \multicolumn{9}{l}{$N=1000$} \\ \midrule
PB	& $9.28${\tiny $\pm0.69$}	& $2.63${\tiny $\pm0.22$}	& $3.52${\tiny $\pm0.32$}	& $4.02${\tiny $\pm0.43$}	& $5.98${\tiny $\pm0.43$}	& $3.42${\tiny $\pm0.27$}	& $3.35${\tiny $\pm0.30$}	& $2.62${\tiny $\pm0.23$}	\\
Ens	& $7.45${\tiny $\pm0.60$}	& $\mathbf{2.42}${\tiny $\pm0.17$}	& $\mathbf{2.85}${\tiny $\pm0.22$}	& $3.49${\tiny $\pm0.36$}	& $4.84${\tiny $\pm0.35$}	& $3.13${\tiny $\pm0.21$}	& $\mathbf{2.84}${\tiny $\pm0.27$}	& $2.40${\tiny $\pm0.17$}	\\
NTKGP	& $21.55${\tiny $\pm0.39$}	& $2.83${\tiny $\pm0.20$}	& $8.03${\tiny $\pm0.20$}	& $11.64${\tiny $\pm0.38$}	& $12.42${\tiny $\pm0.37$}	& $6.13${\tiny $\pm0.25$}	& $5.39${\tiny $\pm0.24$}	& $3.85${\tiny $\pm0.22$}	\\
BS	& $8.58${\tiny $\pm0.68$}	& $\mathbf{2.31}${\tiny $\pm0.15$}	& $3.15${\tiny $\pm0.28$}	& $3.67${\tiny $\pm0.39$}	& $5.80${\tiny $\pm0.42$}	& $3.08${\tiny $\pm0.26$}	& $3.05${\tiny $\pm0.26$}	& $\mathbf{2.31}${\tiny $\pm0.13$}	\\
IPB	& $\mathbf{6.25}${\tiny $\pm0.46$}	& $2.58${\tiny $\pm0.20$}	& $\mathbf{2.78}${\tiny $\pm0.15$}	& $\mathbf{3.22}${\tiny $\pm0.37$}	& $\mathbf{4.23}${\tiny $\pm0.31$}	& $\mathbf{2.79}${\tiny $\pm0.19$}	& $\mathbf{2.98}${\tiny $\pm0.21$}	& $\mathbf{2.27}${\tiny $\pm0.13$}	\\
\bottomrule \end{tabular}
    
\end{table}

\begin{table}[htb]
\setlength{\tabcolsep}{2pt}
\centering\scriptsize
\caption{
Interventional 
density estimation experiment: additional results for quality of uncertainty estimates, when $n=100$. 
Reported are the estimate and $95\%$ CI for the mean of each test metric. 
Boldface indicates the best result ($p<0.05$ in a Wilcoxon signed-rank test).
}\label{tbl:dj-synth-od-100}
\begin{tabular}[h]{ccccccccc} \toprule Method	&  chain-na	&  chain-nonlin	&  diamond-na	&  diamond-nonlin	&  triangle-na	&  triangle-nonlin	&  y-na	&  y-nonlin	\\
\midrule \multicolumn{9}{l}{CDF $L_2$} \\ \midrule
PB	& $0.023${\tiny $\pm0.004$}	& $0.008${\tiny $\pm0.002$}	& $0.041${\tiny $\pm0.005$}	& $0.068${\tiny $\pm0.010$}	& $0.073${\tiny $\pm0.008$}	& $0.049${\tiny $\pm0.010$}	& $0.013${\tiny $\pm0.002$}	& $0.012${\tiny $\pm0.003$}	\\
Ens	& $0.019${\tiny $\pm0.003$}	& $\mathbf{0.007}${\tiny $\pm0.002$}	& $0.035${\tiny $\pm0.002$}	& $0.078${\tiny $\pm0.011$}	& $0.075${\tiny $\pm0.008$}	& $0.049${\tiny $\pm0.010$}	& $0.013${\tiny $\pm0.002$}	& $\mathbf{0.009}${\tiny $\pm0.002$}	\\
NTKGP	& $0.039${\tiny $\pm0.002$}	& $0.009${\tiny $\pm0.002$}	& $0.053${\tiny $\pm0.002$}	& $0.098${\tiny $\pm0.005$}	& $0.086${\tiny $\pm0.006$}	& $0.082${\tiny $\pm0.006$}	& $0.027${\tiny $\pm0.002$}	& $0.028${\tiny $\pm0.003$}	\\
BS	& $0.022${\tiny $\pm0.004$}	& $\mathbf{0.006}${\tiny $\pm0.001$}	& $0.037${\tiny $\pm0.004$}	& $0.083${\tiny $\pm0.011$}	& $0.076${\tiny $\pm0.008$}	& $0.058${\tiny $\pm0.010$}	& $0.015${\tiny $\pm0.003$}	& $\mathbf{0.010}${\tiny $\pm0.001$}	\\
IPB	& $\mathbf{0.013}${\tiny $\pm0.002$}	& $\mathbf{0.006}${\tiny $\pm0.001$}	& $\mathbf{0.028}${\tiny $\pm0.002$}	& $\mathbf{0.054}${\tiny $\pm0.010$}	& $\mathbf{0.059}${\tiny $\pm0.006$}	& $\mathbf{0.035}${\tiny $\pm0.007$}	& $\mathbf{0.010}${\tiny $\pm0.001$}	& $\mathbf{0.011}${\tiny $\pm0.002$}	\\
\midrule \multicolumn{9}{l}{Average coverage of 90\% CI} \\ \midrule
PB	& $0.960${\tiny $\pm0.017$}	& $0.731${\tiny $\pm0.109$}	& $0.762${\tiny $\pm0.090$}	& $0.637${\tiny $\pm0.073$}	& $0.506${\tiny $\pm0.065$}	& $0.640${\tiny $\pm0.068$}	& $0.750${\tiny $\pm0.095$}	& $0.806${\tiny $\pm0.088$}	\\
Ens	& $0.388${\tiny $\pm0.101$}	& $0.334${\tiny $\pm0.090$}	& $0.231${\tiny $\pm0.046$}	& $0.244${\tiny $\pm0.043$}	& $0.181${\tiny $\pm0.025$}	& $0.271${\tiny $\pm0.049$}	& $0.304${\tiny $\pm0.082$}	& $0.345${\tiny $\pm0.085$}	\\
NTKGP	& $0.388${\tiny $\pm0.094$}	& $0.412${\tiny $\pm0.095$}	& $0.256${\tiny $\pm0.044$}	& $0.152${\tiny $\pm0.024$}	& $0.151${\tiny $\pm0.015$}	& $0.158${\tiny $\pm0.017$}	& $0.182${\tiny $\pm0.045$}	& $0.265${\tiny $\pm0.075$}	\\
BS	& $0.861${\tiny $\pm0.075$}	& $0.806${\tiny $\pm0.080$}	& $0.762${\tiny $\pm0.081$}	& $0.572${\tiny $\pm0.074$}	& $0.511${\tiny $\pm0.068$}	& $0.638${\tiny $\pm0.075$}	& $0.801${\tiny $\pm0.062$}	& $0.798${\tiny $\pm0.094$}	\\
IPB	& $0.966${\tiny $\pm0.009$}	& $0.865${\tiny $\pm0.048$}	& $0.934${\tiny $\pm0.028$}	& $0.915${\tiny $\pm0.031$}	& $0.796${\tiny $\pm0.047$}	& $0.930${\tiny $\pm0.028$}	& $0.833${\tiny $\pm0.066$}	& $0.804${\tiny $\pm0.070$}	\\
\midrule \multicolumn{9}{l}{Average width of 90\% CI} \\ \midrule
PB	& $0.216${\tiny $\pm0.014$}	& $0.339${\tiny $\pm0.033$}	& $0.205${\tiny $\pm0.020$}	& $0.382${\tiny $\pm0.035$}	& $0.587${\tiny $\pm0.080$}	& $0.442${\tiny $\pm0.050$}	& $0.431${\tiny $\pm0.032$}	& $0.308${\tiny $\pm0.016$}	\\
Ens	& $0.069${\tiny $\pm0.007$}	& $0.109${\tiny $\pm0.006$}	& $0.049${\tiny $\pm0.003$}	& $0.120${\tiny $\pm0.012$}	& $0.173${\tiny $\pm0.018$}	& $0.138${\tiny $\pm0.012$}	& $0.164${\tiny $\pm0.010$}	& $0.087${\tiny $\pm0.005$}	\\
NTKGP	& $0.117${\tiny $\pm0.004$}	& $0.133${\tiny $\pm0.002$}	& $0.110${\tiny $\pm0.003$}	& $0.173${\tiny $\pm0.005$}	& $0.176${\tiny $\pm0.010$}	& $0.186${\tiny $\pm0.009$}	& $0.199${\tiny $\pm0.011$}	& $0.140${\tiny $\pm0.005$}	\\
BS	& $0.199${\tiny $\pm0.012$}	& $0.339${\tiny $\pm0.015$}	& $0.176${\tiny $\pm0.013$}	& $0.402${\tiny $\pm0.021$}	& $0.615${\tiny $\pm0.061$}	& $0.502${\tiny $\pm0.029$}	& $0.488${\tiny $\pm0.024$}	& $0.323${\tiny $\pm0.020$}	\\
IPB	& $0.168${\tiny $\pm0.007$}	& $0.338${\tiny $\pm0.012$}	& $0.208${\tiny $\pm0.016$}	& $0.768${\tiny $\pm0.046$}	& $1.043${\tiny $\pm0.126$}	& $0.785${\tiny $\pm0.074$}	& $0.459${\tiny $\pm0.021$}	& $0.268${\tiny $\pm0.009$}	\\
\bottomrule \end{tabular}

\end{table}

\begin{table}[htb]
\setlength{\tabcolsep}{2pt}
\centering\scriptsize
\caption{
Interventional 
density estimation experiment: additional results for quality of uncertainty estimates, when $n=1000$.
Reported are the estimate and $95\%$ CI for the mean of each test metric. For CDF $L_2$, 
boldface indicates the best result ($p<0.05$ in a Wilcoxon signed-rank test).
}\label{tbl:dj-synth-od-1000}
\begin{tabular}[h]{ccccccccc} \toprule Method	&  chain-na	&  chain-nonlin	&  diamond-na	&  diamond-nonlin	&  triangle-na	&  triangle-nonlin	&  y-na	&  y-nonlin	\\
\midrule \multicolumn{9}{l}{CDF $L_2$} \\ \midrule
PB	& $0.006${\tiny $\pm0.001$}	& $0.001${\tiny $\pm0.000$}	& $0.017${\tiny $\pm0.001$}	& $0.041${\tiny $\pm0.005$}	& $0.029${\tiny $\pm0.002$}	& $0.010${\tiny $\pm0.002$}	& $0.004${\tiny $\pm0.000$}	& $0.002${\tiny $\pm0.001$}	\\
Ens	& $0.004${\tiny $\pm0.000$}	& $\mathbf{0.001}${\tiny $\pm0.000$}	& $0.016${\tiny $\pm0.001$}	& $0.043${\tiny $\pm0.005$}	& $\mathbf{0.026}${\tiny $\pm0.002$}	& $0.011${\tiny $\pm0.002$}	& $0.003${\tiny $\pm0.000$}	& $\mathbf{0.002}${\tiny $\pm0.000$}	\\
NTKGP	& $0.014${\tiny $\pm0.000$}	& $0.001${\tiny $\pm0.000$}	& $0.027${\tiny $\pm0.001$}	& $0.047${\tiny $\pm0.004$}	& $0.035${\tiny $\pm0.002$}	& $0.016${\tiny $\pm0.001$}	& $0.007${\tiny $\pm0.000$}	& $0.004${\tiny $\pm0.000$}	\\
BS	& $0.005${\tiny $\pm0.001$}	& $\mathbf{0.001}${\tiny $\pm0.000$}	& $0.017${\tiny $\pm0.001$}	& $0.043${\tiny $\pm0.006$}	& $\mathbf{0.027}${\tiny $\pm0.002$}	& $0.011${\tiny $\pm0.001$}	& $0.004${\tiny $\pm0.000$}	& $\mathbf{0.002}${\tiny $\pm0.000$}	\\
IPB	& $\mathbf{0.004}${\tiny $\pm0.001$}	& $0.001${\tiny $\pm0.000$}	& $\mathbf{0.014}${\tiny $\pm0.001$}	& $\mathbf{0.031}${\tiny $\pm0.005$}	& $\mathbf{0.027}${\tiny $\pm0.002$}	& $\mathbf{0.008}${\tiny $\pm0.002$}	& $\mathbf{0.003}${\tiny $\pm0.000$}	& $\mathbf{0.002}${\tiny $\pm0.000$}	\\
\midrule \multicolumn{9}{l}{Average coverage of 90\% CI} \\ \midrule
PB	& $0.870${\tiny $\pm0.064$}	& $0.901${\tiny $\pm0.054$}	& $0.908${\tiny $\pm0.041$}	& $0.746${\tiny $\pm0.051$}	& $0.701${\tiny $\pm0.052$}	& $0.878${\tiny $\pm0.037$}	& $0.958${\tiny $\pm0.027$}	& $0.947${\tiny $\pm0.032$}	\\
Ens	& $0.633${\tiny $\pm0.085$}	& $0.712${\tiny $\pm0.089$}	& $0.522${\tiny $\pm0.055$}	& $0.347${\tiny $\pm0.061$}	& $0.331${\tiny $\pm0.045$}	& $0.408${\tiny $\pm0.044$}	& $0.716${\tiny $\pm0.062$}	& $0.679${\tiny $\pm0.073$}	\\
NTKGP	& $0.654${\tiny $\pm0.108$}	& $0.709${\tiny $\pm0.086$}	& $0.539${\tiny $\pm0.056$}	& $0.254${\tiny $\pm0.038$}	& $0.159${\tiny $\pm0.020$}	& $0.377${\tiny $\pm0.022$}	& $0.683${\tiny $\pm0.072$}	& $0.687${\tiny $\pm0.072$}	\\
BS	& $0.963${\tiny $\pm0.021$}	& $0.989${\tiny $\pm0.008$}	& $0.848${\tiny $\pm0.049$}	& $0.662${\tiny $\pm0.070$}	& $0.624${\tiny $\pm0.056$}	& $0.758${\tiny $\pm0.044$}	& $0.935${\tiny $\pm0.032$}	& $0.937${\tiny $\pm0.029$}	\\
IPB	& $0.927${\tiny $\pm0.029$}	& $0.884${\tiny $\pm0.042$}	& $0.927${\tiny $\pm0.029$}	& $0.838${\tiny $\pm0.050$}	& $0.670${\tiny $\pm0.048$}	& $0.891${\tiny $\pm0.029$}	& $0.876${\tiny $\pm0.050$}	& $0.890${\tiny $\pm0.044$}	\\
\midrule \multicolumn{9}{l}{Average width of 90\% CI} \\ \midrule
PB	& $0.090${\tiny $\pm0.002$}	& $0.182${\tiny $\pm0.005$}	& $0.082${\tiny $\pm0.003$}	& $0.217${\tiny $\pm0.014$}	& $0.600${\tiny $\pm0.041$}	& $0.263${\tiny $\pm0.016$}	& $0.265${\tiny $\pm0.006$}	& $0.152${\tiny $\pm0.004$}	\\
Ens	& $0.045${\tiny $\pm0.001$}	& $0.097${\tiny $\pm0.001$}	& $0.032${\tiny $\pm0.001$}	& $0.069${\tiny $\pm0.002$}	& $0.139${\tiny $\pm0.006$}	& $0.073${\tiny $\pm0.004$}	& $0.144${\tiny $\pm0.003$}	& $0.069${\tiny $\pm0.001$}	\\
NTKGP	& $0.060${\tiny $\pm0.001$}	& $0.104${\tiny $\pm0.001$}	& $0.050${\tiny $\pm0.000$}	& $0.081${\tiny $\pm0.002$}	& $0.128${\tiny $\pm0.004$}	& $0.091${\tiny $\pm0.001$}	& $0.150${\tiny $\pm0.003$}	& $0.085${\tiny $\pm0.001$}	\\
BS	& $0.082${\tiny $\pm0.002$}	& $0.159${\tiny $\pm0.002$}	& $0.067${\tiny $\pm0.001$}	& $0.173${\tiny $\pm0.006$}	& $0.434${\tiny $\pm0.023$}	& $0.175${\tiny $\pm0.004$}	& $0.220${\tiny $\pm0.004$}	& $0.126${\tiny $\pm0.002$}	\\
IPB	& $0.072${\tiny $\pm0.001$}	& $0.153${\tiny $\pm0.002$}	& $0.063${\tiny $\pm0.001$}	& $0.230${\tiny $\pm0.010$}	& $0.450${\tiny $\pm0.021$}	& $0.235${\tiny $\pm0.005$}	& $0.219${\tiny $\pm0.005$}	& $0.117${\tiny $\pm0.002$}	\\
\bottomrule \end{tabular}

\end{table}

\subsection{Interventional Density Estimation}\label{app:exp-dj}

\paragraph{Setup details.} 
For the base estimation algorithm, 
we adopt a fully-connected NN model with 128 hidden units in each layer, and determine the other hyperparameters in the following range: 
(i) number of hidden layers $D\in\{2, 3, 4\}$, (ii) learning rate $\eta\in\{0.1, 0.5, 1, 5\}\times 10^{-3}$, 
(iii) training iterations $L\in\{2,4,8\}\times 1000$, and 
(iv) activation function from \{swish, selu, tanh\}. 
The hyperparameters are determined by evaluating the training objective on an in-distribution validation set, on the \texttt{chain-na} dataset from \citet{chao2023interventional}. 
We use the AdamW optimiser. %
For our method, we instantiate the proximal Bregman objective \eqref{eq:nn-objective} using the weighted score matching loss in \citet{ho2020denoising}, and set 
$\Delta n=0.1n, N=6n$: beyond this range, a larger value of $N$ leads to diminishing improvement, and the results appear somewhat insensitive to the choice of $\Delta n$. Other implementation details are discussed in \Cref{app:impl-details}. 

On the synthetic datasets, we consider two evaluation setups: 
\begin{itemize}
    \item Following \citet{chao2023interventional} we evaluate distributional estimates for $\PP(x_{\mathrm{desc}(i)}\mid \mathrm{do}(x_i=x))$, where $\mathrm{desc}(i)$ denotes the descendents of node $i$ in the causal graph and $x$ ranges over a uniform grid of the $[0.1, 0.9]$ quantile. We report the maximum mean discrepancy for in this setup. 
    \item We present a more direct evaluation of the uncertainty estimates, 
    by evaluating the average coverage of pointwise credible intervals for the mean outcome $\EE(x_d\mid \mathrm{do}(x_{1:d-1}=\cdot))$ and the $L_2$ distance between the estimated CDF and ground truth. The latter 
    is equivalent to CRPS and is thus a meaningful surrogate for forecasting error. 
    The value for $x_{1:d-1}$ is determined by varying one of the variables on a uniform grid and fixing the others to $\{-0.5, 0, 0.5\}$, consecutively. 
\end{itemize}

On the fMRI dataset, we report the median of absolute error following \citet{khemakhem2021causal,chao2023interventional} and the CRPS. 
Our setup, where we average over random seeds (which determine the initialisation and train/validation split), appears different from \citet{khemakhem2021causal}, and we can exactly match their reported results using 
a single (default) seed set in their codebase. Nonetheless, the results remain statistically consistent. 

\paragraph{Full results and discussion.} 
Full results for the synthetic experiments are shown in \Cref{tbl:dj-synth-mmd-all} (in the setting of \Cref{tbl:dj-synth-main} and \citet{chao2023interventional}) and \Cref{tbl:dj-synth-od-100}--\ref{tbl:dj-synth-od-1000} (for the evaluation of uncertainty). 
As we can see, our method attains the best overall performance for both prediction and uncertainty quantification. 
The vanilla ensemble method achieves the best predictive performance across baselines, which is 
consistent with previous reports \citep{fort2019deep,gorishniy2021revisiting}. 
\texttt{NTKGP} is generally uncompetitive; even through the method is applied to the same DNN models, 
it is possible that the ultrawide NN perspective which motivated their design choices is less %
applicable to diffusion models which utilise multi-output NNs. 
The predictive performance of \texttt{PB} is uncompetitive possibly related to its discard of real data. 
For uncertainty quantification, however, both bootstrap baselines demonstrate better performance than the other baselines, although our method still achieves better performance. 
Note that due to the distribution shift we cannot expect the coverage of credible intervals to match their exact nominal level. 

\ifnum\isPreprint=0
\subsection{Code, Computational Resource, and Assets Used}\label{app:exp-misc}

We provide code for the classification experiment as supplementary material. 
Code to reproduce all experiments will be available upon publication. 

Amount of compute required for each experiment is as follows:
\begin{itemize}
    \item 
    \Cref{app:toy-exp-gp}: experiment takes 5 minutes on an Nvidia RTX A2000 GPU.
    \item \Cref{app:gp-mtl}: a single run with $m=200,n_{test}=20$ takes 2 minutes on an Nvidia RTX A2000 GPU; all experiments take an equivalent of 36 hours on a single Nvidia Tesla A100 GPU.
    \item \Cref{app:exp-gp}: a single run on the \texttt{wine-white} dataset takes 4.5 minutes on an Nvidia Tesla V100 GPU; all experiments cost an equivalent of 340 hours on a single V100 GPU.
    \item \Cref{app:exp-tree}: experiments take a total of 25 hours on 96 CPU cores (2 AMD EPYC 7642 CPUs); for the AutoML experiment, a single run for the proposed method takes an average of 67 minutes on 4 EPYC cores; for the XGBoost experiment, 
    a single run on the \texttt{banknote-authentication} dataset takes about 4 seconds on an Intel i5-13500 CPU. 
    \item \Cref{app:exp-dj}: 
    a single run for the proposed method (\texttt{triangle-nonlin}, $n=100$) takes 4 minutes on an Nvidia RTX A2000 GPU; all experiments cost an equivalent of 112 hours on an Nvidia RTX 4090 GPU.
\end{itemize}
We report equivalent time cost for the GPU experiments which run in parallel. %
We expect the time cost of the GP experiments to be improvable as the current implementation does not make optimal use of JIT compilation. 

Main software libraries used in the code include 
JAX \citep[version 0.4.32]{jax2018github}, Flax \citep[version 0.7.4]{flax2020github}, Optax \citep[version 0.1.7]{deepmind2020jax}, GPJax \citep[version 0.7.1]{Pinder2022}, XGBoost \citep[version 2.0.3]{chen2016xgboost}, and AutoGluon \citep[version 1.0.0]{agtabular}, all released under the Apache-2.0 licence; and OpenML \citep[version 0.14.0, BSD licence]{bischl2017openml}. The OpenML datasets used in \Cref{app:exp-tree} are individually licenced with licences listed on the website (\url{https://openml.org}). 
\fi

\ifnum\isPreprint=0
\newpage 
\clearpage
\input{neurips-checklist.tex}
\fi
\end{document}